\documentclass[twoside,11pt]{article}

\usepackage{jmlr2e}
\usepackage{times}
\usepackage{soul}
\usepackage{url}
\usepackage[small]{caption}
\usepackage{graphicx}
\usepackage{amsmath}
\usepackage{booktabs}
\usepackage{algorithm}
\usepackage[switch]{lineno}
\usepackage[noend]{algpseudocode}
\usepackage{amsfonts}
\usepackage{amssymb}
\usepackage{bm}
\usepackage{algorithmicx,algorithm}
\usepackage{color}
\usepackage{multicol,multirow}

\newtheorem{assumption}{Assumption}
\usepackage[normalem]{ulem} 

\newcommand{\R}{{\mathbb{R}}}

\newcommand{\E}{\mathbb{E}}

\usepackage{lastpage}
\jmlrheading{1}{2024}{1--\pageref{LastPage}}{0/00; Revised 0/00}{0/00}{21-0000}{Jie Peng, Weiyu Li, Stefan Vlaski, Qing Ling}

\ShortHeadings{Mean Aggregator Is More Robust Than Robust Aggregators}{Peng, Li, Vlaski and Ling}
\firstpageno{1}

\begin{document}
	
	\title{Mean Aggregator is More Robust than Robust Aggregators under Label Poisoning Attacks on Distributed Heterogeneous Data}
	
	\author{\name Jie Peng \email pengj95@mail2.sysu.edu.cn \\
		\addr School of Computer Science and Engineering\\
		Sun Yat-Sen University\\
		Guangzhou, Guangdong 510006, China
		\AND
		\name Weiyu Li \email weiyuli@g.harvard.edu \\
		\addr School of Engineering and Applied Science\\
		Harvard University\\
		Cambridge, MA 02138, USA
		\AND
		\name Stefan Vlaski \email s.vlaski@imperial.ac.uk \\
		\addr Department of Electrical and Electronic Engineering \\
		Imperial College London\\
		London SW7 2BT, UK
		\AND
		\name Qing Ling \email lingqing556@mail.sysu.edu.cn \\
		\addr School of Computer Science and Engineering\\
		Sun Yat-Sen University\\
		Guangzhou, Guangdong 510006, China
	}
	
	\editor{}
	
	\maketitle
	
	\begin{abstract}
		Robustness to malicious attacks is of paramount importance for distributed learning. Existing works usually consider the classical Byzantine attacks model, which assumes that some workers can send arbitrarily malicious messages to the server and disturb the aggregation steps of the distributed learning process. To defend against such worst-case Byzantine attacks, various robust aggregators have been proposed. They are proven to be effective and much superior to the often-used mean aggregator. In this paper, however, we demonstrate that the robust aggregators are too conservative for a class of weak but practical malicious attacks, as known as label poisoning attacks, where the sample labels of some workers are poisoned. Surprisingly, we are able to show that the mean aggregator is more robust than the state-of-the-art robust aggregators in theory, given that the distributed data are sufficiently heterogeneous. In fact, the learning error of the mean aggregator is proven to be order-optimal in this case. Experimental results corroborate our theoretical findings, showing the superiority of the mean aggregator under label poisoning attacks.
	\end{abstract}
	
	
	\section{Introduction}
	
	With the rising and rapid development of large machine learning models, distributed learning has attracted intensive research attention due to its provable effectiveness in solving large-scale problems \citep{verbraeken2020survey, li2020communication}. In distributed learning, there often exist one parameter server (called server thereafter) owning the global model and some computation devices (called workers thereafter) owning the local data. In the training process, the server sends the global model to the workers, and the workers use their local data to compute the local stochastic gradients or momenta on the global model and send them back to the server. Upon receiving the messages from all workers, the server aggregates them and uses the aggregated stochastic gradient or momentum to update the global model. After the training process, the trained global model is evaluated on the testing data. An essential component of distributed learning is federated learning \citep{mcmahan2017communication, yang2019federated, beltran2023decentralized, ye2023heterogeneous, fraboni2023general}, which is particularly favorable in terms of privacy preservation.
	
	However, the distributed nature of such a server-worker architecture is vulnerable to malicious attacks during the learning process \citep{lewis2023attacks}. Due to data corruptions, equipment failures, or cyber attacks, some workers may not follow the algorithmic protocol, and instead send incorrect messages to the server. Previous works often characterize these attacks by the classical Byzantine attacks model, which assumes that some workers can send arbitrarily malicious messages to the server so that the aggregation steps of the learning process are disturbed \citep{lamport1982byzantine}. For such worst-case Byzantine attacks, various robust aggregators have been proven effective and much superior to the mean aggregator \citep{chen2017distributed,xia2019faba,karimireddy2021learning,wu2023byzantine}.

	The malicious attacks encountered in reality, on the other hand, are often less destructive than the worst-case Byzantine attacks. For example, a distributed learning system may often suffer from label poisoning attacks, which are weak yet of practical interest. Considering a highly secure email system in a large organization (for example, government or university), if hackers (some users) aim to disturb the online training process of a spam detection model, one of the most effective ways for them is to mislabel received emails from ``spam'' to ``non-spam'', resulting in label poisoning attacks. Similar attacks may happen in fraudulent short message service (SMS) detection held by large communication corporations, too.
	
	To this end, in this paper, we consider label poisoning attacks where some workers have local data with poisoned labels and generate incorrect messages during the learning process. Under label poisoning attacks and with some mild assumptions, surprisingly we are able to show that the mean aggregator is more robust than the state-of-the-art robust aggregators in theory. To be specific, we prove that the mean aggregator has a better learning error bound than the robust aggregators (see Theorems \ref{thm:convergence with RAgg} and \ref{thm:convergence with Mean} in Section \ref{sec: convergence analysis}), given that the distributed data are sufficiently heterogeneous. The main contributions of this paper are summarized as follows.
	
	\textbf{C1)} To the best of our knowledge, our work is the first to investigate the robustness of the mean aggregator in distributed learning. Our work reveals an important fact that the mean aggregator is more robust than the existing robust aggregators under specific types of malicious attacks, which motivates us to rethink the usage of different aggregators within practical scenarios.
.
	
	\textbf{C2)} Under label poisoning attacks, we theoretically analyze the learning errors of the mean aggregator and the state-of-the-art robust aggregators. The results show that when the heterogeneity of the distributed data is large, the learning error of the mean aggregator is order-optimal regardless of the fraction of poisoned workers.
	
	\textbf{C3)} We empirically evaluate the performance of the mean aggregator and the state-of-the-art robust aggregators under label poisoning attacks. The experimental results fully support our theoretical findings.
	
	This paper significantly extends upon our previous conference paper \citep{peng2024mean}.
	First, \citet{peng2024mean} consider the distributed gradient descent algorithm, which ignores the stochastic gradient noise that is critical to distributed learning. To address this issue, we investigate distributed stochastic momentum as the backbone algorithm, and provide new theoretical analysis for distributed stochastic momentum with the mean aggregator or the robust aggregators. By properly handling the stochastic gradient noise, we establish the tight upper bounds and the lower bound for the learning error. These theoretical results recover those in \citet{peng2024mean} if the momentum coefficient is set to $1$ and the inner variance bound is $0$ (see Remark \ref{remak-10}). Second, we present new, comprehensive experimental results to compare the performance of various aggregators combined with the distributed stochastic momentum algorithm, validating our theoretical findings.
	
	\section{Related Works}

	Poisoning attacks can be categorized into targeted attacks and untargeted attacks; or model poisoning attacks and data poisoning attacks \citep{kairouz2021advances}. In this paper, we focus on the latter categorization.
	In model poisoning attacks, the malicious workers send arbitrarily poisoned models to the server, while data poisoning attacks yield poisoned messages by fabricating poisoned data at the malicious workers' side \citep{shejwalkar2022back}. Below we briefly review the related works of the two types of poisoning attacks in distributed learning, respectively.
	
	Under model poisoning attacks, most of the existing works design robust aggregators for aggregating local stochastic gradients of the workers and filter out the potentially poisoned messages. The existing robust aggregators include Krum \citep{blanchard2017machine}, geometric median \citep{chen2017distributed}, coordinate-wise median \citep{yin2018byzantine}, coordinate-wise trimmed-mean \citep{yin2018byzantine}, FABA \citep{xia2019faba}, centered clipping \citep{karimireddy2021learning}, VRMOM \citep{tu2021variance}, etc. The key idea behind these robust aggregators is to find a point that has bounded distance to the true stochastic gradient such that the learning error is under control. \citet{farhadkhani2022byzantine} and \citet{allouah2023fixing} propose a unified framework to analyze the performance of these robust aggregators under attacks.
	However, the above works do not consider the effect of the stochastic gradient noise which may provide a shelter for Byzantine attacks and increase the learning error. To address this issue, \citet{khanduri2019byzantine, wu2020federated, karimireddy2021learning, rammal2024communication, guerraoui2024byzantine}  propose to use the variance-reduction  and momentum techniques to alleviate the effect of the stochastic gradient noise and enhance the Byzantine-robustness.
	Though these methods work well when the data distributions are the same over the workers, their performance degrades when the data distributions become heterogeneous \citep{li2019rsa, karimireddy2021byzantine}. Therefore, \citet{li2019rsa} suggests using model aggregation rather than stochastic gradient aggregation to defend against model poisoning attacks in the heterogeneous case. \citet{karimireddy2021byzantine,peng2022byzantine,allouah2023fixing} propose to use the bucketing/resampling and nearest neighbor mixing techniques to reduce the heterogeneity of the messages, prior to aggregation.

	Some other works focus on asynchronous learning \citep{yang2023buffered} or decentralized learning without a server \citep{peng2021byzantine, he2022byzantine, wu2023byzantine}, under model poisoning attacks. Nevertheless, we focus on synchronous distributed learning with a server in this paper. 
	
	There are also a large amount of papers focusing on data poisoning attacks \citep{sun2019can,bagdasaryan2020backdoor,wang2020attack,rosenfeld2020certified, cina2024machine}. To defend against data poisoning attacks, the existing works use data sanitization to remove poisoned data \citep{steinhardt2017certified}, and prune activation units that are inactive on clean data \citep{liu2018fine}. For more defenses against data poisoning attacks, we refer the reader to the survey paper \citep{kairouz2021advances}.
	
	In practice, however, attacks may not necessarily behave as arbitrarily malicious as the above well-established works consider. Some weaker attacks models are structured; for example,  \citet{tavallali2022adversarial} considers the label poisoning attacks in which some workers mislabel their local data and compute the incorrect messages using those poisoned data. Specifically, \citet{tolpegin2020data,lin2021ml,jebreel2023fl, jebreel2024lfighter} consider the case where some workers flip the labels of their local data from source classes to target classes. Notably, label poisoning is a kind of data poisoning but not necessarily the worst-case attack, since label poisoning attacks fabricate the local data, yet only on the label level. 	

	It has been shown that the robust aggregators designed for model poisoning attacks can be applied to defend against the label poisoning attacks, as validated by \citet{fang2020local, karimireddy2021byzantine,gorbunov2022variance}. There also exist some works designing new robust aggregators based on specific properties of label poisoning.
	For example, the work of \citet{tavallali2022adversarial} proposes regularization-based defense to detect and exclude the samples with flipped labels in the training process. However, \citet{tavallali2022adversarial} requires to access a clean validation set, which has privacy concerns in distributed learning. Another work named as LFighter \citep{jebreel2024lfighter} is the state-of-the-art defense for label poisoning attacks in federated learning. \citet{jebreel2024lfighter} proposes to cluster the local gradients of all workers, identify the smaller and denser clusters as the potentially poisoned gradients, and discard them. The key idea of LFighter is that the difference between the stochastic gradients connected to the source and target output neurons of poisoned workers and regular workers becomes larger when the training process evolves. Therefore, we are able to identify the potentially poisoned stochastic gradients. However, LFighter only works well when data distributions at different workers are similar. If the heterogeneity of the distributed data is large, the performance of LFighter degrades, as we will show in Section \ref{sec: experiments}.
	
	Though the recent works of \citet{karimireddy2021byzantine} and \citet{farhadkhani2024relevance} respectively prove the optimality of certain robust aggregators for model poisoning attacks and data poisoning attacks, they only consider the case that the poisoned workers can cause unbounded disturbances to the learning process. In contrast, we consider the case that the disturbances caused by the poisoned workers is bounded and prove the optimality of the mean aggregator for label poisoning attacks when the distributed data are sufficiently heterogeneous, and experimentally validate our theoretical findings.

	The recent work of \citet{shejwalkar2022back}, similar to our findings, reveals the robustness of the mean aggregator under poisoning attacks in production federated learning systems. Nevertheless, their study is restrictive in terms of the poisoning ratio (for example, less than 0.1\% workers are poisoned while we can afford 10\% in the numerical experiments) and lacks theoretical analysis. In contrast, we provide both theoretical analysis and experimental validations.
	
	In conclusion, our work is the first one to investigate the robustness of the mean aggregator in distributed learning. It reveals an important fact that the robust aggregators cannot always outperform the mean aggregator under specific attacks, promoting us to rethink the application scenarios for the use of robust aggregators.
	
	\section{Problem Formulation}
	
	Consider a distributed learning system with one server and $W$ workers. Denote the set of workers as $\mathcal{W}$ with $|\mathcal{W}|=W$, and the set of regular workers as $\mathcal{R}$  with $|\mathcal{R}| = R$. Note that the number and identities of the regular workers are unknown. Our goal is to solve the following distributed learning problem defined over the regular workers in $\mathcal{R}$, at the presence of the set of poisoned workers $\mathcal{W}\setminus\mathcal{R}$:
	\begin{align}\label{problem}
		\min_{x \in \R^D} f(x) &\triangleq \frac{1}{R} \sum_{w \in \mathcal{R}} f_w(x), \\
		\text{with } \  f_w(x) &\triangleq \frac{1}{J}\sum_{j=1}^{J} f_{w, j}(x), \ \ \forall w \in \mathcal{R}. \nonumber
	\end{align}
	Here, $x \in \R^D$ is the global model and $f_w(x)$ is the local cost of worker $w \in \mathcal{R}$ that averages the costs $f_{w, j}(x)$ of $J$ samples. Without loss of generality, we assume that all workers have the same number of samples $J$.
	
	We begin with characterizing the behaviors of the poisoned workers in $\mathcal{W}\setminus\mathcal{R}$. Different to the classical Byzantine attacks model that assumes some workers to disobey the algorithmic protocol and send arbitrarily malicious messages to the server \citep{lamport1982byzantine}, here we assume the poisoned workers to: (i) have samples with poisoned labels; (ii) exactly follow the algorithmic protocol during the distributed learning process. The formal definition is given as follows.

	\begin{definition}[\textbf{Label poisoning attacks}]\label{Def: delta-LPA}
		In solving \eqref{problem}, there exist a number of poisoned workers, whose local costs are in the same form as the regular workers but an arbitrary fraction of sample labels are poisoned. Nevertheless, these poisoned workers exactly follow the algorithmic protocol during the distributed learning process.
	\end{definition}
	
	We solve \eqref{problem} with the distributed stochastic momentum algorithm, which includes the popular distributed gradient descent and distributed stochastic gradient descent algorithms as special cases. At each iteration, each worker randomly accesses one local sample to compute a local stochastic gradient, updates its local momentum, and sends the local momentum to the server. Then, the server aggregates the local momenta of all workers. However, as we have emphasized, the number and identities of the regular workers are unknown, such that the server cannot distinguish the true local momenta from the regular workers and the poisoned local momenta from the poisoned workers. We call the true and poisoned local momenta as messages, which the server must judiciously aggregate.
	
	The distributed stochastic momentum algorithm works as follows. At iteration $t$, the server first broadcasts the global model $x^t$ to all workers. Then,
	each worker $w \in \mathcal{W}$ selects a sample index $i_w^t$ uniformly randomly from $\{1, \cdots, J\}$ and computes the corresponding stochastic gradient at the global model $x^t$. We denote the true stochastic gradient of regular worker $w \in \mathcal{R}$ as $\nabla f_{w, i_w^t}(x^t)$ and the poisoned stochastic gradient of poisoned worker $w \in \mathcal{W}\setminus\mathcal{R}$ as $\nabla \tilde{f}_{w, i_w^t}(x^t)$. Next, the workers update their local momenta and send to the server. For each regular worker $w \in \mathcal{R}$, its true local momentum is
	\begin{align}
		m_w^t = (1 - \alpha) m_w^{t-1} + \alpha \nabla f_{w, i_w^t}(x^t),
	\end{align}
	where $\alpha \in [0, 1]$ is the momentum coefficient and $m_w^{-1}$ is initialized as $\nabla f_{w, i_w^0}(x^0)$. For each poisoned worker $w \in \mathcal{W} \setminus \mathcal{R}$, its poisoned local momentum is
	\begin{align}
		\tilde{m}_w^t = (1 - \alpha) \tilde{m}_w^{t-1} + \alpha \nabla \tilde{f}_{w, i_w^t} (x^t),
	\end{align}
	where $\tilde{m}_w^{-1}$ is initialized as $\nabla \tilde{f}_{w, i_w^0}(x_w^0)$.
	For notational convenience, we denote the message sent by worker $w$, no matter true or poisoned, as

	\begin{align}
		\hat{m}_w^t = \left\{
		\begin{aligned}
			& m_w^t, && w \in \mathcal{R}, \\
			& \tilde{m}_w^t, && w \in \mathcal{W}\setminus\mathcal{R}.
		\end{aligned}
		\right.
	\end{align}

	Upon receiving all messages $\{ \hat{m}_w^t: w\in\mathcal{W}\}$, the server may choose to aggregate them with a robust aggregator $\text{RAgg}(\cdot)$ and then move a step along the negative direction, as
	\begin{align}\label{alg: distributed-SGD-with-agg}
		x^{t+1} = x^t - \gamma \cdot \text{RAgg}(\{ \hat{m}_w^t : w\in\mathcal{W}\}),
	\end{align}
	where $\gamma>0$ is the step size. State-of-the-art robust aggregators include trimmed mean (TriMean) \citep{chen2017distributed}, FABA \citep{xia2019faba},  centered clipping (CC) \citep{karimireddy2021learning}, to name a few.
	
	In this paper, we argue that the mean aggregator $\text{Mean}(\cdot)$, which is often viewed as vulnerable, is more robust than the state-of-the-art robust aggregators under label poisoning attacks. With the mean aggregator, the update is
	\begin{align}\label{alg: distributed-SGD}
		x^{t+1} = x^t - \gamma \cdot \text{Mean}(\{ \hat{m}_w^t : w\in\mathcal{W}\}),
	\end{align}
	where
	\begin{align}\label{eq: Mean}
		\text{Mean}(\{ \hat{m}_w^t : w\in\mathcal{W}\}) \triangleq \frac{1}{W} \sum_{w\in\mathcal{W}} \hat{m}_w^t.
	\end{align}

	We summarize the distributed stochastic momentum algorithm with different aggregators in Algorithm \ref{algorithm: 1}. Note that when the momentum coefficient $\alpha=1$, it reduces to the popular distributed stochastic gradient descent algorithm. Further, if at each iteration, each worker accesses all $J$ samples to compute the local gradient other than stochastic gradient, the algorithm becomes distributed gradient descent.
	
	\begin{algorithm}
		\caption{}
		\label{algorithm: 1}
		{\bf Input:} Initializations $x^0 \in \R^D$, $m_w^{-1} = \nabla f_{w, i_w^0}(x^0)$ if $w \in \mathcal{R}$, $\tilde{m}_w^{-1} = \nabla \tilde{f}_{w, i_w^0}(x_w^0)$ if $w \in \mathcal{W}\setminus\mathcal{R}$, with $i_w^0$ being uniformly randomly sampled from $\{1, \cdots, J\}$; step size $\gamma $; momentum coefficient $\alpha$; number of overall iterations $T$.
		\begin{algorithmic}[1]
			\For{$t = 0, 1, \cdots, T-1$}
			\State Server broadcasts $x^t$ to all workers.
			\State Regular worker $w \in \mathcal{R}$ uniformly randomly samples $i_w^t$ from $\{1, \cdots, J\}$, computes
			
			$\nabla f_{w, i_w^t}(x^t)$, updates $m_w^t = (1 - \alpha) m_w^{t-1} +  \alpha \nabla f_{w, i_w^t}(x^t)$, and sends $\hat{m}_w^t = m_w^t$ to server.

			\State Poisoned worker $w \in \mathcal{W}\setminus\mathcal{R}$ uniformly randomly samples $i_w^t$ from $\{1, \cdots, J\}$, computes
			
			$\nabla \tilde{f}_{w, i_w^t}(x^t)$, updates $\tilde{m}_w^t = (1 - \alpha) \tilde{m}_w^{t-1} +  \alpha \nabla \tilde{f}_{w, i_w^t}(x^t)$, and sends $\hat{m}_w^t = \tilde{m}_w^t$ to server.
			
			\State Server receives $\{\hat{m}_w^t\}_{w \in \mathcal{W}}$ from all workers and updates $x^{t+1}$ according to \eqref{alg: distributed-SGD-with-agg} or \eqref{alg: distributed-SGD}.

			\EndFor \State \textbf{end for}

		\end{algorithmic}
	\end{algorithm}

	\section{Convergence Analysis }
	\label{sec: convergence analysis}
	In this section, we analyze the learning errors of Algorithm \ref{algorithm: 1} with different aggregators under label poisoning attacks. We make the following assumptions. For regular worker $w \in \mathcal{R}$, we respectively denote the true gradients of the local cost and the $j$-th sample cost as $\nabla f_w(\cdot)$ and $\nabla f_{w,j}(\cdot)$. For poisoned worker $w \in \mathcal{W} \setminus \mathcal{R}$, we respectively denote the poisoned gradients of the local cost and the $j$-th sample cost as $\nabla \tilde{f}_w(\cdot)$ and $\nabla \tilde{f}_{w,j}(\cdot)$.
	\begin{assumption}[\textbf{Lower boundedness}]\label{Assump: lower boundedness}
		The global cost $f(\cdot)$ is lower bounded by $f^* $, i.e., $f(x) \geq f^*$.	
	\end{assumption}
	\begin{assumption}[\textbf{Lipschitz continuous gradients}]\label{Assump: L-smooth}
		The global cost $f(\cdot)$ has $L$-Lipschitz continuous gradients. That is, for any $x, y \in \R^D$, it holds that
		\begin{align}
			\|\nabla f(x) - \nabla f(y)\| \leq L \|x - y\|.
		\end{align}
	\end{assumption}
	\begin{assumption}[\textbf{Bounded heterogeneity}]\label{Assump: bounded heterogeneity}
		For any $x \in \R^D$, the maximum distance between the local gradients of any regular worker $w\in \mathcal{R}$ and the global gradient is upper-bounded by $\xi$, i.e.,
		\begin{align}\label{eq: xi}
			\max_{w \in \mathcal{R} } \|\nabla f_w(x) - \nabla f(x)\| \leq \xi.
		\end{align}
	\end{assumption}
	
	\begin{assumption}[\textbf{Bounded inner variance}]\label{Assump: bounded variance}
		For any $x \in \R^D$, the variance of the local stochastic gradients of any worker $w \in \mathcal{W}$ with respect to the local gradient is upper-bounded by $\sigma^2$, i.e.,
		\begin{align}\label{eq: sigma}
			&\E_{i_w}\|\nabla f_{w, i_w}(x) - \nabla f_w(x)\|^2 \leq \sigma^2, \quad \forall w \in \mathcal{R}, \\
			&\E_{i_w}\|\nabla \tilde{f}_{w, i_w}(x) - \nabla \tilde{f}_w(x)\|^2 \leq \sigma^2, \quad \forall w \in \mathcal{W} \setminus \mathcal{R},
		\end{align}
		where $i_w$ denotes a sample index uniformly randomly selected from $\{1, \cdots, J\}$.
	\end{assumption}

	Assumptions \ref{Assump: lower boundedness}, \ref{Assump: L-smooth}, \ref{Assump: bounded heterogeneity} and \ref{Assump: bounded variance} are all common in the analysis of distributed first-order stochastic algorithms. In particular, Assumption \ref{Assump: bounded heterogeneity} characterizes the heterogeneity of the distributed data across the regular workers; larger $\xi$ means higher heterogeneity. Assumption \ref{Assump: bounded variance} is made for both regular and poisoned workers. This is reasonable since the poisoned workers only poison their local labels while keep local features clean, such that the variances of poisoned local stochastic gradients do not drastically change. We will  validate Assumption \ref{Assump: bounded variance} with numerical experiments in Appendix \ref{sec: Appendix G}.

	\begin{assumption}[\textbf{Bounded disturbances of poisoned local gradients}]\label{Assump: bounded effect of poisoned local gradients}
		For any $x \in \R^D$, the maximum distance between the poisoned local gradients of poisoned workers $w \in \mathcal{W}\setminus\mathcal{R}$ and the global gradient is upper-bounded by $A$, i.e.,
		\begin{align}\label{eq: bound of A-2}
			\max_{w \in \mathcal{W}\setminus\mathcal{R}} \|\nabla \tilde{f}_{w}(x) - \nabla f(x)\| \leq A.
		\end{align}
	\end{assumption}
	
	Assumption \ref{Assump: bounded effect of poisoned local gradients} bounds the disturbances caused by the poisoned workers. This assumption does not hold for the worst-case Byzantine attacks model, where the disturbances caused by the Byzantine workers can be arbitrary. However, under label poisoning attacks, we prove that this assumption holds for distributed softmax regression as follows. We will also demonstrate with numerical experiments that this assumption holds naturally in training neural networks.
	
	\subsection{Justification of Assumption \ref{Assump: bounded effect of poisoned local gradients}}
	\label{subsec: reasonableness of bounded gradient}

	\textbf{Example: Distributed softmax regression under label poisoning attacks.} Distributed softmax regression is common for classification tasks, where the local cost of worker $w \in \mathcal{R}$ is in the form of

	\begin{align}\label{ex1: local cost function of softmax regression}
		f_{w}(x) = \frac{1}{J}\sum_{j=1}^{J} f_{w, j}(x), ~ \text{where} ~
		f_{w, j}(x) =  - \sum_{k=1}^{K} \bm{1}\{b^{(w, j)} = k\} \log \frac{\exp(x_k^Ta^{(w, j)})}{\sum_{l=1}^K \exp(x_l^T a^{(w, j)})}.
	\end{align}
	In \eqref{ex1: local cost function of softmax regression}, $K$ stands for the number of classes; $ (a^{(w, j)}, b^{(w, j)})$ represents the $j$-th sample of worker $w$ with $a^{(w, j)} \in \R^d$ and $b^{(w, j)} \in \R$ being the feature and the label, respectively; $\bm{1}\{b ^{(w, j)}= k\}$ is the indicator function that outputs $1$ if $b^{(w, j)} =k$ and $0$ otherwise; $x_k \triangleq [x]_{kd: (k+1)d} \in \R^d$ is the $k$-th block of $x$.
	
	Note that for poisoned worker $w\in\mathcal{W}\setminus\mathcal{R}$, the labels are possibly changed from $b^{(w, j)}$ to $\tilde{b}^{(w,j)}$ for all $j\in\{1, \cdots, J\}$. Therefore, the local cost of worker $w \in \mathcal{W}\setminus\mathcal{R}$ is in the form of
	\begin{align}\label{ex1: local cost function of softmax regression-p}
		\tilde{f}_{w}(x) = \frac{1}{J}\sum_{j=1}^{J} \tilde{f}_{w, j}(x), ~ \text{where} ~
		\tilde{f}_{w, j}(x) =  - \sum_{k=1}^{K} \bm{1}\{\tilde{b}^{(w, j)} = k\} \log \frac{\exp(x_k^Ta^{(w, j)})}{\sum_{l=1}^K \exp(x_l^T a^{(w, j)})}.
	\end{align}

	It is straightforward to verify that the global cost $f(x) $ with the local costs $f_w(x)$ in \eqref{ex1: local cost function of softmax regression} satisfies Assumptions \ref{Assump: lower boundedness} and \ref{Assump: L-smooth}. Since the gradients of local costs $f_w(x)$ in \eqref{ex1: local cost function of softmax regression} are bounded (see Lemma \ref{lemma:bounded gradient in softmax regression} in Appendix \ref{sec: Appendix A}), the global cost $f(x)$ satisfies Assumptions \ref{Assump: bounded heterogeneity} and $\xi$ refers to the heterogeneity of the local costs $f_w(x)$. Further, since the gradients of the regular sample costs $f_{w, j}(\cdot)$ and the poisoned sample costs $\tilde{f}_{w, j}(\cdot)$ are all bounded (see Lemma \ref{lemma:bounded gradient in softmax regression} in Appendix \ref{sec: Appendix A}), the local cost of any worker $w \in \mathcal{W}$ satisfies Assumption \ref{Assump: bounded variance}. Next, we show that Assumption \ref{Assump: bounded effect of poisoned local gradients} also holds.
	
	\begin{lemma}\label{lemma: bound of softmax regression}
		Consider the distributed softmax regression problem where the local costs of the workers are in the forms of \eqref{ex1: local cost function of softmax regression} and \eqref{ex1: local cost function of softmax regression-p}. Therein, the poisoned workers are under label poisoning attacks, with arbitrary fractions of sample labels being poisoned. If $a^{(w,j)}$ is entry-wise non-negative 
		for all $w \in \mathcal{W}$ and all $j \in \{1, \cdots, J\}$, then Assumption \ref{Assump: bounded effect of poisoned local gradients} is satisfied with
		\begin{align}\label{eq: A}
			A  \leq  2\sqrt{K} \max_{w \in \mathcal{W}} \|\frac{1}{J} \sum_{j=1}^{J} a^{(w, j)}\|.
		\end{align}
	\end{lemma}
	
	\begin{proof}
			See Appendix \ref{sec: Appendix A.2}.
	\end{proof}
	
	Lemma \ref{lemma: bound of softmax regression} explicitly gives the upper bound of the smallest possible $A$ in Assumption \ref{Assump: bounded effect of poisoned local gradients}. Observe that the non-negativity assumption of $a^{(w, j)}$ naturally holds; for example, in image classification tasks, each entry of the feature stands for a pixel value. For other tasks, we can shift the features to meet this requirement.
	
	\noindent \textbf{Relation between Assumptions \ref{Assump: bounded heterogeneity} and \ref{Assump: bounded effect of poisoned local gradients}.} Interestingly, the constants $\xi$ and $A$ in Assumptions \ref{Assump: bounded heterogeneity} and \ref{Assump: bounded effect of poisoned local gradients} are tightly related. Similar to Lemma \ref{lemma: bound of softmax regression} that gives the upper bound of the smallest possible $A$ in Assumption \ref{Assump: bounded effect of poisoned local gradients}, for the distributed softmax regression problem, we can give the upper bound of the smallest possible $\xi$ in Assumption \ref{Assump: bounded heterogeneity} as follows.
	
	\begin{lemma}\label{lemma: upper bound of heterogeneity of softmax regression}
		Consider the distributed softmax regression problem where the local costs of the regular workers are in the form of \eqref{ex1: local cost function of softmax regression}. If $a^{(w,j)}$ is entry-wise non-negative for all $w \in \mathcal{R}$ and all $j \in \{1, \cdots, J\}$, then Assumption \ref{Assump: bounded heterogeneity} is satisfied with
		\begin{align}\label{eq: upper bound of xi}
			\xi \leq  2\sqrt{K} \max_{w \in \mathcal{R}}  \|\frac{1}{J} \sum_{j=1}^J a^{(w, j)}\|.
		\end{align}

	\end{lemma}
	
		\begin{proof}
			See Appendix \ref{sec: Appendix A.3}.
	\end{proof}
	
	In particular, in the sufficiently heterogeneous case that each regular worker only has the samples from one class and the samples from one class only belong to one regular worker, the constant $\xi$ is in the same order of $ \max_{w \in \mathcal{R}}  \|\frac{1}{J} \sum_{j=1}^J a^{(w, j)}\|$ (see Lemma \ref{lemma: lower and upper bound of heterogeneity of softmax regression} in Appendix \ref{sec: Appendix A}). Further, if the feature norms of the regular and poisoned workers have similar magnitudes, which generally holds in practice, then $\max_{w \in \mathcal{R}}  \|\frac{1}{J} \sum_{j=1}^J a^{(w, j)}\|$ is in the same order as $\max_{w \in \mathcal{W}}  \|\frac{1}{J} \sum_{j=1}^J a^{(w, j)}\|$. Hence, we can conclude that $A = O(\xi)$ when the distributed data are sufficiently heterogeneous. This conclusion will be useful in our ensuing analysis.

	\subsection{Main Results}
	
	To analyze the learning errors of Algorithm \ref{algorithm: 1} with robust aggregators, we need to characterize the approximation abilities of the robust aggregators, namely, how close their outputs are to the average of the messages from the regular workers. This gives rise to the definition of $\rho$-robust aggregator \citep{wu2023byzantine,dong2023byzantine}.
	\begin{definition}[\textbf{$\rho$-robust aggregator}]\label{Def: (delta_max, rho)-RAR}
		Consider any $W$ messages $y_1, y_2, \cdots, y_W \in \R^D$, among which $R$ messages are from regular workers $w \in \mathcal{R}$. An aggregator $\text{RAgg}(\cdot)$ is said to be a $\rho$-robust aggregator if there exists a contraction constant $\rho \geq 0$ such that
		\begin{align}
			\| \text{RAgg}(\{y_1, \cdots, y_W\}) - \bar{y}\| \leq \rho \cdot \max_{w \in \mathcal{R}} \|y_w - \bar{y} \|,
		\end{align}	
		where $\bar{y} = \frac{1}{R} \sum_{w \in \mathcal{R}} y_w$ is the average message of the regular workers.
	\end{definition}

	From Definition \ref{Def: (delta_max, rho)-RAR}, a small contraction constant $\rho$ means that the output of the robust aggregator is close to the average of the messages from the regular workers. The error is proportional to the heterogeneity of the messages from the regular workers, characterized by $\max_{w \in \mathcal{R}} \|y_w - \bar{y}\|$.
	
	However, since a robust aggregator cannot distinguish the regular and poisoned workers, $\rho$ is unable to be arbitrarily close to $0$. Additionally, when the messages from the poisoned workers are majority, there is no guarantee to satisfy Definition \ref{Def: (delta_max, rho)-RAR}. Therefore, we have the following lemma.
	
	\begin{lemma}\label{lemma: lower bound of rho}
		Denote $\delta \triangleq 1-\frac{R}{W}$ as the fraction of the poisoned workers. Then a $\rho$-robust aggregator exists only if $\delta<\frac12$ and $\rho\ge\min\{\frac{\delta}{1 - 2\delta},1\}$.
	\end{lemma}
	
		\begin{proof}
			See Appendix \ref{sec: Appendix B.1}.
	\end{proof}
	
	We prove that several state-of-the-art robust aggregators, such as TriMean \citep{chen2017distributed}, CC \citep{karimireddy2021learning} and FABA \citep{xia2019faba}, all satisfy Definition \ref{Def: (delta_max, rho)-RAR} when the fraction of poisoned workers is below their respective thresholds. Their corresponding contraction constants $\rho$ are given in Appendix \ref{sec: Appendix B}.
	\begin{remark}
		Our definition is similar to $(f, \kappa)$-robustness in \citet{allouah2023fixing}, while our heterogeneity measure is $\max_{w \in \mathcal{R}} \|y_w - \bar{y}\|$ instead of $\frac{1}{R}\sum_{w \in \mathcal{R}} \|y_w - \bar{y}\|^2$. Due to the fact $\max_{w \in \mathcal{R}} \|y_w - \bar{y}\|^2 \leq \sum_{w \in \mathcal{R}} \|y_w - \bar{y}\|^2$, our definition implies $(f, \kappa)$-robustness in \citet{allouah2023fixing}. Further, according to Propositions 8 and 9 in \citet{allouah2023fixing}, our definition also implies $(f, \lambda)$-resilient averaging and $(\delta_{\max}, c)$-ARAgg in \citet{farhadkhani2022byzantine} and \citet{karimireddy2021byzantine}, respectively. The lower bound of $\rho$ is determined by the fraction of the poisoned workers $\delta$. A smaller $\delta$ leads to a smaller lower bound of $\rho$, which aligns with our intuition. Similar results can be found in \citet{farhadkhani2022byzantine} and \citet{allouah2023fixing}.
	\end{remark}

	Thanks to the contraction property in Definition \ref{Def: (delta_max, rho)-RAR}, we can prove that the learning error of Algorithm \ref{algorithm: 1} with a $\rho$-robust aggregator is bounded under label poisoning attacks.

	\begin{theorem}\label{thm:convergence with RAgg}
		Consider Algorithm \ref{algorithm: 1} with a $\rho$-robust aggregator $\text{RAgg}(\cdot)$ to solve \eqref{problem} and suppose that Assumptions \ref{Assump: lower boundedness}, \ref{Assump: L-smooth}, \ref{Assump: bounded heterogeneity}, and \ref{Assump: bounded variance} hold.
		Under label poisoning attacks where the fraction of poisoned workers is $\delta \in [0, \frac{1}{2})$, if the step size is $\gamma = \min\Big\{O\Big(\sqrt{\frac{ LF^0 + \rho^2 \sigma^2}{TL^2 \sigma^2(\rho^2 + 1)}}\Big), \frac{1}{8L}\Big\} $,
		the momentum coefficient is $\alpha = 8L\gamma$, then we have
		\begin{align}\label{eq: thm1}
			& \frac{1}{T} \sum_{t=1}^{T} \E\|\nabla f(x^t)\|^2 \\
			= & O\Bigg(\rho^2 \xi^2 + \sqrt{\frac{( L F^0 + \rho^2 \sigma^2)(\rho^2 + 1)\sigma^2}{T}} + \frac{ LF^0 + (\rho^2 + 1)\sigma^2 + \rho^2 \xi^2}{T}\Bigg). \notag
		\end{align}
		where the expectation is taken over the algorithm's randomness and $F^0 \triangleq f(x^0) - f^*$.
	\end{theorem}
	
	\begin{proof}
			See Appendix \ref{sec: Appendix C}.
	\end{proof}

	Interestingly, we are also able to prove that under label poisoning attacks, Algorithm \ref{algorithm: 1} with the mean aggregator has a bounded learning error.

	\begin{theorem}\label{thm:convergence with Mean}
		Consider Algorithm \ref{algorithm: 1} with the mean aggregator $\text{Mean}(\cdot)$ to solve \eqref{problem} and suppose that Assumptions \ref{Assump: lower boundedness}, \ref{Assump: L-smooth}, \ref{Assump: bounded variance}, and \ref{Assump: bounded effect of poisoned local gradients} hold.
		Under label poisoning attacks where the fraction of poisoned workers is $\delta \in [0, 1)$, if the step size is $ \gamma = \min \Big\{ O\Big(\sqrt{\frac{L F^0 + \delta^2 \sigma^2}{TL^2 \sigma^2 (\delta^2 + 1)}}\Big), \frac{1}{8L}\Big\}$, the momentum coefficient is $\alpha = 8L\gamma$, then we have
		\begin{align}\label{eq: convergence error of Mean}
			& \frac{1}{T} \sum_{t=1}^{T} \E\|\nabla f(x^t)\|^2 \\
			= & O\Bigg( \delta^2 A^2 + \sqrt{\frac{( LF^0 + \delta^2 \sigma^2)(\delta^2 + 1) \sigma^2}{T}} + \frac{L F^0 + (\delta^2 + 1) \sigma^2 + \delta^2 A^2}{T}\Bigg). \notag
		\end{align}
		where the expectation is taken over the algorithm's randomness and $F^0 \triangleq f(x^0) - f^*$.
	\end{theorem}
	
		\begin{proof}
			See Appendix \ref{sec: Appendix D}.
	\end{proof}
	
	Theorems \ref{thm:convergence with RAgg} and \ref{thm:convergence with Mean} demonstrate that Algorithm \ref{algorithm: 1} with both $\rho$-robust aggregators and the mean aggregator can sublinearly converge to neighborhoods of a first-order stationary point of \eqref{problem}, while the non-vanishing learning errors are $O(\rho^2\xi^2)$ for $\rho$-robust aggregators and $O(\delta^2 A^2)$ for the mean aggregator. It is worth noting that the constants within $O(\rho^2 \xi^2)$ for the $\rho$-robust aggregator and $O(\delta^2 A^2)$ for the mean aggregator are the same (see Theorem \ref{complete version of convergence with RAgg} in Appendix \ref{sec: Appendix C} and Theorem \ref{complete version of convergence with Mean} in Appendix \ref{sec: Appendix D}). We omit these constants here to present the theoretical results concisely. Observe that without the $O(\rho^2\xi^2)$ and $O(\delta^2 A^2)$ terms, the $O(\frac{1}{\sqrt{T}})$ convergence rates are optimal for first-order nonconvex stochastic optimization algorithms \citep{arjevani_lower_2023}.
	
	Before comparing the learning errors in Theorems \ref{thm:convergence with RAgg} and \ref{thm:convergence with Mean}, we first give the lower bound of the learning error for a class of identity-invariant algorithms.

	\begin{theorem}\label{thm:lower bound}
		Under label poisoning attacks with $\delta = 1-\frac{R}{W}$ fraction of poisoned workers, consider any algorithm running for $T$ iterations and generating $x^t$ at iteration $t$. Suppose that the output of the algorithm is invariant with respect to the identities of the workers. Then, there exist $R$ regular local functions $\{f_w(x): {w \in \mathcal{R}}\}$ and $W-R$ poisoned local functions $\{\tilde{f}_w(x): {w \in \mathcal{W}\setminus\mathcal{R}}\}$, which are all composed of $J$ sample costs $f_{w, j}(x)$ or $\tilde{f}_{w, j}(x)$ respectively, satisfying Assumptions \ref{Assump: lower boundedness}, \ref{Assump: L-smooth},  \ref{Assump: bounded heterogeneity}, \ref{Assump: bounded variance}, and \ref{Assump: bounded effect of poisoned local gradients} such that the iterates $\{x^t: t = 1, \cdots, T\}$ of the algorithm satisfy
		\begin{align}\label{eq: lower bound}
				\frac{1}{T} \sum_{t=1}^{T} \E\|\nabla f(x^t)\|^2 =  \Omega(\delta^2 \min\{A^2, \xi^2\}).
		\end{align}
				where the expectation is taken over the algorithm's randomness.
			\end{theorem}
			
				\begin{proof}
					See Appendix \ref{sec: Appendix E}.
			\end{proof}
			
		The identity-invariant property in Theorem \ref{thm:lower bound} means that, given any $W$ local costs of which $R$ are regular, the output of the algorithm is invariant with respect to which local costs are regular or poisoned. This property excludes the ``omniscient'' algorithms that know the identities of the workers and thus can exclude the local costs of the poisoned workers. In fact, all practical algorithms are identity-invariant, including Algorithm 1 with any $\rho$-robust aggregator or the mean aggregator. In \citet{karimireddy2021byzantine} and \citet{allouah2023fixing}, the authors implicitly confine their analyses to the algorithms that share the same identity-invariant property as stated in Theorem \ref{thm:lower bound}, resulting in comparable lower bounds to ours. On the other hand, in \citet{karimireddy2021learning}, the authors establish a lower bound for the algorithms with a stronger iteration-wise permutation-invariant property, which excludes our distributed stochastic momentum algorithm.

For any identity-invariant algorithm, it is impossible to fully eliminate the effect of malicious attacks, which results in the non-vanishing learning error as demonstrated in Theorem \ref{thm:lower bound}. When the disturbances caused by the poisoned workers are small such that  $A < \xi$, the lower bound in \eqref{eq: lower bound} becomes $\Omega(\delta^2A^2)$, matching the learning error of Algorithm \ref{algorithm: 1} with the mean aggregator in \eqref{eq: convergence error of Mean}. It implies that the learning error of the mean aggregator is order-optimal in the presence of small disturbances. On the other hand, when the disturbances caused by the poisoned worker are large such that $A \geq \xi$, the lower bound in \eqref{eq: lower bound} becomes $\Omega(\delta^2\xi^2)$. In Table \ref{Table:delta_max, rho}, we compare the learning errors for different aggregators, given large heterogeneity such that $A$ is at most the same order as $\xi$ (which holds when the distributed data are sufficiently heterogeneous, as we have discussed in Section \ref{subsec: reasonableness of bounded gradient}). 
			\begin{table}[!h]
				\centering
				\begin{tabular}{|c|c|c|}
					\hline
					&\\ [-8pt]
					Aggregator & Learning error \\
					\hline
					&\\ [-10pt]
					TriMean & $O (\frac{\delta^2\xi^2}{(1 - 2\delta)^2}) $ \\
					\hline
					&\\ [-8pt]
					CC  &$O( \delta\xi^2)$ \\[2pt]
					\hline
					&\\ [-10pt]
					FABA  &$O (\frac{\delta^2\xi^2}{(1 - 3\delta)^2}) $ \\
					\hline
					&\\ [-8pt]
					Mean & $O(\delta^2 \xi^2)$  \\
					\hline
					\hline
					&\\ [-8pt]
					Lower bound   & $\Omega(\delta^2 \xi^2)$ \\
					\hline
				\end{tabular}
				\caption{Learning errors of Algorithm \ref{algorithm: 1} with TriMean, CC, FABA and the mean aggregator when the heterogeneity of distributed data is sufficiently large such that $A$ is in the same order as $\xi$. The lower bound of the learning error is also given.
				}
				\label{Table:delta_max, rho}
			\end{table}
			
			According to Table \ref{Table:delta_max, rho}, we know that the learning errors of TriMean, FABA and the mean aggregator all match the lower bound in terms of the order, when $\delta$ is small. However, the learning errors of TriMean and FABA explode when $\delta$ approaches $\frac{1}{2}$ and $\frac{1}{3}$, respectively, while the mean aggregator is insensitive. Therefore, the learning error of the mean aggregator is order-optimal regardless of the fraction of poisoned workers. In addition, the learning error of the mean aggregator is smaller than that of CC by a magnitude of $\delta$.
			
					\begin{remark}
						\label{remak-10}
						Note that our backbone algorithm, distributed stochastic momentum, degenerates to distributed stochastic gradient descent (when $\alpha = 1$) and distributed gradient descent (when $\alpha = 1$ and $\sigma = 0$). Following our analysis, the non-vanishing learning errors of distributed stochastic gradient descent, when using a $\rho$-robust aggregators and the mean aggregator, are $O(\rho^2 \sigma^2 + \rho^2 \xi^2)$ and $O(\delta^2 \sigma^2 + \delta^2 A^2)$ respectively (substituting $\alpha = 1$ to the proofs of Theorems \ref{thm:convergence with RAgg} and \ref{thm:convergence with Mean} yields these results). When the data heterogeneity term $\xi^2$ and the disturbance term $A^2$ both dominate the variance of the stochastic gradients $\sigma^2$, our conclusions made in this paper remain valid. Further letting $\sigma = 0$, the non-vanishing learning errors of distributed gradient descent, when using
						a $\rho$-robust aggregator and the mean aggregator, are $O(\rho^2 \xi^2)$ and $O(\delta^2 A^2)$ respectively. This way, we can obtain the same results in Table \ref{Table:delta_max, rho}, as shown in \citet{peng2024mean}.
					\end{remark}

					\section{Numerical Experiments}
					\label{sec: experiments}
					In this section, we conduct numerical experiments to validate our theoretical findings and demonstrate the performance of Algorithm \ref{algorithm: 1} with the mean and robust aggregators under label poisoning attacks. The code is available at \url{https://github.com/pengj97/LPA}.
					
						\begin{table}[t]
						\centering
						\setlength{\tabcolsep}{2mm}{\begin{tabular}{|c|c|}
								\hline
								& \\ [-8pt]
								Layer Name & Layer size \\\hline
								& \\ [-8pt]
								Convolution + ReLU & 3 $\times$ 3 $\times$ 16 \\ \hline
								& \\ [-8pt]
								Max pool & 2 $\times$ 2 \\ \hline
								& \\ [-8pt]
								Convolution + ReLU & 3 $\times$ 3 $\times$ 32 \\ \hline
								& \\ [-8pt]
								Max pool & 2 $\times$ 2 \\ \hline
								& \\ [-8pt]
								Fully connected + ReLU & 128 \\ \hline
								& \\ [-8pt]
								Softmax & 10 \\ \hline
						\end{tabular}}
						\caption{Architecture of the convolutional neural network trained on the CIFAR10 dataset.}
						\label{table: CNN}
					\end{table}

					\begin{figure}
						\centering
						\includegraphics[scale=0.18]{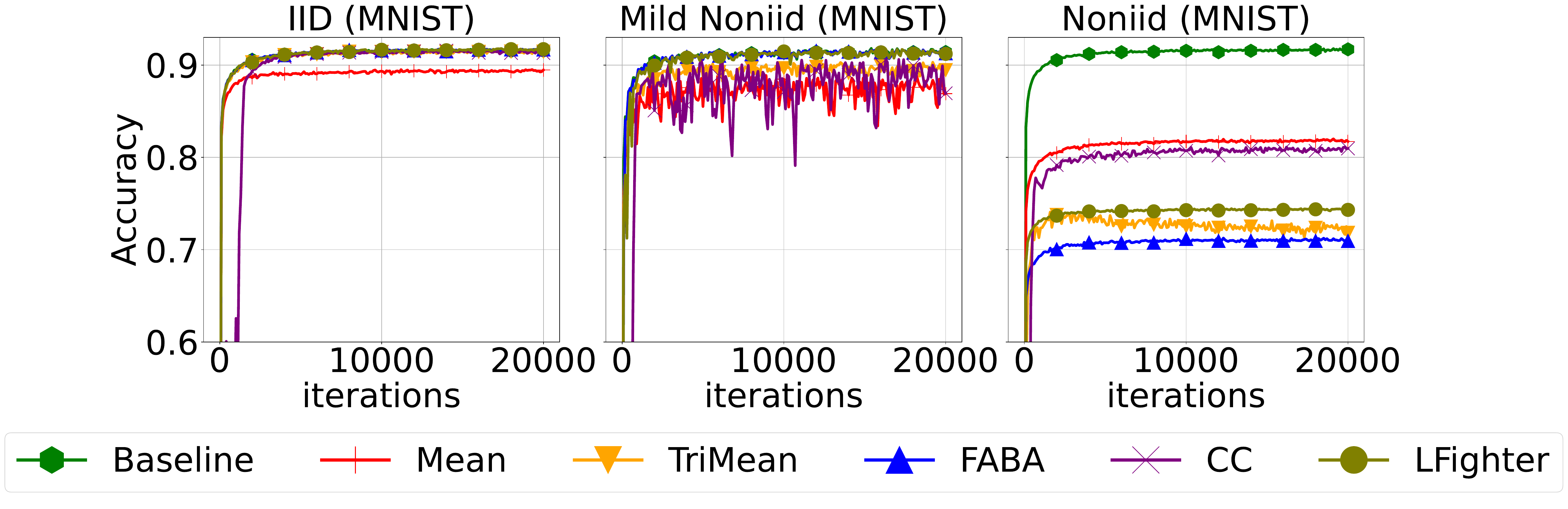}
						\caption{Accuracies of softmax regression on the MNIST dataset under static label flipping attacks.}
						\label{fig:SR_mnist_label_flipping}
					\end{figure}

					\begin{figure}
						\centering
						\includegraphics[scale=0.18]{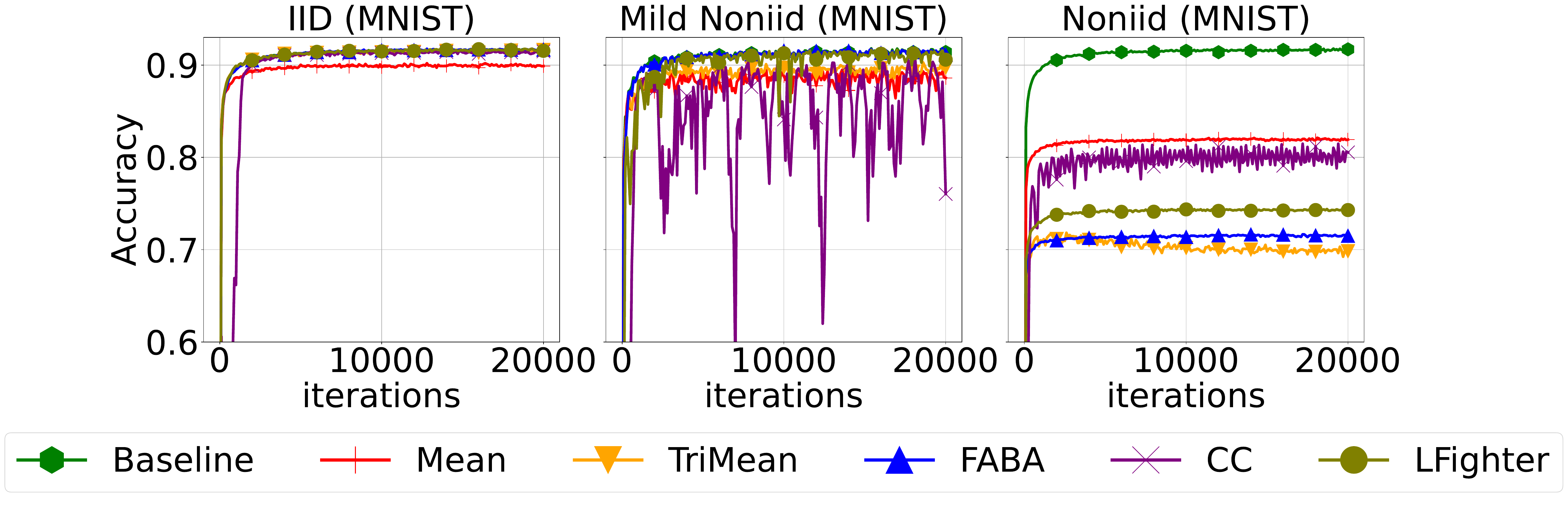}
						\caption{Accuracies of softmax regression on the MNIST dataset under dynamic label flipping attacks.}
						\label{fig:SR_mnist_furthest_label_flipping}
					\end{figure}
					
					\begin{figure}
						\centering
						\hspace{-20pt}
						\includegraphics[scale=0.18]{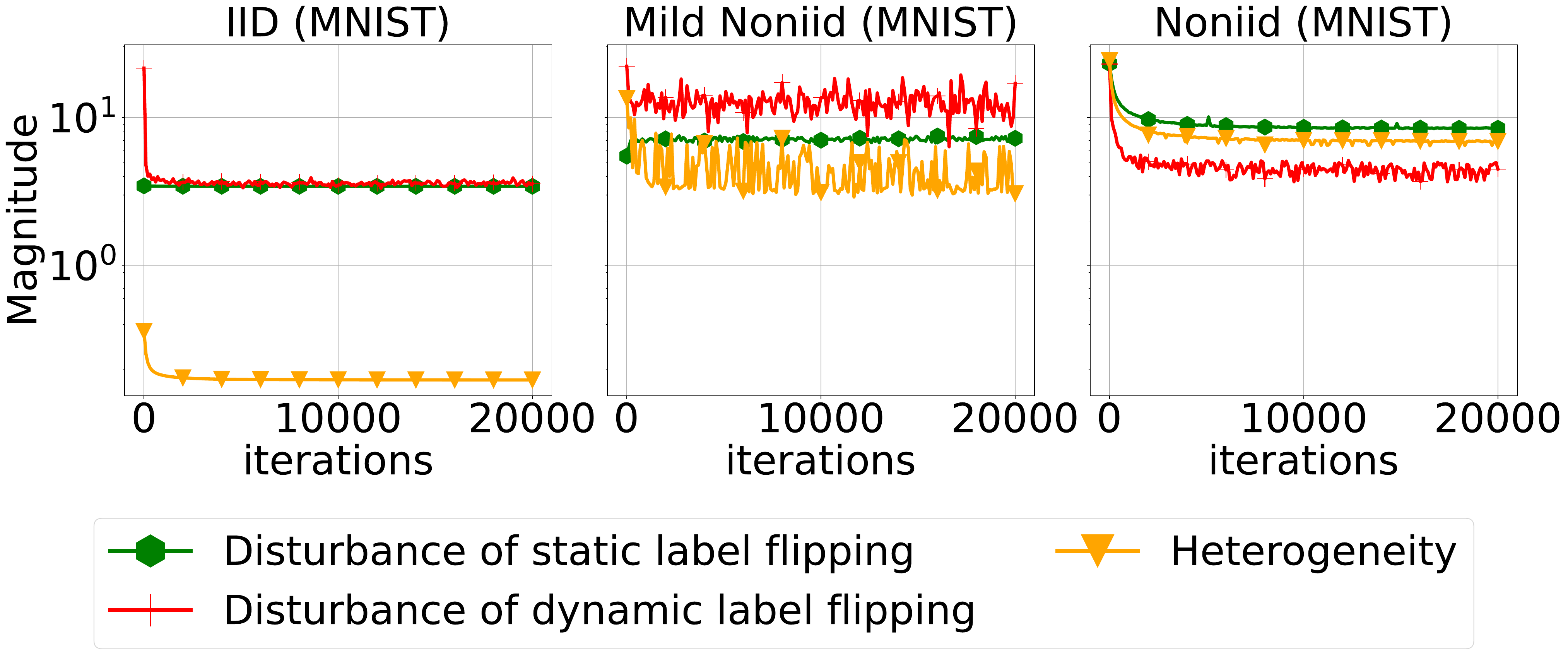}
						\caption{Heterogeneity of regular local gradients (the smallest $\xi$ satisfying Assumption \ref{Assump: bounded heterogeneity}) and disturbance of poisoned local gradients (the smallest $A$ satisfying Assumption \ref{Assump: bounded effect of poisoned local gradients}) in softmax regression on the MNIST dataset, under static label flipping and dynamic label flipping attacks.}
						\label{fig:SR_mnist_A_hetero}
					\end{figure}
					
					\subsection{Experimental Settings}
					\noindent \textbf{Datasets and partitions.} In the numerical experiments, we investigate a convex problem of softmax regression on the MNIST dataset. 
					We also consider two non-convex problems. The first one is to train two-layer perceptrons, in which each layer has 50 neurons and the activation function is ReLU, on the MNIST dataset. The second one is to train convolutional neural networks, whose architecture is given in Table \ref{table: CNN}, on the CIFAR10 dataset\footnote{Although the ReLU function is non-smooth, the training process rarely reaches the non-smooth point. Therefore, the results on the ReLU function are similar to those on a smooth function (which can be obtained by modifying the ReLU function and has Lipschitz continuous gradients).}.

					We setup $W = 10$ workers where $R=9$ workers are regular and the remaining one is poisoned. The impact of different fractions of poisoned workers is demonstrated in Appendix \ref{Impact of fraction of the poisoned workers}.
					We consider three data distributions: i.i.d., mild non-i.i.d. and non-i.i.d. cases. In the i.i.d. case, we uniformly randomly divide the training data among all workers. In the mild non-i.i.d. case, we divide the training data using the Dirichlet distribution with hyper-parameter $\beta=1$ by default \citep{hsu2019measuring}. In the non-i.i.d. case, we assign each class of the training data to one worker.
					
					\noindent \textbf{Label poisoning attacks.} We investigate two types of label poisoning attacks: static label flipping where the poisoned worker flips label $b$ to $9 - b$ with $b$ ranging from $0$ to $9$, and dynamic label flipping where the poisoned worker flips label $b$ to the least probable label with respect to the global model $x^t$ \citep{shejwalkar2022back}.

					\label{subsec: benchmark methods}
					\noindent \textbf{Aggregators to compare.} We are going to compare the mean aggregator with several representative $\rho$-robust aggregators, including TriMean, FABA, CC, and LFighter. The baseline is the mean aggregator without attacks. The step size is $\gamma= 0.01$ and the momentum coefficient is $\alpha = 0.1$.

					\subsection{Convex Case}
					\noindent \textbf{Classification accuracy.} We consider softmax regression on the MNIST dataset. The classification accuracies under static label flipping and dynamic label flipping attacks are shown in Figure \ref{fig:SR_mnist_label_flipping} and Figure \ref{fig:SR_mnist_furthest_label_flipping}, respectively. In the i.i.d. case, all methods perform  well and close to the baseline, but the mean aggregator has an apparently lower classification accuracy. In the mild non-i.i.d. case, FABA and LFighter are the best among all aggregators and the other aggregators have similar performance. In the non-i.i.d. case, since the heterogeneity is large, all aggregators are tremendously affected by the label poisoning attacks, and have gaps to the baseline in terms of classification accuracy. Notably, the mean aggregator performs the best among all aggregators in this case, which validates our theoretical results.

					\noindent \textbf{Heterogeneity of regular local gradients and disturbance of poisoned local gradients.} To further validate the reasonableness of Assumptions \ref{Assump: bounded heterogeneity} and \ref{Assump: bounded effect of poisoned local gradients}, as well as the correctness of our theoretical results in Section \ref{subsec: reasonableness of bounded gradient}, we compute the smallest $\xi$ and $A$ that satisfy Assumptions \ref{Assump: bounded heterogeneity} and \ref{Assump: bounded effect of poisoned local gradients} for the softmax regression problem. As shown in Figure \ref{fig:SR_mnist_A_hetero}, the disturbances of the poisoned local gradients, namely $A$, are bounded under both static label flipping and dynamic label flipping attacks, which corroborates the theoretical results in Lemma \ref{lemma: bound of softmax regression}. From i.i.d., mild non-i.i.d. to the non-i.i.d. case, the heterogeneity of the regular local gradients characterized by $\xi$ increases. Particularly, in the non-i.i.d. case, $\xi$ is close to $A$ under both static label flipping and dynamic  label flipping attacks, which aligns our discussions below Lemma \ref{lemma: upper bound of heterogeneity of softmax regression}. Recall Table \ref{Table:delta_max, rho} that shows when the heterogeneity is in the order of the disturbances caused by the label poisoning attacks, the learning error of the mean aggregator is order-optimal. This explains the results in Figures \ref{fig:SR_mnist_label_flipping} and \ref{fig:SR_mnist_furthest_label_flipping}.

					\subsection{Nonconvex Case}
					
					\noindent \textbf{Classification accuracy.} Next, we train two-layer perceptrons on the MNIST dataset and convolutional neural networks on the CIFAR10 dataset under static label flipping and dynamic label flipping attacks, as depicted in Figures \ref{fig:NeuralNetwork_label_flipping} and \ref{fig:NeuralNetwork_furthest_label_flipping}. In the i.i.d. case, all methods have good performance and are close to the baseline, except for CC that performs worse than the other aggregators on the CIFAR10 dataset under dynamic label flipping attacks. In the mild non-i.i.d. case and on the MNIST dataset, all methods perform well and are close to the baseline. On the CIFAR10 dataset, Mean, FABA and LFighter are the best and close to the baseline, CC and TriMean are worse, while TriMean is the worst and with an obvious gap under dynamic label flipping attacks. In the non-i.i.d. case, all methods are affected by the attacks and cannot reach the same classification accuracy of the baseline, but the mean aggregator is still the best. CC, FABA and LFighter are worse and TriMean fails.

					\noindent\textbf{Heterogeneity of regular local gradients and disturbance of poisoned local gradients.} We also calculate the smallest values of $\xi$ and $A$ satisfying Assumptions \ref{Assump: bounded heterogeneity} and \ref{Assump: bounded effect of poisoned local gradients}, respectively.
					As shown in Figure \ref{fig:NN_A_hetero}, the disturbance of poisoned local gradients measured by $A$ are bounded on the MNIST and CIFAR10 datasets under both static label flipping and dynamic label flipping attacks. From i.i.d., mild non-i.i.d. to the non-i.i.d. case, the heterogeneity of regular local gradients $\xi$ is increasing. In the non-i.i.d. case, $\xi$ is close to $A$.
										\begin{figure}
						\centering
						\includegraphics[scale=0.18]{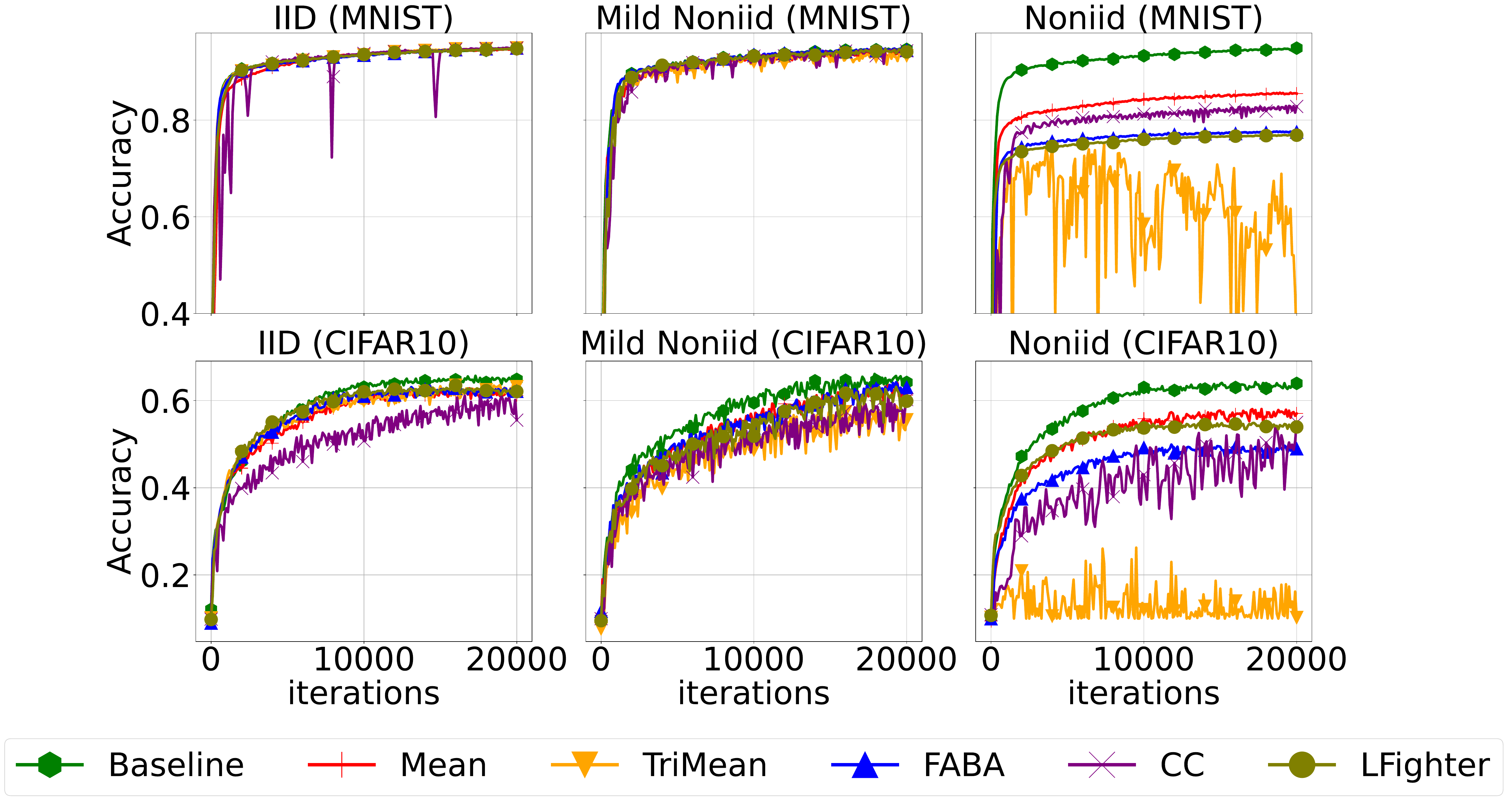}
						\caption{Accuracies of two-layer perceptrons on the MNIST dataset and convolutional neural networks on the CIFAR10 dataset under static label flipping attacks.}
						\label{fig:NeuralNetwork_label_flipping}
					\end{figure}
					\begin{figure}
						\centering
						\includegraphics[scale=0.18]{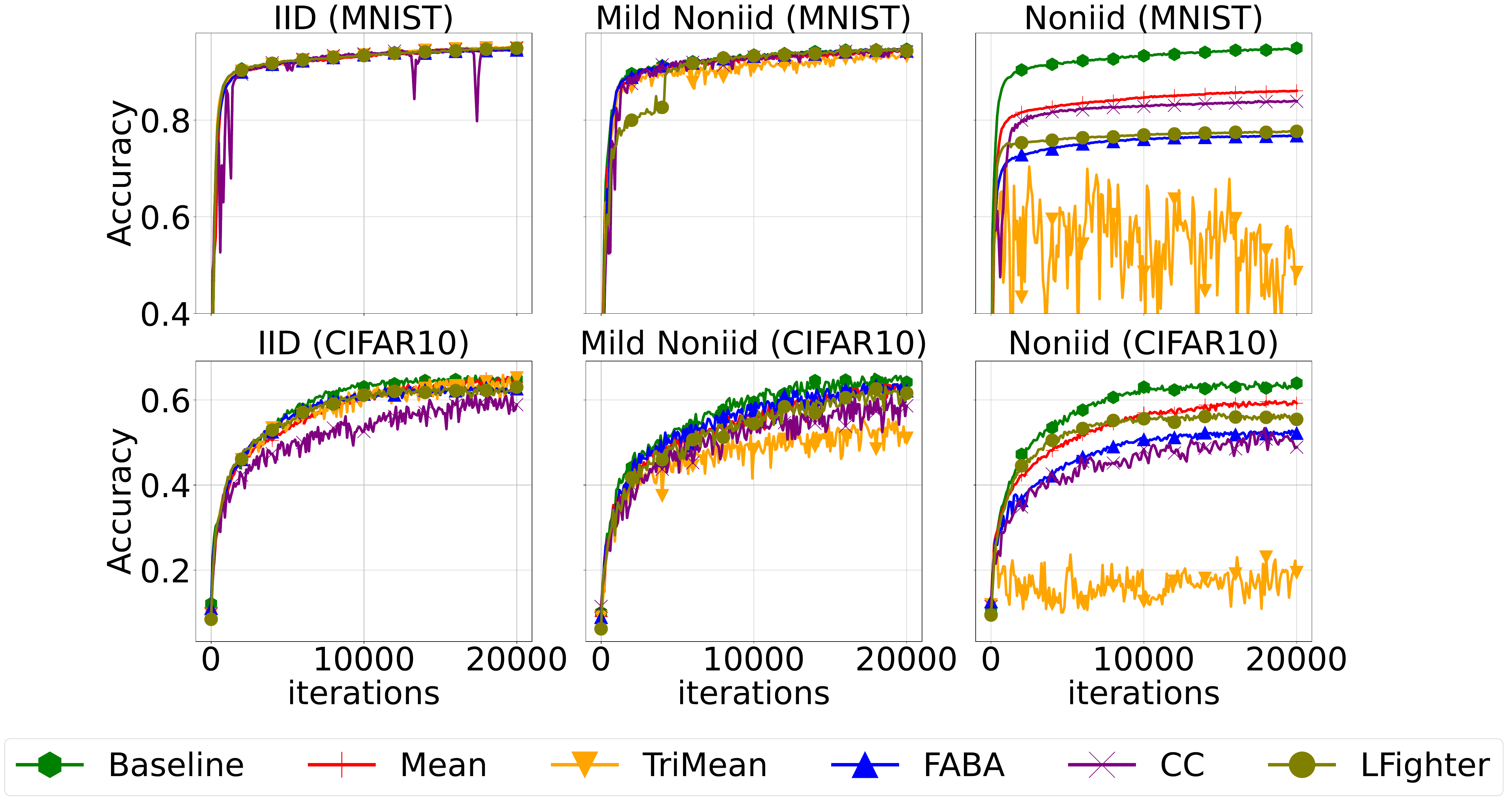}
						\caption{Accuracies of two-layer perceptrons on the MNIST dataset and convolutional neural networks on the CIFAR10 dataset under dynamic label flipping attacks.}
						\label{fig:NeuralNetwork_furthest_label_flipping}
					\end{figure}
					
					\begin{figure}
						\centering
						\hspace{-20pt}
						\includegraphics[scale=0.18]{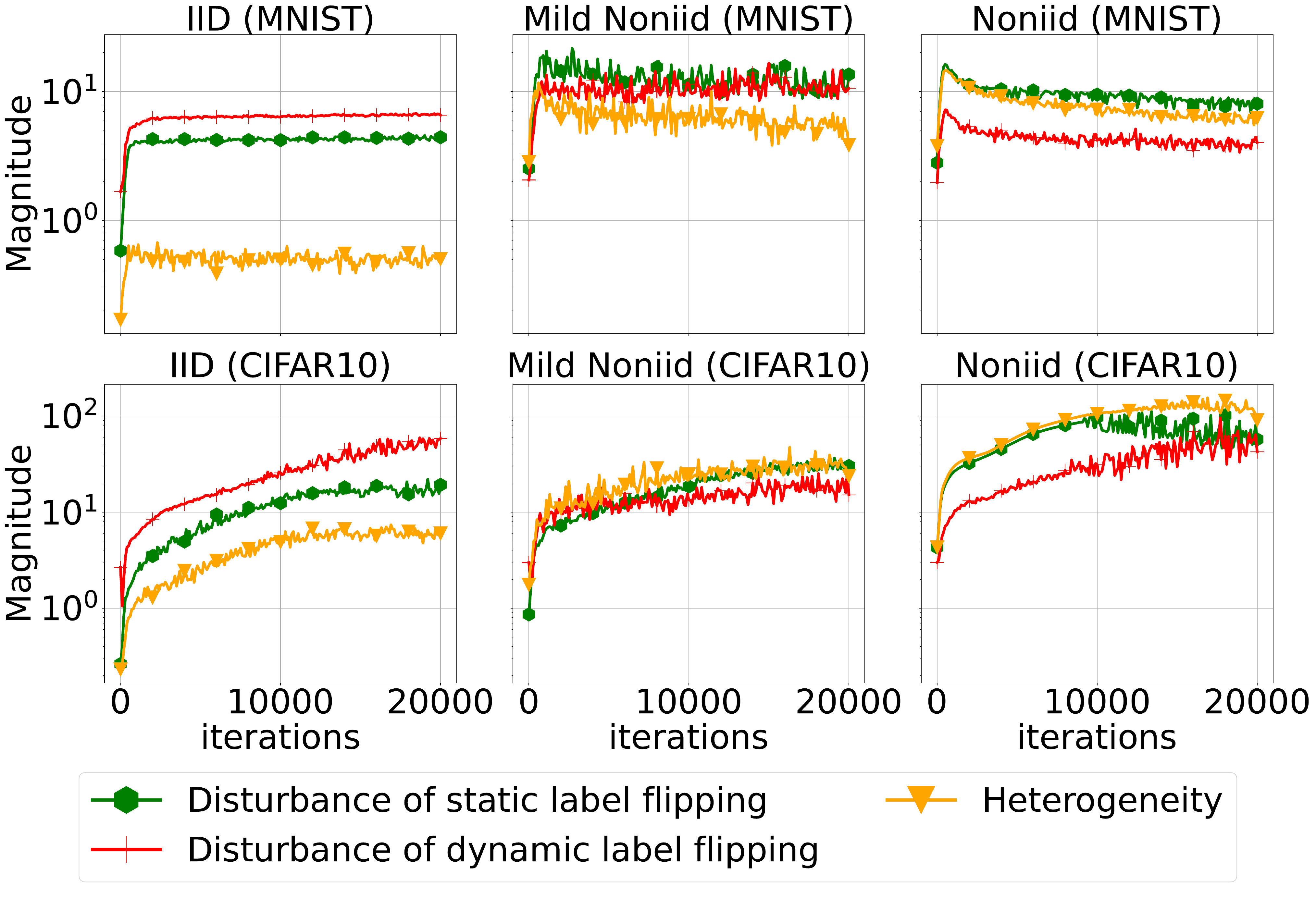}
						\caption{Heterogeneity of regular local gradients (the smallest $\xi$ satisfying Assumption \ref{Assump: bounded heterogeneity}) and disturbance of poisoned local gradients (the smallest $A$ satisfying Assumption \ref{Assump: bounded effect of poisoned local gradients}) in training two-layer perceptrons on the MNIST dataset and training convolutional neural networks on the CIFAR10 dataset, under static label flipping and dynamic label flipping attacks.}
						\label{fig:NN_A_hetero}
					\end{figure}
					
					\begin{figure}
						\centering
						\includegraphics[scale=0.32]{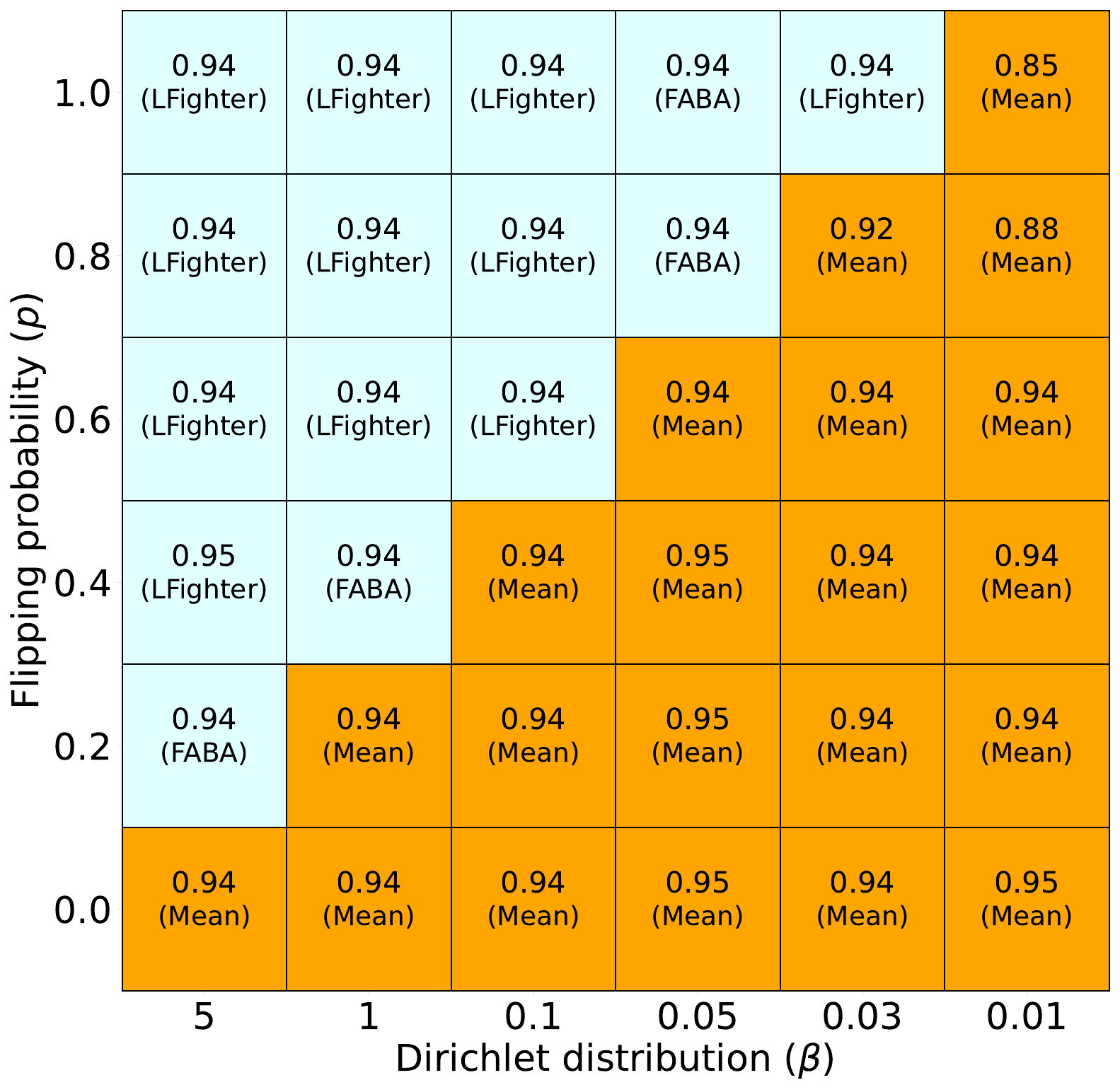}
						\caption{Best accuracies of trained two-layer perceptrons by all aggregators on the MNIST dataset under static label flipping attacks. Each block is associated with a hyper-parameter $\beta$ that characterizes the heterogeneity and the flipping probability $p$ that characterizes the attack strength. For each block, the best accuracy and the corresponding aggregator is marked. Orange means that the mean aggregator is the best.}
						\label{fig:NN_mnist_alpha_prob}
					\end{figure}	
					
					\subsection{Impacts of Heterogeneity and Attack Strengths}
					
					To further show the impacts of heterogeneity of data distributions and strengths of label poisoning attacks, we compute classification accuracies of the trained two-layer perceptrons on the MNIST dataset, varying the data distributions and the levels of label poisoning attacks. We employ the Dirichlet distribution by varying the hyper-parameter $\beta = \{5, 1, 0.1, 0.05, 0.03, 0.01\}$ to simulate various heterogeneity of data distributions, in which a smaller $\alpha$ corresponds to larger heterogeneity \citep{hsu2019measuring}. In addition, we let the poisoned worker apply static label flipping attacks by flipping labels with probability $p = \{0.0, 0.2, 0.4, 0.6, 0.8, 1.0\}$ to simulate different attack strengths. A larger flipping probability indicates stronger attacks.
					
					We present the best performance among all aggregators, and mark the corresponding best aggregator in Figure \ref{fig:NN_mnist_alpha_prob}. More details are in Table 3 of Appendix \ref{sec: Appendix F}. The mean aggregator outperforms the robust aggregators when the heterogeneity is large. For example, the mean aggregator exhibits superior performance when $\beta=0.01$ and the flipping probability $p = \{0.0, 0.2, 0.4, 0.6, 0.8, 1.0\}$, as well as when $\beta = 0.03$ and $p = \{0.0, 0.2, 0.4, 0.6, 0.8\}$. Furthermore, fixing the flipping probability $p$, when the hyper-parameter $\beta$ becomes smaller which means that the heterogeneity becomes larger, the mean aggregator gradually surpasses the robust aggregators. Fixing the hyper-parameter $\beta$, when the flipping probability $p$ becomes smaller which means that the attack strength becomes smaller, the mean aggregator gradually surpasses the robust aggregators. According to the above observations, we recommend to apply the mean aggregator if the distributed data are sufficiently heterogeneous, or the disturbance caused by label poisoning attacks is comparable to the heterogeneity of regular local gradients.

					\section{Conclusions}
					
					We studied the distributed learning problem subject to the label poisoning attacks. We theoretically proved that when the distributed data are sufficiently heterogeneous, the learning error of the mean aggregator is order-optimal. 
					Further corroborated by numerical experiments, our work revealed an important fact that state-of-the-art robust aggregators cannot always outperform the mean aggregator, if the attacks are confined to label poisoning. We expect that this fact can motivate readers to revisit which application scenarios are proper for using robust aggregators. In our future work, we will extend the analysis to the more challenging decentralized learning problem.

					
					

					\appendix
					\newpage
					
					\contentsline {section}{Appendix}{19}{section.6}%
						\contentsline {section}{\numberline {A}Analysis of Distributed Softmax Regression}{20}{appendix.A}%
						\contentsline {subsection}{\numberline {A.1}Bounded Gradients of Local Costs}{20}{subsection.A.1}%
						\contentsline {subsection}{\numberline {A.2}Proof of Lemma \ref {lemma: bound of softmax regression}}{22}{subsection.A.2}%
						\contentsline {subsection}{\numberline {A.3}Proofs of Lemma \ref {lemma: upper bound of heterogeneity of softmax regression} and Its Extension}{23}{subsection.A.3}%
						\contentsline {section}{\numberline {B}Analysis of $\rho $-Robust Aggregators}{24}{appendix.B}%
						\contentsline {subsection}{\numberline {B.1}Proof of Lemma \ref {lemma: lower bound of rho}}{24}{subsection.B.1}%
						\contentsline {subsection}{\numberline {B.2}TriMean}{25}{subsection.B.2}%
						\contentsline {subsection}{\numberline {B.3}CC}{28}{subsection.B.3}%
						\contentsline {subsection}{\numberline {B.4}FABA}{30}{subsection.B.4}%
						\contentsline {section}{\numberline {C}Proof of Theorem \ref {thm:convergence with RAgg}}{33}{appendix.C}%
						\contentsline {section}{\numberline {D}Proof of Theorem \ref {thm:convergence with Mean}}{38}{appendix.D}%
						\contentsline {section}{\numberline {E}Proof of Theorem \ref {thm:lower bound}}{40}{appendix.E}%
						\contentsline {section}{\numberline {F}Impacts of Heterogeneity and Attack Strengths}{43}{appendix.F}%
						\contentsline {section}{\numberline {G}Bounded Variance of Stochastic Gradients}{45}{appendix.G}%
						\contentsline {section}{\numberline {H}Impact of Fraction of Poisoned Workers}{45}{appendix.H}%
					
					
					%

					\newpage

					\section{Analysis of Distributed Softmax Regression}
					\label{sec: Appendix A}
					In this section, we analyze the property of distributed softmax regression where the local cost of worker $w \in \mathcal{W}$ is in the forms of \eqref{ex1: local cost function of softmax regression} and \eqref{ex1: local cost function of softmax regression-p}. We first show that the gradients of the sample costs $f_{w, j}(x)$, $\tilde{f}_{w, j}(x)$ and the local cost $f_w(x)$, $\tilde{f}_w(x)$ are bounded. Then, we prove Lemma \ref{lemma: bound of softmax regression} that provides the constant $A$ in Assumption \ref{Assump: bounded effect of poisoned local gradients}. Last, we prove Lemma \ref{lemma: upper bound of heterogeneity of softmax regression} that gives the constant $\xi$ in Assumption \ref{Assump: bounded heterogeneity}, and further demonstrate that when the distributed data across the regular workers are sufficiently heterogeneous, the constant $\xi$ is in the same order of $\max_{w\in\mathcal{R}} \|\frac{1}{J} \sum_{j=1}^J a^{(w, j)}\|$.

					\subsection{Bounded Gradients of Local Costs}
					\label{sec: Appendix A.1}
					\begin{lemma}\label{lemma:bounded gradient in softmax regression}
						Consider the distributed softmax regression problem where the local cost of worker $w \in \mathcal{W}$ is in the forms of \eqref{ex1: local cost function of softmax regression} and \eqref{ex1: local cost function of softmax regression-p}. Then, the gradients of the sample costs are bounded by the norms of the corresponding sample features, i.e.,
						\begin{align}
							&\|\nabla f_{w, j}(x)\| \leq 2  \|a^{(w, j)}\|, \quad \forall w \in \mathcal{R}, \ \forall  j \in [1, \cdots, J], \\
							&\|\nabla \tilde{f}_{w, j}(x)\| \leq 2  \|a^{(w, j)}\|, \quad \forall w \in \mathcal{W} \setminus \mathcal{R}, \ \forall  j \in [1, \cdots, J].
						\end{align}
						and the gradient of the local cost is bounded by the maximum norm of the local features, i.e.,
						\begin{align}
							&\|\nabla f_w(x)\| \leq 2 \max_{j \in [J]} \|a^{(w, j)}\|, \quad \forall w \in \mathcal{R}, \\
							&\|\nabla \tilde{f}_w(x)\| \leq 2 \max_{j \in [J]} \|a^{(w, j)}\|, \quad \forall w \in \mathcal{W} \setminus \mathcal{R}.
						\end{align}
						Moreover, if $a^{(w,j)}$ is entry-wise non-negative for all $w \in \mathcal{W}$ and all $j \in \{1, \cdots, J\}$, we have
						\begin{align}\label{eq: norm-bound-positive}
							&\|\nabla f_w(x)\| \leq \sqrt{K} \|\frac1J\sum_{j=1}^J  a^{(w, j)} \|, \quad \forall w \in \mathcal{R}, \\
							&\|\nabla \tilde{f}_w(x)\| \leq \sqrt{K} \|\frac1J\sum_{j=1}^J  a^{(w, j)} \|, \quad \forall w \in \mathcal{W} \setminus \mathcal{R}.
						\end{align}
					\end{lemma}
					
					\begin{proof}
						For notational convenience, we denote the local cost of worker $w \in \mathcal{W}$ as
						\begin{align}
							\hat{f}_{w}(x) = \frac{1}{J}\sum_{j=1}^{J} \hat{f}_{w, j}(x),~\text{where}~\hat{f}_{w, j}(x) =  - \sum_{k=1}^{K} \bm{1}\{\hat{b}^{(w, j)} = k\} \log \frac{\exp(x_k^Ta^{(w, j)})}{\sum_{l=1}^K \exp(x_l^T a^{(w, j)})},
						\end{align}
						and
						\begin{align}
							\hat{b}^{(w, j)} = \left\{
							\begin{aligned}
								&b^{(w, j)}, &&w \in \mathcal{R}, \\
								&\tilde{b}^{(w, j)}, &&w \in \mathcal{W} \setminus \mathcal{R}.
							\end{aligned}
							\right.
						\end{align}

						We first prove that the sample gradient $\nabla \hat{f}_{w, j}(\cdot)$ is bounded. 	
						For the $k$-th block of $\nabla \hat{f}_{w, j}(\cdot)$, we have
						\begin{align}
							\nabla_{x_{k}} \hat{f}_{w, j}(x)= -a^{(w, j)} \big(\bm{1}\{\hat{b}^{(w, j)}=k\}-\frac{\exp(x_k^Ta^{(w, j)})}{\sum_{l=1}^K \exp(x_l^T a^{(w, j)})}\big).\label{eq:block-sample-gradient}
						\end{align}
						Therefore, the entire sample gradient $\nabla \hat{f}_{w, j}(\cdot)$ satisfies
						\begin{align}\label{eq:lem4-by-product}
							\|\nabla \hat{f}_{w, j}(x)\|^2 &=\sum_{k=1}^{K} \| \nabla_{x_{k}} \hat{f}_{w, j}(x)\|^2\\
							&=\sum_{k=1}^{K} \|a^{(w, j)}\big(\bm{1}\{\hat{b}^{(w, j)}=k\}-\frac{\exp(x_k^Ta^{(w, j)})}{\sum_{l=1}^K \exp(x_l^T a^{(w, j)})}\big) \|^2\nonumber \\
							&=\sum_{k=1}^{K} \big(\bm{1}\{\hat{b}^{(w, j)}=k\}-\frac{\exp(x_k^Ta^{(w, j)})}{\sum_{l=1}^K \exp(x_l^T a^{(w, j)})}\big)^2 \|a^{(w, j)} \|^2. \nonumber
						\end{align}		

						Since
						\begin{align}
							&\sum_{k=1}^K \Big(1\{\hat{b}^{(w, j)}=k\}-\frac{\exp(x_k^Ta^{(w, j)})}{\sum_{l=1}^K \exp(x_l^T a^{(w, j)})}\Big)^2 \\ \leq  &\Big(\sum_{k=1}^K|1\{\hat{b}^{(w, j)}=k\}-\frac{\exp(x_k^Ta^{(w, j)})}{\sum_{l=1}^K \exp(x_l^T a^{(w, j)})}|\Big)^2 \leq 4 \nonumber,
						\end{align}
						we have
						\begin{align}
							\|\nabla \hat{f}_{w, j}(x) \|^2\le\sum_{k=1}^{K}\big(\bm{1}\{\hat{b}^{(w, j)}=k\}-\frac{\exp(x_k^Ta^{(w, j)})}{\sum_{l=1}^K \exp(x_l^T a^{(w, j)})}\big)^2 \|a^{(w, j)}\|^2 \leq \big(2 \|a^{(w, j)}\|\big)^2,
						\end{align}
						and
						\begin{align}
							\|\nabla \hat{f}_{w, j}(x) \|\le 2 \|a^{(w, j)}\|.
						\end{align}
						which shows the upper bound for the gradient of the sample cost function $\hat{f}_{w, j}(\cdot)$.
						
						Now we prove that the local gradient $\nabla \hat{f}_w(\cdot)$ is also bounded.
						By $\nabla \hat{f}_w (x) = \frac{1}{J} \sum_{j=1}^J \nabla \hat{f}_{w, j}(x)$, we have
						\begin{align}
							\|\nabla \hat{f}_w (x)\| = \|\frac{1}{J} \sum_{j=1}^J \nabla \hat{f}_{w, j}(x)\| \leq  \frac{1}{J} \sum_{j=1}^J\| \nabla \hat{f}_{w, j}(x)\| \leq 2 \max_{j \in [J]} \|a^{(w, j)}\|.
						\end{align}

						Finally, we refine the bound of the local gradient $\nabla \hat{f}_w(\cdot)$ under the non-negativity assumption. Starting from the second equality in \eqref{eq:lem4-by-product} and using $\nabla_{x_{k}} \hat{f}_w(x) = \frac{1}{J} \sum_{j=1}^J \nabla_{x_{k}} \hat{f}_{w, j}(x)$, we have
						\begin{align}
							\|\nabla \hat{f}_w(x)\|^2 &= \sum_{k=1}^{K} \|\frac1J \sum_{j=1}^J a^{(w, j)}\big(\bm{1}\{b^{(w, j)}=k\}-\frac{\exp(x_k^Ta^{(w, j)})}{\sum_{l=1}^K \exp(x_l^T a^{(w, j)})}\big) \|^2 \\
							&\le\sum_{k=1}^{K} \left\|\max_{j\in[J]} \big|\bm{1}\{b^{(w, j)}=k\}-\frac{\exp(x_k^Ta^{(w, j)})}{\sum_{l=1}^K \exp(x_l^T a^{(w, j)})}\big| \cdot \frac1J\sum_{j=1}^J  a^{(w, j)} \right\|^2, \nonumber
						\end{align}
						where the inequality is due to $\big|\sum_{j\in[J]} c_j x_j \big| \le \max_{j\in[J]} |c_j|\cdot \sum_{j\in[J]}x_j$ for any sequence $\{c_j\}_{j\in[J]}$ and any positive sequence $\{x_j\}_{j\in[J]}$. Since $\max_{j\in[J]} \big|\bm{1}\{b^{(w, j)}=k\}-\frac{\exp(x_k^Ta^{(w, j)})}{\sum_{l=1}^K \exp(x_l^T a^{(w, j)})}\big|\le1$,
						we reach our conclusion of
						\begin{align}
							\|\nabla \hat{f}_w(x)\|^2 &\le K\|\frac1J\sum_{j=1}^J  a^{(w, j)} \|^2,
						\end{align}
						which completes the proof.
					\end{proof}
					
					\subsection{Proof of Lemma \ref{lemma: bound of softmax regression}}
					\label{sec: Appendix A.2}

					\begin{proof}
						Note that
						\begin{align}\label{proof of lemma1: 2}
							\max_{w \in \mathcal{W}\setminus\mathcal{R}} \|\nabla \tilde{f}_w(x) - \nabla f(x)\|^2
							\leq 2\max_{w \in \mathcal{W}\setminus\mathcal{R}} \|\nabla \tilde{f}_w(x) \|^2 +  2\|\nabla f(x)\|^2.
						\end{align}
						From Lemma \ref{lemma:bounded gradient in softmax regression}, we know the first term at the right-hand side of \eqref{proof of lemma1: 2} can be upper-bounded as
						\begin{align}\label{proof of lemma1: 3}
							\max_{w \in \mathcal{W}\setminus\mathcal{R}} \|\nabla \tilde{f}_w(x) \|^2 \le K\max_{w \in \mathcal{W}\setminus\mathcal{R}} \|\frac{1}{J} \sum_{j=1}^{J} a^{(w, j)}\|^2.
						\end{align}
						
						For the second term at the right-hand side of \eqref{proof of lemma1: 2}, applying the inequality of $\|\frac1R \sum_{w\in\mathcal{R}} \nabla f_w\|^2\le\frac1R\sum_{w\in\mathcal{R}} \|\nabla f_w\|^2$ gives
						\begin{align}\label{proof of lemma1: 5}
							\|\nabla f(x)\|^2\le\frac1R\sum_{w\in\mathcal{R}} \|\nabla f_w(x)\|^2
							\le \max_{w \in \mathcal{R}} \|\nabla f_w(x)\|^2\le K\max_{w \in \mathcal{R}} \|\frac{1}{J} \sum_{j=1}^{J} a^{(w, j)}\|^2,
						\end{align}
						where the last inequality similarly comes from the assumption that $a^{(w,j)}$ is entry-wise non-negative for all $w \in \mathcal{W}$ and all $j \in \{1, \cdots, J\}$.
						
						Combining \eqref{proof of lemma1: 3} and \eqref{proof of lemma1: 5}, we have
						\begin{align}
							\max_{w \in \mathcal{W}\setminus\mathcal{R}} \|\nabla \tilde{f}_w(x) -\nabla f(x)\|^2 \leq 4K \max_{w \in \mathcal{W}} \|\frac{1}{J} \sum_{j=1}^J a^{(w, j)}\|^2,
						\end{align}
						or equivalently
						\begin{align}
							\max_{w \in \mathcal{W}\setminus\mathcal{R}} \|\nabla \tilde{f}_w(x) -\nabla f(x)\| \leq 2\sqrt{K} \max_{w \in \mathcal{W}} \|\frac{1}{J} \sum_{j=1}^J a^{(w, j)}\|,
						\end{align}
						which is exactly (8) with $A \leq  2\sqrt{K}\max_{w \in \mathcal{W}} \|\frac{1}{J} \sum_{j=1}^{J} a^{(w, j)}\|$.
					\end{proof}
					
					\subsection{Proofs of Lemma \ref{lemma: upper bound of heterogeneity of softmax regression} and Its Extension}
					\label{sec: Appendix A.3}
					
					The following lemma combines Lemma \ref{lemma: upper bound of heterogeneity of softmax regression} and its extension in the sufficiently heterogeneous case.
					\begin{lemma}\label{lemma: lower and upper bound of heterogeneity of softmax regression}
						Consider the distributed softmax regression problem where the local costs of the regular workers are in the form of \eqref{ex1: local cost function of softmax regression}. If $a^{(w,j)}$ is entry-wise non-negative for all $w \in \mathcal{R}$ and all $j \in \{1, \cdots, J\}$, then Assumption \ref{Assump: bounded heterogeneity} is satisfied with
						\begin{align}\label{eq: upper bound of xi}
							\xi \leq  2\sqrt{K} \max_{w \in \mathcal{R}}  \|\frac{1}{J} \sum_{j=1}^J a^{(w, j)}\|.
						\end{align}
						Further, if any regular worker $w\in \mathcal{R}$ only has the samples from one class and the samples from one class only belongs to one regular worker (i.e., $b^{(w,j)}=b^{(w^\prime,j^\prime)}$ if and only if $w=w^\prime$, for all $w,w^\prime\in\mathcal{R}$ and all $j,j^\prime\in \{1, \cdots, J\}$), we have
						\begin{align}\label{eq: exact order of xi}
							\xi = \Theta\big(\max_{w \in \mathcal{R}}  \|\frac{1}{J} \sum_{j=1}^J a^{(w, j)}\|\big).
						\end{align}
					\end{lemma}
					
					\begin{proof}
						Note that for any regular worker $w\in\mathcal{R}$, it holds
						\begin{align}\label{proof of lemma2: 1}
							\|\nabla f_w(x) - \nabla f(x)\| &= \|\big(1-\frac1R\big)\nabla f_w(x) - \sum_{w^\prime\in\mathcal{R},w^\prime \neq w}\frac1R\nabla f_{w^\prime}(x)\| \\
							&\leq \big(1-\frac1R\big)\|\nabla f_w(x)\| + \frac1R\sum_{w^\prime\in\mathcal{R},w^\prime \neq w}\|\nabla f_{w^\prime}(x)\|\nonumber\\
							&\leq 2 \max_{w^\prime \in \mathcal{R}} \|\nabla f_{w^\prime}(x) \|.\nonumber
						\end{align}
						
						\begin{align}\label{proof of lemma2: 4}
							\max_{w \in \mathcal{R}} \|\nabla f_w(x) -\nabla f(x)\| \leq 2\sqrt{K} \max_{w \in \mathcal{R}} \|\frac{1}{J} \sum_{j=1}^J a^{(w, j)}\|,
						\end{align}
						and thus Assumption \ref{Assump: bounded heterogeneity} is satisfied with
						\begin{align}\label{proof of lemma2: upper bound}
							\xi \leq 2\sqrt{K} \max_{w \in \mathcal{R}} \|\frac{1}{J} \sum_{j=1}^J a^{(w, j)}\|.
						\end{align}

						Next, we prove the lower bound of $\xi$ when any regular worker $w\in \mathcal{R}$ only has the samples from one class and the samples from one class only belongs to one regular worker. For any regular worker $w^\prime\in\mathcal{R}$ and any $x\in\R^D$, Assumption \ref{Assump: bounded heterogeneity} gives that
						\begin{align}\label{proof of lemma2: 6}
							\xi^2 \ge \|\nabla f_{w^\prime}(x) - \nabla f(x)\|^2
							=&\sum_{k=1}^{K} \|\frac{1}{J} \sum_{j=1}^J a^{(w', j)} (\bm{1}\{b^{(w', j)}=k\}-\frac{\exp(x_k^Ta^{(w', j)})}{\sum_{l=1}^K \exp(x_l^T a^{(w', j)})}) \\&- \frac{1}{R} \sum_{w \in \mathcal{R}}\frac{1}{J} \sum_{j=1}^J a^{(w, j)} (\bm{1}\{b^{(w, j)}=k\}-\frac{\exp(x_k^Ta^{(w, j)})}{\sum_{l=1}^K \exp(x_l^T a^{(w, j)})})\|^2 \nonumber.
						\end{align}
						Letting  $[x_k]_i = 0$ for any $i \in \{1, \cdots, D\}$ and $k \in \{1, \cdots, K\}$, it holds that
						\begin{align*}
							\frac{\exp(x_k^Ta^{(w, j)})}{\sum_{l=1}^K \exp(x_l^T a^{(w^, j)})} = \frac{1}{K}, \quad \forall w \in \mathcal{R}, ~ j \in \{1, \cdots, J\}, k \in \{1, \cdots, K\}.
						\end{align*}
						Given the heterogeneous label distribution, there exists $k^\prime\in \{1, \cdots, K\}$, such that $b^{(w^\prime, j)}=k^\prime\neq b^{(w, j)}$ for all $w\neq w^\prime$ and $j\in \{1, \cdots, J\}$.  Specifically, taking one of the summands in \eqref{proof of lemma2: 6} with $k=k^\prime$, we obtain
						\begin{align}
							\xi^2 \ge& \|\big(1-\frac1R\big)\frac{1}{J} \sum_{j=1}^J a^{(w^\prime, j)} (1 - \frac{1}{K}) \\ & \quad \quad - \frac{1}{R} \sum_{w\in \mathcal{R}, w\ne w^\prime}\frac{1}{J} \sum_{j=1}^J a^{(w, j)} (\bm{1}\{b^{(w, j)}=k^\prime\}-\frac{\exp(x_{k^\prime}^Ta^{(w, j)})}{\sum_{l=1}^K \exp(x_l^T a^{(w, j)})})\|^2 \nonumber\\
							=& \|\big(1-\frac1R\big)\big(1-\frac1K\big)\frac{1}{J} \sum_{j=1}^J a^{(w', j)} + \frac{1}{RK}\sum_{w\in \mathcal{R}, w\ne w^\prime}\frac{1}{J} \sum_{j=1}^J a^{(w, j)} \|^2 \nonumber\\
							\ge& \|\big(1-\frac1R\big)\big(1-\frac1K\big)\frac{1}{J} \sum_{j=1}^J a^{(w', j)}\|^2 \nonumber,
						\end{align}
						where the last inequality is due to the fact that each term in the summation is non-negative.
						Note that $w^\prime\in\mathcal{R}$ is arbitrary, which results in
						\begin{align}\label{proof of lemma2: lower bound}
							\xi\ge \big(1-\frac1R\big)\big(1-\frac1K\big)\max_{w \in \mathcal{R}} \|\frac{1}{J} \sum_{j=1}^{J} a^{(w,j)}\|.
						\end{align}
						
						Combining \eqref{proof of lemma2: upper bound} and \eqref{proof of lemma2: lower bound}, we have
						\begin{align}
							\xi = \Theta\big(\max_{w \in \mathcal{R}} \|\frac{1}{J} \sum_{j=1}^{J} a^{(w,j)}\|\big),
						\end{align}
						which completes the proof.
					\end{proof}

					\section{Analysis of $\rho$-Robust Aggregators}
					\label{sec: Appendix B}
					In this section, we prove Lemma \ref{lemma: lower bound of rho} that explores the approximation abilities of $\rho$-robust aggregators, and show that the state-of-the-art robust aggregators, including TriMean, CC, and FABA, are all $\rho$-robust aggregators when the fraction of poisoned workers is below their respective thresholds.

					\subsection{Proof of Lemma \ref{lemma: lower bound of rho}}
					\label{sec: Appendix B.1}
					An equivalent statement of Lemma \ref{lemma: lower bound of rho} is shown below.
					\begin{lemma}
						Denote $\delta \triangleq 1-\frac{R}{W}$ as the fraction of the poisoned workers. For any aggregator RAgg, if $\delta\ge\frac12$ or $\rho<\min\{\frac{\delta}{1 - 2\delta},1\}$, then there exist $W$ messages $y_1, y_2, \ldots, y_W \in \R^D$ such that 
						\begin{align}\label{eq: rho-rar-reverse}
							\| \text{RAgg}(\{y_1, \ldots, y_W\}) - \bar{y}\| > \rho \cdot \max_{w \in \mathcal{R}} \|y_w - \bar{y} \|.
						\end{align}
						where $\bar{y} = \frac{1}{R} \sum_{w \in \mathcal{R}} y_w$ is the average message of the regular workers.
					\end{lemma}
					
					\begin{proof}
						Without loss of generality, consider $D=1$ and $\mathcal{R}=\{1,2,\ldots,R\}$. For the high-dimensional cases, setting all entries but one as zero degenerates to the scalar case. The key idea of our proof is to find two sets of $W$ messages that are the same but the messages from regular workers are different. Therefore, any aggregator cannot distinguish between them and will yield the same output. Elaborately designing these two sets, we shall guarantee that at least one of them satisfies \eqref{eq: rho-rar-reverse}.

						We first consider the case that $\delta \geq \frac{1}{2}$, or equivalently, $2R \le W$. In this case, if we can find  $W$ messages $y_1, \cdots, y_W \in \R^D$ such that $\| \text{RAgg}(\{y_1, \ldots, y_W\}) - \bar{y}\| \neq 0$ and $\max_{w \in \mathcal{R}} \|y_w - \bar{y} \| = 0$, we have found the set of messages satisfying \eqref{eq: rho-rar-reverse}. The construction is as follows. Let $y_1 = \cdots = y_R $ $= 0$, $y_{R+1} = \cdots y_{2R} = \rho+1$ and $y_{2R + 1} = \cdots = y_{W} = 0$.
						Then $\bar{y} = 0$ and $\max_{w \in \mathcal{R}} \|y_w - \bar{y}\| = 0$. If $\text{RAgg}(\{y_1, \ldots, y_W\})\ne0$, these $W$ messages satisfy \eqref{eq: rho-rar-reverse} and our task is fulfilled. Otherwise, we know that $\text{RAgg}(\{y_1, \ldots, y_W\}) = 0$. Rearranging these messages as $z_{1} = \ldots z_{R} = \rho+1$ and $z_{R + 1} = \ldots = z_{W} = 0$, the aggregator outputs the same value $\text{RAgg}(\{z_1, \ldots, z_W\}) = 0$, while $\bar{z} = \rho+1$ and $\max_{w \in \mathcal{R}} \|z_w - \bar{z}\| = 0$. Thus, $\{z_w,w\le W\}$ is the set of messages satisfying \eqref{eq: rho-rar-reverse}.
						
						Second, we consider the case that $\delta < \frac{1}{2}$ and $\rho<\min\{\frac{\delta}{1 - 2\delta},1\}$. In this case, if we can find $W$ messages $y_1, \cdots, y_W$ such that $\| \text{RAgg}(\{y_1, \ldots, y_W\}) - \bar{y}\| = \frac{\delta}{1 - \delta}$ and $\max_{w \in \mathcal{R}} \|y_w - \bar{y} \| = \frac{\max\{1 - 2\delta, \delta\}}{1 - \delta}$, we have found the set of messages satisfying \eqref{eq: rho-rar-reverse}. Similar to the above construction yet with $2R>W$, we consider $y_1 = \ldots = y_R = 0$ and $y_{R+1} = \ldots = y_{W} = 1$. Accordingly, we have $\bar{y} = 0$ and $\max_{w \in \mathcal{R}} \|y_w - \bar{y}\| = 0$. If $\text{RAgg}(\{y_1, \ldots, y_W\})\ne0$, we have found $W$ messages to satisfy \eqref{eq: rho-rar-reverse}. Otherwise, we know that $\text{RAgg}(\{y_1, \ldots, y_W\}) = 0$. Rearranging those messages as $z_{1} = \ldots z_{W-R} = 1$ and $z_{W-R + 1} = \ldots = z_{W} = 0$, the aggregator outputs the same value $\text{RAgg}(\{z_1, \ldots, z_W\}) = 0$, while $\bar{z} = \frac\delta{1-\delta}$ and $\max_{w \in \mathcal{R}} \|z_w - \bar{z}\| = \frac{\max\{1-2\delta,\delta\}}{1-\delta}$. In consequence, $\{z_w,w\le W\}$ is the set of messages satisfying \eqref{eq: rho-rar-reverse}.
					\end{proof}
					
					Our proof is motivated by those of \citet[Proposition 2]{farhadkhani2022byzantine} and \citet[Proposition 6]{allouah2023fixing}. However, their definitions of the robust aggregator are different to ours such that the proofs are different too.

					\subsection{TriMean}
					TriMean is an aggregator that discards the smallest $W - R$ elements and largest $W - R$ elements in each dimension. The aggregated output of TriMean in dimension $d$ is given by
					\begin{align}\label{eq: TriMean}
						[\text{TriMean}(\{y_1, \ldots, y_W\})]_d = \frac{1}{2R - W} \sum_{w \in [\mathcal{U}]_d} [y_w]_d,
					\end{align}
					where $[\cdot]_d$ denotes the $d$-th coordinate of a vector, and $[\mathcal{U}]_d$ is the set of workers whose $d$-th elements are not filtered after removal. Below we show that TriMean is a $\rho$-robust aggregator if $\delta < \frac{1}{2}$.
					
					\begin{lemma}\label{lemma: TriMean}
						Denote $\delta \triangleq 1 - \frac{R}{W}$ as the fraction of the poisoned workers. If $\delta < \frac{1}{2}$, TriMean is a $\rho$-robust aggregator with $\rho = \frac{3\delta}{1 - 2\delta} \min\{\sqrt{D}, \sqrt{R}\}$.
					\end{lemma}
					
					\begin{proof}
						We first analyze the aggregated result in one dimension $d$ and then extend it to all dimensions. Denote $[\mathcal{U}_R]_d \triangleq [\mathcal{U}]_d \cap \mathcal{R}$ and $[\mathcal{U}_P]_d \triangleq [\mathcal{U}]_d \cap (\mathcal{W} \setminus \mathcal{R})$ as the set of remaining regular workers and poisoned workers after removal, respectively.
						
						If TriMean successfully removes all the poisoned workers in dimension $d$, such that $[\mathcal{U}_P]_d = \emptyset$ and $[\mathcal{U}]_d \subseteq \mathcal{R}$, it holds
						\begin{align}\label{proof:trimean-case1}
							\|[\text{TriMean}(\{y_1, \ldots, y_W\})]_d  - [\bar{y}]_d\|
							=& \|\frac{1}{2R - W} \sum_{w \in [\mathcal{U}]_d} [y_w]_d - \frac{1}{R} \sum_{w \in \mathcal{R}} [y_w]_d\| \\
							=& \|(\frac{1}{2R - W} - \frac{1}{R}) \sum_{w \in \mathcal{R}} [y_w]_d - \frac{1}{2R - W} \sum_{w \in \mathcal{R} \setminus [\mathcal{U}]_d}  [y_w]_d\| \nonumber\\
							=& \|\frac{W - R}{2R - W}[\bar{y}]_d - \frac1{2R- W}\sum_{w \in \mathcal{R} \setminus [\mathcal{U}]_d} [y_w]_d \| \nonumber \\
							\leq&\frac{1}{2R - W}\sum_{w \in \mathcal{R} \setminus [\mathcal{U}]_d} \|[y_w]_d - [\bar{y}]_d\| \nonumber \\
							\leq &\frac{\delta}{1 - 2\delta} \cdot \max_{w \in \mathcal{R}} \|[y_w]_d - [\bar{y}]_d\|. \nonumber
						\end{align}

						Otherwise, TriMean cannot remove all poisoned workers in dimension $d$, which means $[\mathcal{U}_P]_d \neq \emptyset$. Define
						\begin{align}
							\bar{y}_{R_d} \triangleq \frac{1}{|[\mathcal{U}_R]_d|} \sum_{w \in [\mathcal{U}_R]_d} [y_w]_d, \quad \bar{y}_{P_d} \triangleq \frac{1}{|[\mathcal{U}_P]_d|} \sum_{w \in [\mathcal{U}_P]_d} [y_w]_d
						\end{align}
						as the average of elements in $[\mathcal{U}_R]_d$ and $[\mathcal{U}_P]_d$, respectively. Also denote $u_i \triangleq \frac{|[\mathcal{U}_P]_d|}{|[\mathcal{U}]_d|}$ as the fraction of poisoned workers that remains.
						With the above definitions, we have $u_i\le\frac{W - R}{2R-W} = \frac{\delta}{1-2\delta}$ and
						\begin{align} \label{proof: trimean}
							&\|[\text{TriMean}(\{y_1, \ldots, y_W\})]_d  - [\bar{y}]_d\| \\
							= &\|u_i \cdot \bar{y}_{P_d} + (1 - u_i) \cdot \bar{y}_{R_d} -  [\bar{y}]_d \| \nonumber\\
							\leq& u_i  \|\bar{y}_{P_d} - [\bar{y}]_d \| + (1 - u_i) \|\bar{y}_{R_d} - [\bar{y}]_d \|. \nonumber
						\end{align}
						For the first term at the right-hand side of \eqref{proof: trimean}, we have
						\begin{align}\label{proof: trimean-term1}
							u_i  \|\bar{y}_{P_d} - [\bar{y}]_d \|
							=& u_i\|\frac{1}{|[\mathcal{U}_P]_d|} \sum_{w \in [\mathcal{U}_P]_d} [y_w]_d - [\bar{y}]_d \| \\
							\leq& \frac{\delta}{1-2\delta}\max_{w \in [\mathcal{U}_P]_d} \|[y_w]_d - [\bar{y}]_d \| \nonumber\\
							\leq& \frac{\delta}{1-2\delta}\max_{w \in \mathcal{R}}\| [y_w]_d - [\bar{y}]_d \|, \nonumber
						\end{align}
						where the last inequality is due to the principle of filtering. Specifically, as the poisoned workers cannot own the $W - R$ largest values in dimension $d$, there exists a regular worker such that its $d$-th element is larger than $[y_w]_d$ for all $w \in [\mathcal{U}_P]_d$. Similarly, there exists a regular worker with the $d$-th element smaller than all $[y_w]_d$'s. This observation guarantees the last inequality in \eqref{proof: trimean-term1}.
						For the second term at the right-hand side of \eqref{proof: trimean}, we have
						\begin{align}\label{proof: trimean-term2}
							(1 - u_i)  \|\bar{y}_{R_d} - [\bar{y}]_d \|
							=& (1 - u_i)\|(\frac{1}{|[\mathcal{U}_R]_d|} - \frac{1}{R}) \sum_{w \in \mathcal{R}} [y_w]_d - \frac{1}{|[\mathcal{U}_R]_d|} \sum_{w \in \mathcal{R} \setminus [\mathcal{U}_R]_d } [y_w]_d\| \\
							=& (1 - u_i) \cdot \frac{1}{|[\mathcal{U}_R]_d|} \|\sum_{w \in \mathcal{R} \setminus [\mathcal{U}_R]_d  }([y_w]_d - [\bar{y}]_d)\| \nonumber\\
							\leq& \frac{1}{2R - W} \sum_{w \in \mathcal{R} \setminus [\mathcal{U}_R]_d  }\|[y_w]_d - [\bar{y}]_d\| \nonumber\\
							\leq& \frac{R - |[\mathcal{U}_R]_d|}{2R - W} \max_{w \in \mathcal{R}}\|[y_w]_d - [\bar{y}]_d\| \nonumber\\
							\leq& \frac{2\delta}{1 - 2\delta} \max_{w \in \mathcal{R}}\|[y_w]_d - [\bar{y}]_d\|,\nonumber
						\end{align}
						where the last inequality is due to $|[\mathcal{U}_R]_d|=(2R - W) - |[\mathcal{U}_P]_d|\geq (2R - W) - (W - R)=3R-2W$.
						
						Substituting \eqref{proof: trimean-term1} and \eqref{proof: trimean-term2} into \eqref{proof: trimean}, we have
						\begin{align}\label{proof:trimean-case2}
							\|[\text{TriMean}(\{y_1, \ldots, y_W\})]_d  - [\bar{y}]_d\| \leq \frac{3\delta}{1 - 2\delta} \max_{w \in \mathcal{R}}\|[y_w]_d - [\bar{y}]_d\|.
						\end{align}
						Combining the first case \eqref{proof:trimean-case1} and the second case \eqref{proof:trimean-case2}, we have
						\begin{align}\label{proof:trimean-scalar}
							\|[\text{TriMean}(\{y_1, \ldots, y_W\})]_d  - [\bar{y}]_d\| \leq \frac{3\delta}{1 - 2\delta} \max_{w \in \mathcal{R}}\|[y_w]_d - [\bar{y}]_d\|.
						\end{align}

						Next, we extend the scalar scenario to the vector scenario. Notice that
						\begin{align}\label{proof:trimean-vector-1}
							\sum_{d=1}^{D} \max_{w \in \mathcal{R}} \|[y_w]_d - [\bar{y}]_d\|^2 \leq \sum_{d=1}^{D} \max_{w \in \mathcal{R}} \|y_w - \bar{y}\|^2 \leq D \max_{w \in \mathcal{R}} \|y_w - \bar{y}\|^2.
						\end{align}
						On the other hand, we also have
						\begin{align}\label{proof:trimean-vector-2}
							&\sum_{d=1}^{D} \max_{w \in \mathcal{R}} \|[y_w]_d - [\bar{y}]_d\|^2 \\
							\leq& \sum_{d=1}^{D} \sum_{w \in \mathcal{R}} \|[y_w]_d - [\bar{y}]_d\|^2 \nonumber\\
							\leq & \sum_{w \in \mathcal{R}} \sum_{d=1}^{D} \|[y_w]_d - [\bar{y}]_d\|^2 \nonumber\\ \leq &R \max_{w \in \mathcal{R}} \|y_w - \bar{y}\|^2. \nonumber
						\end{align}
						Combining \eqref{proof:trimean-vector-1} and \eqref{proof:trimean-vector-2} gives
						\begin{align}\label{proof:trimean-vector-3}
							\sum_{d=1}^{D} \max_{w \in \mathcal{R}} \|[y_w]_d - [\bar{y}]_d\|^2 \leq \min\{D, R\} \max_{w \in \mathcal{R}} \|y_w - \bar{y}\|^2.
						\end{align}
						Substituting \eqref{proof:trimean-vector-3} into \eqref{proof:trimean-scalar}, we have
						\begin{align}\label{proof:trimean-square}
							\|\text{TriMean}(\{y_1, \ldots, y_W\})  - \bar{y}\|^2
							=&\sum_{d=1}^{D}\|[\text{TriMean}(\{y_1, \ldots, y_W\})]_d  - [\bar{y}]_d\|^2 \\
							\leq& (\frac{3\delta}{1 - 2\delta})^2 \sum_{d=1}^{D}\max_{w \in \mathcal{R}}\|[y_w]_d - [\bar{y}]_d\|^2 \nonumber\\
							\leq& (\frac{3\delta}{1 - 2\delta})^2 \min\{D, R\} \max_{w \in \mathcal{R}} \|y_w - \bar{y}\|^2.\nonumber
						\end{align}
						Taking the square roots on both sides of \eqref{proof:trimean-square}, we have
						\begin{align}
							\|\text{TriMean}(\{y_1, \ldots, y_W\})  - \bar{y}\| \leq \frac{3\delta}{1 - 2\delta} \cdot \min\{\sqrt{D}, \sqrt{R}\} \cdot \max_{w \in \mathcal{R}} \|y_w - \bar{y}\|,
						\end{align}
						which completes the proof.	
					\end{proof}
					
					\subsection{CC}
					CC is an aggregator that iteratively clips the messages from workers. CC starts from some point $v^0$. At iteration $i$, the update rule of CC can be formulated as
					\begin{align}\label{eq: cc update}
						v^{i+1} = v^i + \frac{1}{W} \sum_{w=1}^W \text{CLIP}(y_w - v^i, \tau ),
					\end{align}
					where
					\begin{align}
						\text{CLIP}(y_w - v^i, \tau) =
						\left\{
						\begin{aligned}
							&y_w - v^i, &&  \|y_w - v^i\| \leq \tau, \\
							&\frac{\tau}{\|y_w - v^i\|}(y_w - v^i), &&  \|y_w - v^i\| > \tau,
						\end{aligned}
						\right.
					\end{align}
					and $\tau \geq 0$ is the clipping threshold. After $L$ iterations, CC outputs the last vector as
					\begin{align}
						\text{CC}(\{y_1, \ldots, y_W\}) = v^L .
					\end{align}

					Below we prove that with proper initialization and clipping threshold, one-step CC ($L=1$) is a $\rho$-robust aggregator if $\delta < \frac{1}{2}$.
					
					\begin{lemma}\label{lemma: CC}
						Denote $\delta \triangleq 1 - \frac{R}{W}$ as the fraction of the poisoned workers. If $\delta < \frac{1}{2}$, choosing the starting point $v_0$ satisfying $\|v_0 - \bar{y}\|^2 \leq \max_{w \in \mathcal{R}} \|y_w - \bar{y}\|^2$ and the clipping threshold $\tau = \sqrt{\frac{4(1-\delta)\max_{w \in \mathcal{R}} \|y_w - \bar{y}\|^2}{\delta}}$, one-step CC is a $\rho$-robust aggregator with $\rho = \sqrt{24\delta}$.
					\end{lemma}
					
					\begin{proof}
						The output of one-step CC is
						\begin{align}\label{eq: cc}
							\text{CC}(\{y_1, \ldots, y_W\}) = v^0 + \frac{1}{W} \sum_{w=1}^W \text{CLIP}(y_w - v^0, \tau ).
						\end{align}
						Note that if $\max_{w \in \mathcal{R} } \|y_w - \bar{y}\| = 0$, we have $\tau = 0$ and $v^0 = \bar{y}$, which leads to $\text{CC}(\{y_1, \ldots, y_W\}) = \bar{y}$. Therefore, we have
						\begin{align}\label{eq: cc-tau=0}
							\|\text{CC}(\{y_1, \ldots, y_W\}) - \bar{y}\|  = \max_{w \in \mathcal{R} } \|y_w - \bar{y}\| \leq \sqrt{24\delta}\max_{w \in \mathcal{R} } \|y_w - \bar{y}\|.
						\end{align}
						Below we consider the case that $\max_{w \in \mathcal{R} } \|y_w - \bar{y}\| > 0$, then by definition $\tau > 0$.
						
						Denoting $\hat{y}_w = v^0 + \text{CLIP}(y_w - v^0, \tau)$ for any $w \in \{1, \ldots, W\} $, we have
						\begin{align}\label{eq: cc-hat-y-w}
							\text{CC}(\{y_1, \ldots, y_W\}) = \frac{1}{W} \sum_{w=1}^{W} \hat{y}_w.
						\end{align}
						According to \eqref{eq: cc-hat-y-w}, we have
						\begin{align}\label{proof:cc-1}
							\|\text{CC}(\{y_1, \ldots, y_W\}) - \bar{y}\|^2
							= & \| \frac{1}{W} \sum_{w=1}^{W} \hat{y}_w - \bar{y}\|^2 \\
							=& \|(1-\delta) \cdot (\frac{1}{R} \sum_{w \in \mathcal{R}} \hat{y}_w - \bar{y}) + \delta \cdot \frac{1}{W - R} \sum_{w \in \mathcal{W} \setminus \mathcal{R}} (\hat{y}_w - \bar{y})\|^2 \nonumber \\
							\leq& 2(1-\delta)^2 \|\frac{1}{R} \sum_{w \in \mathcal{R}} (\hat{y}_w - y_w) \|^2 + 2\delta^2 \| \frac{1}{W - R} \sum_{w \in \mathcal{W} \setminus \mathcal{R}} (\hat{y}_w - \bar{y})\|^2 \nonumber \\
							\leq& 2(1-\delta)^2 \frac{1}{R} \sum_{w \in \mathcal{R}} \|\hat{y}_w - y_w\|^2 + 2\delta^2  \frac{1}{W - R} \sum_{w \in \mathcal{W} \setminus \mathcal{R}} \|\hat{y}_w - \bar{y}\|^2. \nonumber
						\end{align}
						where the last two inequalities are due to the Cauchy-Schwarz inequality.
						
						For any $w \in \mathcal{R}$, the term $\|\hat{y}_w - y_w\|$ holds
						\begin{align}
							\|\hat{y}_w - y_w\| = v^0 - y_w + \text{CLIP}(y_w - v^0, \tau).
						\end{align}
						If the regular message $y_w$ is not clipped, meaning that $\text{CLIP}(y_w - v^0, \tau) = y_w - v^0$, we have
						\begin{align}\label{proof:cc-not-clipped}
							\|\hat{y}_w - y_w\| = 0.
						\end{align}
						Otherwise, we have
						\begin{align}
							\|\hat{y}_w - y_w\| = \|v^0 - y_w - \frac{\tau}{\|y_w - v^l\|}(v^0 - y_w)\| = \|v^0 - y_w\| - \tau.
						\end{align}
						Since
						\begin{align}
							\|v^0 - y_w\| - \tau \leq \frac{\|v^0 - y_w\|^2}{\tau} \leq  \frac{2\|v^0 - \bar{y}\|^2 + 2\|y_w - \bar{y}\|^2}{\tau} \leq \frac{4\max_{w\in\mathcal{R}}\|y_w - \bar{y}\|^2}{\tau},
						\end{align}
						where the first inequality is due to $a - b \leq \frac{a^2}{b}$ that holds for any $a \geq 0, b > 0$, and the last inequality is due to $\|v_0 - \bar{y}\|^2 \leq \max_{w \in \mathcal{R}} \|y_w - \bar{y}\|^2$, we have
						\begin{align}\label{proof:cc-clipped}
							\|\hat{y}_w - y_w\| \leq \frac{4\max_{w\in\mathcal{R}}\|y_w - \bar{y}\|^2}{\tau}.
						\end{align}
						Combining \eqref{proof:cc-not-clipped} and \eqref{proof:cc-clipped}, we have
						\begin{align}\label{proof:cc-regular}
							\|\hat{y}_w - y_w\| \leq \frac{4\max_{w\in\mathcal{R}}\|y_w - \bar{y}\|^2}{\tau}.
						\end{align}

						For any $w \in \mathcal{W} \setminus \mathcal{R}$, the term $\|\hat{y}_w -\bar{y}\|^2$ holds
						\begin{align}\label{proof:cc-poisoned}
							\|\hat{y}_w - \bar{y}\|^2 \leq 2\|\hat{y}_w - v^0\|^2 + 2\|v^0 - \bar{y}\|^2 \leq 2\tau^2 + 2\max_{w \in \mathcal{R}} \|y_w -\bar{y}\|^2,
						\end{align}
						where the last inequality is due to $\|\hat{y}_w - v^0\|^2 = \|\text{CLIP}(y_w - v^0, \tau)\|^2 \leq \tau^2$ and $\|v_0 - \bar{y}\|^2 \leq \max_{w \in \mathcal{R}} \|y_w - \bar{y}\|^2$.
						
						Substituting \eqref{proof:cc-regular} and \eqref{proof:cc-poisoned} into \eqref{proof:cc-1}, we have
						\begin{align}
							&\|\text{CC}(\{y_1, \ldots, y_W\}) - \bar{y}\|^2 \\
							\leq &2(1-\delta)^2 (\frac{4\max_{w \in \mathcal{R}} \|y_w - \bar{y}\|^2}{\tau})^2 + 4\delta^2\tau^2 + 4\delta^2 \max_{w \in \mathcal{R}} \|y_w - \bar{y}\|^2 \nonumber\\
							\leq &24\delta(1 - \delta) \max_{w \in \mathcal{R}} \|y_w - \bar{y}\|^2 + 4\delta^2 \max_{w \in \mathcal{R}} \|y_w - \bar{y}\|^2 \nonumber\\
							\leq &24\delta\max_{w \in \mathcal{R}} \|y_w - \bar{y}\|^2 \nonumber,
						\end{align}
						where the second inequality is due to $\tau = \sqrt{\frac{4(1-\delta)\max_{w \in \mathcal{R}} \|y_w - \bar{y}\|^2}{\delta}}$.
						
						Therefore, we have
						\begin{align}\label{eq: cc-tau>0}
							\|\text{CC}(\{y_1, \ldots, y_W\}) - \bar{y}\| \leq \sqrt{24\delta} \max_{w \in \mathcal{R}} \|y_w - \bar{y}\|.
						\end{align}
						Combining \eqref{eq: cc-tau=0} and \eqref{eq: cc-tau>0}, we have that one-step CC is a $\rho$-robust aggregator with $\rho = \sqrt{24\delta}$. This completes the proof.
					\end{proof}
					
					\subsection{FABA}
					
					FABA is an aggregator that iteratively discards a possible outlier and averages the messages that remain after $W - R$ iterations. To be more concrete,
					denote $\mathcal{U}^{(i)}$ as the set of workers that are not discarded at the $i$-th iteration. Initialized with $\mathcal{U}^{(0)} = \{1, \ldots, W\}$, at iteration $i$, FABA computes the average of the messages from $\mathcal{U}^{(i)}$ and discards the worker whose message is farthest from that average to form $\mathcal{U}^{(i+1)}$. After $W - R$ iterations, FABA obtains $\mathcal{U}^{(W - R)}$ with $R$ workers, and then outputs
					\begin{align}\label{eq: FABA}
						\text{FABA}(\{y_1, \ldots, y_W\}) = \frac{1}{R} \sum_{w \in \mathcal{U}^{(W - R)}} y_w.
					\end{align}
					Below we prove that FABA is a $\rho$-robust aggregator if $\delta < \frac{1}{3}$.
					\begin{lemma}\label{lemma: FABA}
						Denote $\delta \triangleq 1 - \frac{R}{W}$ as the fraction of poisoned workers. If $\delta < \frac{1}{3}$, FABA is a $\rho$-robust aggregator with $\rho = \frac{2\delta}{1 - 3\delta}$.
					\end{lemma}
					
					\begin{proof}
						For notational convenience, denote $\mathcal{R}^{(i)} \triangleq \mathcal{U}^{(i)} \cap \mathcal{R}$ and $\mathcal{P}^{(i)} \triangleq \mathcal{U}^{(i)} \cap (\mathcal{W} \setminus \mathcal{R})$ as the sets of the regular workers and the poisoned workers in $\mathcal{U}^{(i)}$, respectively. Further denote three different averages
						\begin{align}
							\bar{y}_{\mathcal{U}^{(i)}} \triangleq \frac{1}{|\mathcal{U}^{(i)}|} \sum_{w \in \mathcal{U}^{(i)}} y_w, \quad \bar{y}_{\mathcal{R}^{(i)}} \triangleq \frac{1}{|\mathcal{R}^{(i)}|} \sum_{w \in \mathcal{R}^{(i)}} y_w, \quad \bar{y}_{\mathcal{P}^{(i)}} \triangleq \frac{1}{|\mathcal{P}^{(i)}|} \sum_{w \in \mathcal{P}^{(i)}} y_w
						\end{align}
						over $\mathcal{U}^{(i)}$, $\mathcal{R}^{(i)}$ and $\mathcal{P}^{(i)}$, respectively. Then
						\begin{align}\label{eq: u-r-p-relation}
							\bar{y}_{\mathcal{U}^{(i)}} = (1 - u_i) \cdot \bar{y}_{\mathcal{R}^{(i)}} + u_i \cdot \bar{y}_{\mathcal{P}^{(i)}},
						\end{align}
						and our goal is to bound $\|\bar{y}_{\mathcal{U}^{(P)}} - \bar{y}\|$ by $\max_{w \in \mathcal{R}} \|y_w - \bar{y}\|$.
						
						Denote $u_i \triangleq \frac{|\mathcal{P}^{(i)}|}{|\mathcal{U}^{(i)}|}$ as the fraction of the poisoned workers in $\mathcal{U}^{(i)}$. From $\delta < \frac{1}{3}$, $u_i\le\frac{W - R}{R} < \frac{1}{2}$ for any $i \in \{0, \ldots, W - R\}$.
						We claim that a regular worker is filtered out at iteration $i$ only if
						\begin{align}\label{eq:claim-for-faba}
							\|\bar{y}_{\mathcal{R}^{(i)}} - \bar{y}_{\mathcal{P}^{(i)}}\| \leq \frac{1}{1 - 2u_i} \max_{w \in \mathcal{R}} \|y_w - \bar{y}_{\mathcal{R}^{(i)}}\|.
						\end{align}
						This is because if $\|\bar{y}_{\mathcal{R}^{(i)}} - \bar{y}_{\mathcal{P}^{(i)}}\| > \frac{1}{1 - 2u_i} \max_{w \in \mathcal{R}} \|y_w - \bar{y}_{\mathcal{R}^{(i)}}\|$, then for any $w\in\mathcal{R}$, we have
						\begin{align} \label{proof: faba-case1-3}
							\|y_w - \bar{y}_{\mathcal{U}^{(i)}}\|
							\leq& \|y_w - \bar{y}_{\mathcal{R}^{(i)}}\| + \|\bar{y}_{\mathcal{R}^{(i)}}  -  \bar{y}_{\mathcal{U}^{(i)}} \| \\
							=&  \|y_w - \bar{y}_{\mathcal{R}^{(i)}}\|  + \frac{u_i}{1-u_i} \|\bar{y}_{\mathcal{P}^{(i)}} - \bar{y}_{\mathcal{U}^{(i)}} \| \nonumber\\
							\leq&  \max_{w \in \mathcal{R}} \|y_w - \bar{y}_{\mathcal{R}^{(i)}}\|  + \frac{u_i}{1-u_i} \|\bar{y}_{\mathcal{P}^{(i)}} - \bar{y}_{\mathcal{U}^{(i)}} \| \nonumber\\
							<& \frac{1-2u_i}{1-u_i}\|\bar{y}_{\mathcal{P}^{(i)}} - \bar{y}_{\mathcal{U}^{(i)}} \| + \frac{u_i}{1-u_i}\|\bar{y}_{\mathcal{P}^{(i)}} - \bar{y}_{\mathcal{U}^{(i)}} \|\nonumber\\
							\le& \max_{w\in \mathcal{P}^{(i)}} \|y_{w} - \bar{y}_{\mathcal{U}^{(i)}} \|,\nonumber
						\end{align}
						where \eqref{eq: u-r-p-relation} is applied to both the second and fourth lines. Therefore, there exists $w^\prime \in \mathcal{P}^{(i)}$ with farther distance to $\bar{y}_{\mathcal{U}^{(i)}}$ than all the regular workers, which guarantees that all the remaining regular workers will not be removed in this iteration.

						If at every iteration FABA discards a poisoned worker, then $\mathcal{U}^{(W - R)}=\mathcal{R}$ and
						\begin{align}\label{proof:faba-first-case}
							\|\bar{y}_{\mathcal{U}^{(W - R)}} - \bar{y}\| = 0 \leq \frac{2\delta}{1 - 3\delta}  \cdot \max_{w \in \mathcal{R}} \|y_w - \bar{y}\|.
						\end{align}
						Otherwise, there are iterations with regular workers removed. Denote $i^*$ as the last one among the $W - R$ iterations that removes the regular worker. Denote $w^{(i^*)}$ as the discarded worker at iteration $i^*$, we have $w^{(i^*)} \in \mathcal{R}$ and from the algorithmic principle of removal
						\begin{align}\label{eq:faba-principle}
							\|y_{w^{(i^*)}} - \bar{y}_{\mathcal{U}^{(i^*)}}\| = \max_{w \in \mathcal{U}^{(i^*)}} \|y_w - \bar{y}_{\mathcal{U}^{(i^*)}}\| .
						\end{align}
						
						Note that
						\begin{align}\label{proof:faba-case2-1}
							\|\bar{y}_{\mathcal{U}^{(W - R)}} - \bar{y}\| \leq \|\bar{y}_{\mathcal{U}^{(W - R)}} - \bar{y}_{\mathcal{U}^{(i^*)}}\| + \|\bar{y}_{\mathcal{U}^{(i^*)}} - \bar{y}\|,
						\end{align}
						thus it suffices to bound the two terms at the right separately. First we notice that
						\begin{align}\label{proof:faba-case2-2}
							\|\bar{y}_{\mathcal{U}^{(W - R)}} - \bar{y}_{\mathcal{U}^{(i^*)}}\|
							=& \|\frac{1}{|\mathcal{U}^{(W - R)}|} \sum_{w \in \mathcal{U}^{(W - R)}} y_w - \frac{1}{|\mathcal{U}^{(i^*)}|} \sum_{w \in \mathcal{U}^{(i^*)}} y_w\| \\
							=& \|(\frac{1}{|\mathcal{U}^{(W - R)}|} - \frac{1}{|\mathcal{U}^{(i^*)}|})\sum_{w \in \mathcal{U}^{(i^*)}} y_w - \frac{1}{|\mathcal{U}^{(W - R)}|} \sum_{w \in \mathcal{U}^{(i^*)} \setminus \mathcal{U}^{(W - R )}}  y_w\| \nonumber \\
							=& \frac{1}{|\mathcal{U}^{(W - R)}|} \|\sum_{w \in \mathcal{U}^{(i^*)} \setminus \mathcal{U}^{(W - R)}} (y_w - \bar{y}_{\mathcal{U}^{(i^*)}})\| \nonumber\\
							\leq&\frac{|\mathcal{U}^{(i^*)}| - |\mathcal{U}^{(W - R)}|}{|\mathcal{U}^{(W - R)}|} \max_{w \in \mathcal{U}^{(i^*)} \setminus \mathcal{U}^{(W - R)}} \|y_w - \bar{y}_{\mathcal{U}^{(i^*)}}\|\nonumber\\
							\leq& \frac{W - R -i^*}{R} \|y_{w^{(i^*)}} - \bar{y}_{\mathcal{U}^{(i^*)}}\| \nonumber\\
							\leq & \frac{W - R -i^*}{R} (\|y_{w^{(i^*)}} - \bar{y}_{\mathcal{R}^{(i^*)}}\| +\| \bar{y}_{\mathcal{R}^{(i^*)}} -  \bar{y}_{\mathcal{U}^{(i^*)}}\|  ) \nonumber\\
							= & \frac{W - R -i^*}{R} (\|y_{w^{(i^*)}} - \bar{y}_{\mathcal{R}^{(i^*)}}\| + u_{i^*}\| \bar{y}_{\mathcal{R}^{(i^*)}} -  \bar{y}_{\mathcal{P}^{(i^*)}}\|  ) \nonumber \\
							\leq &\frac{W - R -i^*}{R} \left(\max_{w \in \mathcal{R}}\|y_{w} -  \bar{y}_{\mathcal{R}^{(i^*)}}\|  +  \frac{u_{i^*}}{1 - 2u_{i^*}} \cdot\max_{w \in \mathcal{R}}\|y_{w} -  \bar{y}_{\mathcal{R}^{(i^*)}}\|\right) \nonumber \\
							=& \frac{W - R -i^*}{R} \cdot \frac{1 - u_{i^*}}{1 - 2u_{i^*}} \max_{w \in \mathcal{R}}\|y_{w} - \bar{y}_{\mathcal{R}^{(i^*)}}\|, \nonumber
						\end{align}
						where the second inequality is from \eqref{eq:faba-principle} and the last inequality is from \eqref{eq:claim-for-faba} with $i=i^*$.
						Additionally, the second term at the right-hand side of \eqref{proof:faba-case2-1} has upper bound
						\begin{align}\label{proof:faba-case2-3}
							\|\bar{y}_{\mathcal{U}^{(i^*)}} - \bar{y}\| = \|(1 - u_{i^*}) \cdot \bar{y}_{\mathcal{R}^{(i^*)}} + u_{i^*} \cdot \bar{y}_{\mathcal{P}^{(i^*)}} - \bar{y}\| \leq \| \bar{y}_{\mathcal{R}^{(i^*)}}   - \bar{y}\| + u_{i^*} \| \bar{y}_{\mathcal{R}^{(i^*)}} - \bar{y}_{\mathcal{P}^{(i^*)}}\|.
						\end{align}
						For $\| \bar{y}_{\mathcal{R}^{(i^*)}} - \bar{y}\| $, we have
						\begin{align}\label{proof:faba-case2-4}
							\| \bar{y}_{\mathcal{R}^{(i^*)}}   - \bar{y}\|
							=& \|\frac{1}{|\mathcal{R}^{(i^*)}|} \sum_{w \in \mathcal{R}^{(i^*)}} y_w - \frac{1}{R} \sum_{w \in \mathcal{R}} y_w\| \\
							=& \frac{1}{|\mathcal{R}^{(i^*)}|} \| \sum_{w \in \mathcal{R} \setminus \mathcal{R}^{i^*}} (\bar{y} - y_w) \| \nonumber \\
							\leq& \frac{R-|\mathcal{R}^{(i^*)}|}{|\mathcal{R}^{(i^*)}|}  \max_{w\in  \mathcal{R}} \| y_w - \bar{y}\|. \nonumber
						\end{align}
						Substituting \eqref{proof:faba-case2-4} and \eqref{eq:claim-for-faba} into \eqref{proof:faba-case2-3}, we have
						\begin{align}\label{proof:faba-case2-5}
							\|\bar{y}_{\mathcal{U}^{(i^*)}} - \bar{y}\| \leq \frac{R-|\mathcal{R}^{(i^*)}|}{|\mathcal{R}^{(i^*)}|}  \max_{w\in  \mathcal{R}} \| y_w - \bar{y}\| + \frac{u_{i^*}}{1 - 2u_{i^*}} \max_{w \in \mathcal{R}} \|y_w - \bar{y}_{\mathcal{R}^{(i^*)}}\|.
						\end{align}
						Substituting \eqref{proof:faba-case2-2} and \eqref{proof:faba-case2-5} into \eqref{proof:faba-case2-1}, we have
						\begin{align}
							&\|\bar{y}_{\mathcal{U}^{(W - R)}} - \bar{y}\| \\ \leq  &(\frac{W - R -i^*}{R} \cdot \frac{1 - u_{i^*}}{1 - 2u_{i^*}} + \frac{u_{i^*}}{1 - 2u_{i^*}}) \max_{w \in \mathcal{R}} \|y_w - \bar{y}_{\mathcal{R}^{(i^*)}}\| + \frac{R-|\mathcal{R}^{(i^*)}|}{|\mathcal{R}^{(i^*)}|}  \max_{w\in  \mathcal{R}} \| y_w - \bar{y}\|. \nonumber
						\end{align}
						Since
						\begin{align}
							\max_{w \in \mathcal{R}} \|y_w - \bar{y}_{\mathcal{R}^{(i^*)}}\| \leq \max_{w \in \mathcal{R}} \|y_w - \bar{y} \| + \| \bar{y}_{\mathcal{R}^{(i^*)}} - \bar{y}\| \leq \frac{R}{|\mathcal{R}^{(i^*)}|}  \max_{w\in  \mathcal{R}} \| y_w - \bar{y}\|,
						\end{align}
						where the last inequality comes from \eqref{proof:faba-case2-4}, we have
						\begin{align}
							&\|\bar{y}_{\mathcal{U}^{(W - R)}} - \bar{y}\| \\
							\leq& \Big((\frac{W - R -i^*}{R} \cdot \frac{1 - u_{i^*}}{1 - 2u_{i^*}} + \frac{u_{i^*}}{1 - 2u_{i^*}} ) \frac{R}{|\mathcal{R}^{(i^*)}|}  + \frac{R-|\mathcal{R}^{(i^*)}|}{|\mathcal{R}^{(i^*)}|}\Big) \max_{w \in \mathcal{R}} \|y_w - \bar{y}\| \nonumber \\
							=& \frac{2u_{i^*}}{1 - 2u_{i^*}}  \max_{w \in \mathcal{R}} \|y_w - \bar{y}\|.
						\end{align}
						Since $u_{i^*} = \frac{|\mathcal{P}^{(i^*)}|}{|\mathcal{U}^{(i^*)}|} \leq \frac{W - R}{R}=\frac{\delta}{1-\delta}$, we have
						\begin{align}\label{proof:faba-second-case}
							\|\bar{y}_{\mathcal{U}^{(W - R)}} - \bar{y}\| \leq \frac{2\delta}{1 - 3\delta}  \cdot \max_{w \in \mathcal{R}} \|y_w - \bar{y}\|.
						\end{align}

						Combining the first case \eqref{proof:faba-first-case} and the second case \eqref{proof:faba-second-case}, we have
						\begin{align}
							\|\text{FABA}(\{y_1, \ldots, y_W\} - \bar{y})\|	=\|\bar{y}_{\mathcal{U}^{(P)}} - \bar{y}\| \leq \frac{2\delta}{1 - 3\delta}  \cdot \max_{w \in \mathcal{R}} \|y_w - \bar{y}\|,
						\end{align}
						which completes the proof.		
					\end{proof}
					
					\section{Proof of Theorem \ref{thm:convergence with RAgg}}
					\label{sec: Appendix C}

					We give a complete version of Theorem \ref{thm:convergence with RAgg} as follows.
					\begin{theorem}\label{complete version of convergence with RAgg}
						Consider Algorithm \ref{algorithm: 1} with a $\rho$-robust aggregator $\text{RAgg}(\cdot)$ to solve \eqref{problem} and suppose that Assumptions \ref{Assump: lower boundedness}, \ref{Assump: L-smooth}, \ref{Assump: bounded heterogeneity}, and \ref{Assump: bounded variance} hold.
						Under label poisoning attacks where the fraction of poisoned workers is $\delta \in [0, \frac{1}{2})$, if the step size is
						\begin{align}
							\gamma = \min\Big\{\sqrt{\frac{4(f(x^0) - f^*) + \frac{15\rho^2(R+\frac{1}{R})\sigma^2}{8L}}{T(40L\sigma^2)(3\rho^2(R+\frac{1}{R}) + \frac{2}{R})}}, \frac{1}{8L}\Big\},
						\end{align}
						the momentum coefficient $\alpha = 8L\gamma$, then we have
						\begin{align}\label{eq: thm1}
							&\frac{1}{T} \sum_{t=1}^{T} \E\|\nabla f(x^t)\|^2 \\ \leq&  15 \rho^2  \xi^2 +  \sqrt{\frac{20L\sigma^2(\frac{2}{R} + 3\rho^2 (R+\frac{1}{R}))}{T}} \cdot \sqrt{32(f(x^0) - f^*) + \frac{15}{L}\rho^2(R+\frac{1}{R})\sigma^2} \nonumber\\ + &\frac{32L(f(x^0) - f^*)}{T} + \frac{15\rho^2 (R+\frac{1}{R})\sigma^2}{T} + \frac{\frac{10\sigma^2}{R} + 12\rho^2((R+\frac{1}{R})\sigma^2 + \xi^2) - \|\nabla f(x^0)\|}{T}. \nonumber
						\end{align}
						where the expectation is taken over the algorithm's randomness.
					\end{theorem}
					
					\begin{proof}
						For notational convenience, denote $m^t = \text{RAgg}(\{\hat{m}_w^t: w \in \mathcal{W}\})$ and $\bar{m}^t = \frac{1}{R} \sum_{w \in \mathcal{R}} m_w^t$. Denote the conditional expectation $\E[\cdot | i_w^\tau: \tau < t, w \in \mathcal{W}]$ as $\E_t[\cdot]$. Because $f(x)$ has $L$-Lipschitz continuous gradients from Assumption \ref{Assump: L-smooth}, it holds that
						\begin{align}
							f(x^{t+1}) &\leq f(x^t) + \langle \nabla f(x^t), x^{t+1} - x^t \rangle + \frac{L}{2} \|x^{t+1} - x^t\|^2 \\
							&\leq f(x^t) - \gamma \langle \nabla f(x^t), m^t \rangle + \frac{L}{2} \gamma^2\|m^t\|^2. \nonumber
						\end{align}
						Since
						\begin{align}
							-\langle  \nabla f(x^t), m^t\rangle = \frac{1}{2} \|m^t -\nabla f(x^t)\|^2 - \frac{1}{2} \|\nabla f(x^t)\|^2 - \frac{1}{2} \|m^t\|^2,
						\end{align}
						we have
						\begin{align}\label{proof:thm1-1}
							f(x^{t+1}) \leq& f(x^t) + \frac{\gamma}{2} \|m^t - \nabla f(x^t)\|^2 - \frac{\gamma}{2} \|\nabla f(x^t)\|^2 - \frac{\gamma}{2} (1 - L \gamma) \|m^t\|^2 \\
							\leq &f(x^t) + \frac{\gamma}{2} \|m^t - \nabla f(x^t)\|^2 - \frac{\gamma}{2} \|\nabla f(x^t)\|^2\nonumber \\
							\leq &f(x^t) + \gamma \|m^t - \bar{m}^t\|^2 + \gamma \|\bar{m}^t - \nabla f(x^t)\|^2 - \frac{\gamma}{2} \|\nabla f(x^t)\|^2\nonumber \\
							= &f(x^t) + \gamma \|m^t - \bar{m}^t\|^2 + \gamma \|e^t\|^2 - \frac{\gamma}{2} \|\nabla f(x^t)\|^2,\nonumber
						\end{align}
						where the second inequality is from $\gamma \leq \frac{1}{8L} < \frac{1}{L}$ and $e^t \triangleq \bar{m}^t - \nabla f(x^t)$. Taking expectations on both sides of \eqref{proof:thm1-1} reaches
						\begin{align}
							\E[f(x^{t+1})] \leq \E[f(x^t)] + \gamma \E\|m^t - \bar{m}^t\|^2 + \gamma \E\|e^t\|^2 - \frac{\gamma}{2} \E \|\nabla f(x^t)\|^2.
						\end{align}

						For the term $\|m^t - \bar{m}^t\|^2$, since $\text{RAgg}(\cdot)$ is a $\rho$-robust aggregator with Definition \ref{Def: (delta_max, rho)-RAR}, it holds that
						\begin{align}\label{proof:thm1-rho}
							\|m^t - \bar{m}^t\|^2 \leq \rho^2 \cdot \max_{w \in \mathcal{R} } \|m_w^t - \bar{m}^t\|^2.
						\end{align}
						For the term $\max_{w \in \mathcal{R} } \|m_w^t - \bar{m}^t\|^2$, we have
						\begin{align}\label{proof:thm1-2}
							&\max_{w \in \mathcal{R} } \|m_w^t - \bar{m}^t\|^2 \\ \leq & \max_{w \in \mathcal{R}} \Big\{ 3 \|m_w^t - \E[m_w^t]\|^2 + 3 \|\bar{m}^t - \E[\bar{m}^t]\|^2 + 3\|\E[m_w^t] - \E[\bar{m}^t]\|^2\Big\} \nonumber \\
							\leq &3 \max_{w \in \mathcal{R} } \|m_w^t - \E[m_w^t]\|^2 + 3 \|\bar{m}^t - \E[\bar{m}^t]\|^2 + 3\max_{w \in \mathcal{R} } \|\E[m_w^t] - \E[\bar{m}^t]\|^2. \nonumber
						\end{align}
						We will bound the expectation of each term at the right-hand side of \eqref{proof:thm1-2} as follows.
						
						For the first term at the right-hand side of \eqref{proof:thm1-2}, we have
						\begin{align}
							\E[\max_{w \in \mathcal{R} } \|m_w^t - \E[m_w^t]\|^2] \leq \E[\sum_{w \in \mathcal{R} } \|m_w^t - \E[m_w^t]\|^2]  =  \sum_{w \in \mathcal{R}} \E \|m_w^t - \E[m_w^t]\|^2.
						\end{align}
						For any $w \in \mathcal{R}$, denoting
						\begin{align}
							A_w^t \triangleq  \E \|m_w^t - \E[m_w^t]\|^2,
						\end{align}
						we obtain that
						\begin{align}\label{proof: thm1-At-1}
							A_w^t &= \E\|(1 - \alpha) (m_w^{t-1} - \E[m_w^{t-1}]) + \alpha (\nabla f_{w, i_w^t}(x^t) - \nabla f_w(x^t))\|^2 \\
							&= (1 - \alpha)^2 \E \|m_w^{t-1} - \E[m_w^{t-1}]\|^2 + \alpha^2 \E\|\nabla f_{w, i_w^t}(x^t) - \nabla f_w(x^t)\|^2 \nonumber \\
							&= (1 - \alpha)^2 \E \|m_w^{t-1} - \E[m_w^{t-1}]\|^2 + \alpha^2 \E[\E_t\|\nabla f_{w, i_w^t}(x^t) - \nabla f_w(x^t)\|^2] \nonumber \\
							& \leq (1 - \alpha)^2 A_w^{t-1} + \alpha^2 \sigma^2  \nonumber \\
							& \leq (1 - \alpha)^{2t} A_w^{0} + (\sum_{l=1}^{t} (1 - \alpha)^{2(t-l)})\alpha^2 \sigma^2,  \nonumber
						\end{align}
						where the first inequality comes from Assumption \ref{Assump: bounded variance} and the second inequality is due to the use of telescopic cancellation. Since
						\begin{align}\label{proof: thm1-At-2}
							A_w^0 = \E \|m_w^0 - \E[m_w^0]\|^2 = \E\|\nabla f_{w, i_w^0}(x^0) - \nabla f_w(x^0)\|^2 \leq \sigma^2,
						\end{align}
						where the inequality is due to Assumption \ref{Assump: bounded variance}, we have
						\begin{align}\label{proof: thm1-At-3}
							A_w^t \leq \sigma^2 (\alpha + (1 - \alpha)^{2t}),
						\end{align}
						and
						\begin{align}\label{proof:thm1-At}
							\E[\max_{w \in \mathcal{R} } \|m_w^t - \E[m_w^t]\|^2] \leq R\sigma^2 (\alpha + (1 - \alpha)^{t+1}).
						\end{align}

						For the second term at the right-hand side of \eqref{proof:thm1-2}, denoting
						\begin{align}
							B^t \triangleq  \E \|\bar{m}^t - \E[\bar{m}^t]\|^2,
						\end{align}
						we have that
						\begin{align}\label{proof: thm1-Bt-1}
							B^t &= \E \|\frac{1}{R} \sum_{w \in \mathcal{R}} (m_w^t - \E[m_w^t])\|^2 \\
							& = \frac{1}{R^2} \E\|\sum_{w \in \mathcal{R}} (m_w^t - \E[m_w^t]) \|^2 \nonumber \\
							& = \frac{1}{R^2} \Big(\sum_{w \in \mathcal{R}} \E\| m_w^t - \E[m_w^t] \|^2 + \sum_{w \in \mathcal{R}} \sum_{v \in \mathcal{R}, v \neq w} \E \langle m_w^t - \E[m_w^t], m_v^t - \E[m_v^t] \rangle\Big) \nonumber \\
							& = \frac{1}{R^2} \Big(\sum_{w \in \mathcal{R}} \E\| m_w^t - \E[m_w^t] \|^2 + \sum_{w \in \mathcal{R}} \sum_{v \in \mathcal{R}, v \neq w} \underbrace{ \langle \E[m_w^t] - \E[m_w^t], \E[m_v^t] - \E[m_v^t] \rangle}_{=0}\Big) \nonumber \\
							& = \frac{1}{R^2}  \sum_{w \in \mathcal{R}} \E\| m_w^t - \E[m_w^t] \|^2 \nonumber\\
							& = \frac{1}{R^2} \sum_{w \in \mathcal{R}} A_w^t. \nonumber
						\end{align}
						With \eqref{proof: thm1-At-3},  we obtain that
						\begin{align}\label{proof:thm1-Bt}
							B^t \leq \frac{\sigma^2}{R} (\alpha + (1 - \alpha)^{2t}).
						\end{align}
						
						For the third term at the right-hand side of \eqref{proof:thm1-2}, denoting
						\begin{align}
							C^t \triangleq \max_{w \in \mathcal{R} } \|\E[m_w^t] - \E[\bar{m}^t]\|^2,
						\end{align}
						we have
						\begin{align}
							C^t =& \max_{w \in \mathcal{R} } \|(1 - \alpha) (\E[m_w^{t-1}] - \E[\bar{m}^{t-1}]) + \alpha (\nabla f_w(x^t) - \nabla f(x^t))\|^2 \\
							\leq & \max_{w \in \mathcal{R} } \Big\{ (1 - \alpha) \|\E[m_w^{t-1}] - \E[\bar{m}^{t-1}]\|^2 + \alpha \|\nabla f_w(x^t) - \nabla f(x^t)\|^2\Big\} \nonumber \\
							\leq &  (1 - \alpha)  \max_{w \in \mathcal{R} }  \|\E[m_w^{t-1}] - \E[\bar{m}^{t-1}]\|^2 + \alpha\max_{w \in \mathcal{R} } \|\nabla f_w(x^t) - \nabla f(x^t)\|^2 \nonumber \\
							\leq & (1 - \alpha) C^{t-1} + \alpha \xi^2 \nonumber \\
							\leq & (1 - \alpha)^t C^0 + (\sum_{l=0}^{t-1} (1 - \alpha)^l)\alpha  \xi^2, \nonumber
						\end{align}
						where the third inequality is from Assumption \ref{Assump: bounded heterogeneity}. Since
						\begin{align}
							C^0 = \max_{w \in \mathcal{R} } \|\nabla f_w(x^0) - \nabla f(x^0)\|^2 \leq \xi^2,
						\end{align}
						which come from Assumption \ref{Assump: bounded heterogeneity}, it holds that
						\begin{align}\label{proof:thm1-Ct}
							C^t \leq   \xi^2  \Big((1 - \alpha)^t + (\sum_{l=0}^{t-1} (1 - \alpha)^l)\alpha \Big) = \xi^2.
						\end{align}
						
						Substituting \eqref{proof:thm1-At}, \eqref{proof:thm1-Bt} and \eqref{proof:thm1-Ct} into \eqref{proof:thm1-2} and taking expectations on both sides of \eqref{proof:thm1-2}, we have
						\begin{align}
							\E [\max_{w \in \mathcal{R} } \|m_w^t - \bar{m}^t\|^2] \leq 3((R+\frac{1}{R}) \sigma^2 (\alpha + (1 - \alpha)^{2t}) + \xi^2).
						\end{align}
						With \eqref{proof:thm1-rho}, we have
						\begin{align}\label{proof:thm1-aggregation-error}
							\E \|m^t - \bar{m}^t\|^2 \leq 3\rho^2((R+\frac{1}{R}) \sigma^2 (\alpha + (1 - \alpha)^{2t}) + \xi^2).
						\end{align}

						For the term $\E\|e^t\|^2$ in \eqref{proof:thm1-1}, according to \citep[Lemma 10]{karimireddy2021byzantine}, $\E\|e^0\|^2 \leq \frac{\sigma^2}{R}$ and for $t \geq 1$, it holds that
						\begin{align}\label{proof:thm1-et}
							\E\|e^t\|^2 \leq (1-\frac{2\alpha}5) \E\|e^{t-1}\|^2 + \frac\alpha{10} \E\|\nabla f(x^{t-1}) \|^2 + \frac\alpha{10} \E\|m^{t-1}-\bar{m}^{t-1}\|^2 + \frac{\alpha^2 \sigma^2}R.
						\end{align}

						Combining \eqref{proof:thm1-1} and \eqref{proof:thm1-et}, we have
						\begin{align}\label{proof:thm1-3}
							&\E[f(x^{t+1})] + \frac{5\gamma}{2\alpha} \E\|e^t\|^2 \\  \leq & \E[f(x^t)] +\gamma \E\|m^t - \bar{m}^t\|^2 + \gamma \E\|e^t\|^2 - \frac{\gamma}{2}\E \|\nabla f(x^t)\|^2  + (\frac{5\gamma}{2 \alpha} - \gamma) \E\|e^{t-1}\|^2  \nonumber\\ &+ \frac{\gamma}{4} \E\|\nabla f(x^{t-1})\|^2 + \frac{\gamma}{4} \E\|m^{t-1} - \bar{m}^{t-1}\|^2 + \frac{5\alpha\gamma\sigma^2}{2R}. \nonumber
						\end{align}
						Using \eqref{proof:thm1-aggregation-error} to bound $\E\|m^t - \bar{m}^t\|^2$ and $\E \|m^{t-1} - \bar{m}^{t-1}\|^2$, it is obtained that
						\begin{align}
							&\E[f(x^{t+1})] + \frac{5\gamma}{2\alpha} \E\|e^t\|^2 \\ \leq &\E[f(x^t)] +\frac{15}{4} \gamma \rho^2 (R+\frac{1}{R}) \sigma^2 (\alpha + (1 - \alpha)^{2t}) + \frac{15}{4} \gamma \rho^2\xi^2 + \gamma \E\|e^t\|^2 \nonumber \\ & - \frac{\gamma}{2}\E \|\nabla f(x^t)\|^2 + (\frac{5\gamma}{2 \alpha} - \gamma) \E\|e^{t-1}\|^2  + \frac{\gamma}{4}\E \|\nabla f(x^{t-1})\|^2 +  \frac{5\alpha\gamma\sigma^2}{2R}. \nonumber
						\end{align}
						Therefore, we have
						\begin{align}
							&\E[f(x^{t+1})] + (\frac{5\gamma}{2\alpha} - \gamma) \E\|e^t\|^2 + \frac{\gamma}{4} \E \|\nabla f(x^t)\|^2 \\ \leq& \E[f(x^{t})] + (\frac{5\gamma}{2\alpha} - \gamma) \E\|e^{t-1}\|^2 + \frac{\gamma}{4} \E \|\nabla f(x^{t-1})\|^2 - \frac{\gamma}{4}\E \|\nabla f(x^t)\|^2 \nonumber\\ &+ \frac{15}{4} \gamma \rho^2 (R+\frac{1}{R}) \sigma^2 (\alpha + (1 - \alpha)^{2t}) + \frac{15}{4} \gamma \rho^2\xi^2 +  \frac{5\alpha\gamma\sigma^2}{2R}. \nonumber
						\end{align}
						Denoting $\mathcal{E}_t = \E[f(x^{t+1})] + (\frac{5\gamma}{2\alpha} - \gamma) \E\|e^t\|^2 + \frac{\gamma}{4} \E \|\nabla f(x^t)\|^2 $, we have
						\begin{align}
							\frac{\gamma}{4} \E \|\nabla f(x^t)\|^2 \leq \mathcal{E}_{t-1} - \mathcal{E}_t +  \frac{15}{4} \gamma \rho^2 (R+\frac{1}{R}) \sigma^2 (\alpha + (1 - \alpha)^{2t}) + \frac{15}{4} \gamma \rho^2\xi^2 +  \frac{5\alpha\gamma\sigma^2}{2R},
						\end{align}
						and
						\begin{align}\label{proof:thm1-avergage-over-t}
							&\frac{1}{T} \sum_{t=1}^{T} \E \|\nabla f(x^t)\|^2 \\\leq& \frac{4 (\mathcal{E}_0 - \mathcal{E}_T )}{\gamma T} + \frac{1}{T}\sum_{t=1}^{T}15\rho^2 (R+\frac{1}{R}) \sigma^2 (\alpha + (1 - \alpha)^{2t}) + 15 \rho^2 \xi^2 + \frac{10 \alpha \sigma^2}{R} \nonumber\\
							\leq& \frac{4 (\mathcal{E}_0 - \mathcal{E}_T )}{\gamma T} + \frac{15\rho^2 (R+\frac{1}{R}) \sigma^2}{\alpha T} + 15\rho^2 (R+\frac{1}{R}) \sigma^2 \alpha + 15 \rho^2 \xi^2 + \frac{10 \alpha \sigma^2}{R}.  \nonumber \\
							\leq& \frac{4 (\mathcal{E}_0 - f^* )}{\gamma T} + \frac{15\rho^2 (R+\frac{1}{R}) \sigma^2}{\alpha T} + 15\rho^2 (R+\frac{1}{R}) \sigma^2 \alpha + 15 \rho^2 \xi^2 + \frac{10 \alpha \sigma^2}{R}.  \nonumber
						\end{align}
						where the second inequality is from $\sum_{t=1}^{T} (1 - \alpha)^{2t} \leq \frac{1}{\alpha}$ when $0 \leq \alpha \leq 1$, and the third inequality is from $\mathcal{E}_T \geq f^*$ due to Assumption \ref{Assump: lower boundedness}. Using the following equalities and inequalities
						\begin{align}
							\mathcal{E}_0 &= \E[f(x^1)] + \frac{3\gamma}{2} \E \|e^0\|^2 + \frac{\gamma}{4} \|\nabla f(x^0)\|^2, \\
							\E[f(x^1)] &\leq f(x^0) + \gamma \E \|m^0 - \bar{m}^0\|^2 + \gamma \E\|e^0\|^2 - \frac{\gamma}{2} \|\nabla f(x^0)\|^2, \\
							\E\|e^0\|^2 &= \E \|\bar{m}^0 - \nabla f(x^0)\|^2 = \E \|\frac{1}{R} \sum_{w \in \mathcal{R}} (\nabla f_{w, i_w^0}(x^0) - \nabla f_w(x^0))\|^2 \leq \frac{\sigma^2}{R}, \\
							\E\|m^0 - \bar{m}^0\|^2 &\leq \rho^2 \E[\max_{w \in \mathcal{R}} \|m_w^0 - \bar{m}^0\|^2] \leq 3\rho^2 ((R+\frac{1}{R}) \sigma^2 + \xi^2),
						\end{align}
						we have
						\begin{align}\label{proof:thm1-e0}
							\mathcal{E}_0 \leq f(x^0) + \frac{5\gamma \sigma^2}{2 R} + 3\gamma \rho^2 ((R+\frac{1}{R}) \sigma^2 + \xi^2) - \frac{\gamma}{4} \|\nabla f(x^0)\|^2.
						\end{align}

						Substituting \eqref{proof:thm1-e0} into \eqref{proof:thm1-avergage-over-t}, we have
						\begin{align}
							&\frac{1}{T} \sum_{t=1}^{T} \E \|\nabla f(x^t)\|^2 \\
							\leq& \frac{4 (f(x^0) - f^*)}{\gamma T} + \frac{15\rho^2 (R+\frac{1}{R}) \sigma^2}{8\gamma L T} + 120\rho^2 (R+\frac{1}{R}) \sigma^2\gamma L + \frac{80L\gamma \sigma^2}{R} \nonumber\\&+\frac{\frac{10\sigma^2}{R} + 12\rho^2((R+\frac{1}{R})\sigma^2 + \xi^2) - \|\nabla f(x^0)\|}{T}  + 15 \rho^2 \xi^2, \nonumber \\ = & \frac{1}{\gamma} \cdot \Big(\frac{4 (f(x^0) - f^*)}{T} + \frac{15\rho^2 (R+\frac{1}{R}) \sigma^2}{8LT}\Big) + \gamma \cdot\Big(120\rho^2 (R+\frac{1}{R}) \sigma^2 L + \frac{80L\sigma^2}{R}\Big) \nonumber\\
							&+\frac{\frac{10\sigma^2}{R} + 12\rho^2((R+\frac{1}{R})\sigma^2 + \xi^2) - \|\nabla f(x^0)\|}{T}  + 15 \rho^2 \xi^2, \nonumber
						\end{align}
						where the inequality uses the fact that $\alpha = 8L \gamma$. Substituting the step size
						\begin{align}
							\gamma = \min\Big\{\sqrt{\frac{4(f(x^0) - f^*) + \frac{15\rho^2(R+\frac{1}{R})\sigma^2}{8L}}{T(40L\sigma^2)(3\rho^2(R+\frac{1}{R}) + \frac{2}{R})}}, \frac{1}{8L}\Big\},
						\end{align}
						we have
						\begin{align}
							\frac{1}{\gamma} = & \max\Big\{\sqrt{\frac{T(40L\sigma^2)(3\rho^2(R+\frac{1}{R})+ \frac{2}{R})}{4(f(x^0) - f^*) + \frac{15\rho^2(R+\frac{1}{R})\sigma^2}{8L}}}, 8L\Big\} \\
							\leq & \sqrt{\frac{T(40L\sigma^2)(3\rho^2(R+\frac{1}{R})+ \frac{2}{R})}{4(f(x^0) - f^*) + \frac{15\rho^2(R+\frac{1}{R})\sigma^2}{8L}}} + 8L. \notag
						\end{align}
						Thus, we have
						\begin{align}
							&\frac{1}{T} \sum_{t=1}^{T} \E \|\nabla f(x^t)\|^2 \\ \leq & 15 \rho^2 \xi^2 +  \sqrt{\frac{20L\sigma^2(\frac{2}{R} + 3\rho^2 (R+\frac{1}{R}))}{T}} \cdot \sqrt{32(f(x^0) - f^*) + \frac{15}{L}\rho^2(R+\frac{1}{R})\sigma^2} \nonumber\\ &+ \frac{32L(f(x^0) - f^*)}{T} + \frac{15\rho^2 (R+\frac{1}{R})\sigma^2}{T} + \frac{\frac{10\sigma^2}{R} + 12\rho^2((R+\frac{1}{R})\sigma^2 + \xi^2) - \|\nabla f(x^0)\|}{T}. \nonumber
						\end{align}
						which completes the proof.
					\end{proof}

					\section{Proof of Theorem \ref{thm:convergence with Mean}}
					\label{sec: Appendix D}

					We give a complete version of Theorem \ref{thm:convergence with Mean} as follows.
					\begin{theorem}\label{complete version of convergence with Mean}
						Consider Algorithm \ref{algorithm: 1} with the mean aggregator $\text{Mean}(\cdot)$ to solve \eqref{problem} and suppose that Assumptions \ref{Assump: lower boundedness}, \ref{Assump: L-smooth}, \ref{Assump: bounded variance}, and \ref{Assump: bounded effect of poisoned local gradients} hold.
						Under label poisoning attacks where the fraction of poisoned workers is $\delta \in [0, 1)$, if the step size is
						\begin{align}
							\gamma = \min\Big\{\sqrt{\frac{4(f(x^0) - f^*) + \frac{30\delta^2\sigma^2}{8L}}{T(40L\sigma^2)(6\delta^2+ \frac{2}{R})}}, \frac{1}{8L}\Big\},
						\end{align}
						the momentum coefficient $\alpha = 8L\gamma$, then we have
						\begin{align}
							\frac{1}{T} \sum_{t=1}^{T} \E\|\nabla f(x^t)\|^2\leq & 15 \delta^2A^2  + \sqrt{\frac{20L\sigma^2(\frac{2}{R} + 6\delta^2)}{T}} \cdot \sqrt{32(f(x^0) - f^*) + \frac{30}{L}\delta^2\sigma^2} \\ &+ \frac{32L(f(x^0) - f^*)}{T} + \frac{30\delta^2\sigma^2}{T}+ \frac{\frac{10\sigma^2}{R} + 24\delta^2(\sigma^2 + \xi^2) - \|\nabla f(x^0)\|}{T} . \nonumber
						\end{align}
						where the expectation is taken over the algorithm's randomness.
					\end{theorem}	
					\begin{proof}
						The proof of Theorem \ref{thm:convergence with Mean} is similar to that of Theorem \ref{thm:convergence with RAgg}, except for the analysis of the distance between the aggregated message and the true average of the regular messages. For notational simplicity, denote the conditional expectation $\E[\cdot | i_w^\tau: \tau < t, w \in \mathcal{W}]$ as $\E_t[\cdot]$. Denoting $m^t =$ $\text{Mean}(\{\hat{m}_w^t: w \in \mathcal{W}\})$ and $\bar{m}^t = \frac{1}{R} \sum_{w \in \mathcal{R}} m_w^t$, we have
						\begin{align}
							\E\|m^t - \bar{m}^t\|^2 &= \E\|\frac{1}{W} \sum_{w \in \mathcal{W}} \hat{m}_w^t - \frac{1}{R} \sum_{w \in \mathcal{R}} m_w^t\|^2 \\ &= \E \|\frac{1}{W} \sum_{w \in \mathcal{W}\setminus\mathcal{R}} \tilde{m}_w^t - (\frac{1}{R} - \frac{1}{W}) \sum_{w \in \mathcal{R}} m_w^t\|^2 \nonumber \\ &= (\frac{1}{W})^2 \E\|\sum_{w \in \mathcal{W}\setminus\mathcal{R}}(\tilde{m}_w^t - \bar{m}^t)\|^2 \nonumber \\ &\leq \frac{W - R}{W^2} \sum_{w \in \mathcal{W}\setminus\mathcal{R}} \E \|\tilde{m}_w^t - \bar{m}^t\|^2. \nonumber
						\end{align}

						For the term $\E \|\tilde{m}_w^t - \bar{m}^t\|^2$, for any $w \in \mathcal{W} \setminus \mathcal{R}$, we have
						\begin{align}\label{proof:thm2-1}
							\E \|\tilde{m}_w^t - \bar{m}^t\|^2 \leq 3\Big(\E\|\tilde{m}_w^t - \E[\tilde{m}_w^t]\|^2 + \E\|\bar{m}^t - \E[\bar{m}^t]\|^2 + \|\E[\tilde{m}_w^t] - \E[\bar{m}^t]\|^2\Big).
						\end{align}
						For the first term at the right-hand side of \eqref{proof:thm2-1}, similar to the proof from \eqref{proof: thm1-At-1} to \eqref{proof: thm1-At-3}, for any $w \in \mathcal{W} \setminus \mathcal{R}$, we have
						\begin{align}\label{proof:thm2-2}
							\E\|\tilde{m}_w^t - \E[\tilde{m}_w^t]\|^2	&= \E\|(1 - \alpha) (\tilde{m}_w^{t-1} - \E[\tilde{m}_w^{t-1}]) + \alpha (\nabla \tilde{f}_{w, i_w^t}(x^t) - \nabla \tilde{f}_w(x^t))\|^2 \\
							& =(1 - \alpha)^2 \E\|\tilde{m}_w^{t-1} - \E[\tilde{m}_w^{t-1}] \|^2 + \alpha^2  \E \|\nabla \tilde{f}_{w, i_w^t}(x^t) - \nabla \tilde{f}_w(x^t)\|^2 \nonumber \\
							& =(1 - \alpha)^2 \E\|\tilde{m}_w^{t-1} - \E[\tilde{m}_w^{t-1}] \|^2 + \alpha^2  \E [\E_t\|\nabla \tilde{f}_{w, i_w^t}(x^t) - \nabla \tilde{f}_w(x^t)\|^2] \nonumber \\
							& =(1 - \alpha)^2 \E\|\tilde{m}_w^{t-1} - \E[\tilde{m}_w^{t-1}] \|^2 + \alpha^2 \sigma^2\nonumber \\
							&\leq \sigma^2(\alpha + (1 - \alpha)^{2t}). \nonumber
						\end{align}
						For the second term at the right-hand side of \eqref{proof:thm2-1}, similar to the proof of \eqref{proof: thm1-Bt-1}, we have
						\begin{align}\label{proof:thm2-3}
							\E\|\bar{m}^t - \E[\bar{m}^t]\|^2	&= \E\|\frac{1}{R} \sum_{w \in \mathcal{R}} m_w^t - \E[m_w^t]\|^2 \\
							&= \frac{1}{R^2} \sum_{w \in \mathcal{R}} \E \|m_w^t - \E[m_w^t]\|^2 \nonumber \\
							&\leq \frac{\sigma^2}{R}(\alpha + (1 - \alpha)^{2t}). \nonumber
						\end{align}
						For the third term at the right-hand side of \eqref{proof:thm2-1}, we have
						\begin{align}\label{proof:thm2-4}
							\|\E[\tilde{m}_w^t] - \E[\bar{m}^t]\|^2 &= \|(1 - \alpha)(\E[\tilde{m}_w^{t-1}] - \E[\bar{m}^{t-1}]) + \alpha (\nabla \tilde{f}_w(x^t) - \nabla f(x^t))\|^2 \\
							&\leq   (1 - \alpha) \|\E[\tilde{m}_w^{t-1}] - \E[\bar{m}^{t-1}] \|^2 + \alpha \|\nabla \tilde{f}_w(x^t) - \nabla f(x^t)\|^2 \nonumber \\
							&\leq   (1 - \alpha) \|\E[\tilde{m}_w^{t-1}] - \E[\bar{m}^{t-1}] \|^2 + \alpha A^2 \nonumber \\
							&\leq   (1 - \alpha)^t \|\E[\tilde{m}_w^{0}] - \E[\bar{m}^{0}] \|^2 + (\sum_{l=0}^{t-1} (1 - \alpha)^{l})\alpha A^2 \nonumber \\
							&=   (1 - \alpha)^t \|\nabla \tilde{f}_w(x^0)- \nabla f(x^0) \|^2 + (\sum_{l=0}^{t-1} (1 - \alpha)^{l})\alpha A^2 \nonumber \\
							&\leq A^2 ((1 - \alpha)^t + 1 - (1 - \alpha)^t) \nonumber \\
							&=A^2, \nonumber
						\end{align}
						where the second inequality and the fourth inequality are due to Assumption \ref{Assump: bounded effect of poisoned local gradients}.
						
						Substituting \eqref{proof:thm2-1}, \eqref{proof:thm2-2}, \eqref{proof:thm2-3} into \eqref{proof:thm2-4}, we have
						\begin{align}
							\E\|m^t - \bar{m}^t\|^2 &\leq 3 \cdot \frac{(W - R)^2}{W^2} \cdot (\sigma^2( \alpha + (1 - \alpha)^{2t})(1 + \frac{1}{R}) + A^2) \\ &\leq 3 \cdot \frac{(W - R)^2}{W^2} \cdot (2\sigma^2( \alpha + (1 - \alpha)^{2t}) + A^2) \nonumber \\ &= 6\delta^2\sigma^2 ( \alpha + (1 - \alpha)^{2t}) + 3\delta^2 A^2, \nonumber
						\end{align}
						where the second inequality is due to $R \geq 1$.
						
						The rest of the proof is the same as that of Theorem \ref{thm:convergence with RAgg} and therefore omitted. This completes the proof.
					\end{proof}

					\section{Proof of Theorem \ref{thm:lower bound}}
					\label{sec: Appendix E}
					
					\begin{proof}
						The key idea of the proof is to construct two instances that have the same set of $W$ local costs after the poisoned workers carry out label poisoning attacks, while the objectives based on the $R$ regular local costs are different in these two instances. Since the algorithm whose output is invariant with respect to the identities of the workers cannot distinguish these two instances, it yields the same output. However, the two different objectives imply that the algorithmic output must fail on at least one among the two. Therefore, the instance that has a larger learning error gives a lower bound of the learning error for the algorithm.
						
						Without loss of generality, we let $\mathcal{W} = \{1, \cdots, W\}$ be the set of workers, within which $\mathcal{R} = \{1,\cdots,R\}$ is the set of regular workers. The samples of all workers have two possible labels, $1$ or $2$, which correspond to two different functions $f(x;1)$ and $f(x;2)$ with

							\begin{align}\label{proof: th9-case1-first-equation}
								f(x;k) = \frac{(1-\delta)c}{\sqrt{2}} [x]_k + \frac{L}{2} \|x\|^2,\quad k\in\{1,2\},
							\end{align}
							within which $c \triangleq \min\{\xi, A\}$.

						Each worker $w\in \{1, \cdots, W\}$ has the same $J$ samples costs, in the form of
						\begin{align}
							\hat{f}_{w, j}(x) = \hat{f}_w(x) = f(x;\hat{b}^{(w)}), \quad \forall j \in \{1, \cdots, J\},
						\end{align}
						or equivalently
						\begin{align}
							f_{w, j}(x) = f_w(x) &= f(x; b^{(w)}),\quad\forall w\le R, \ \forall j \in \{1, \cdots, J\},\\
							\tilde{f}_{w, j}(x) = \tilde{f}_w(x) &= f(x;\tilde{b}^{(w)}),\quad\forall w> R, \ \forall j \in \{1, \cdots, J\}.
						\end{align}

						Now we construct two instances with different sets of labels. The first set of labels, denoted as $\{\hat{b}^{(w,1)},w\in \{1, \cdots, W\}\}$, is given by
						\begin{align}
							\hat{b}^{(w,1)} = \begin{cases}
								1, &w\le R,\\
								2, &w > R.
							\end{cases}
						\end{align}
						The second set of labels, denoted as $\{\hat{b}^{(w,2)},w\in \{1, \cdots, W\}\}$, is given by
						\begin{align}
							\hat{b}^{(w,2)} = \begin{cases}
								1, &w> W-R,\\
								2, &w \le W-R.
							\end{cases}
						\end{align}

						Denote $f^{(1)}(x)=\frac1R\sum_{w=1}^Rf(x;\hat{b}^{(w,1)})$ and $f^{(2)}(x)=\frac1R\sum_{w=1}^Rf(x;\hat{b}^{(w,2)})$ as the two objectives. We can check that all the assumptions are satisfied in these two instances. Since
							\begin{align}
								\nabla f(x;k)&=\frac{(1-\delta)c}{\sqrt{2}}e_k+Lx,
							\end{align}
						where $e_k$ is the unit vector with the $k$-th element being $1$, the gradients of $f^{(1)}$ and $f^{(2)}$ are
							\begin{align}\label{gradient of instance 1}
								\nabla f^{(1)}(x) &= \frac{(1-\delta)c}{\sqrt{2}}e_1+Lx,
							\end{align}
							\begin{align}\label{gradient of instance 2}
								\nabla f^{(2)}(x) &= \frac{(1-2\delta)c}{\sqrt{2}}e_1+\frac{\delta c}{\sqrt{2}}e_2+Lx.
							\end{align}
						As a result, we know their minimums are achieved at $x^{*,(1)}=-\frac{(1-\delta)c}{\sqrt{2}L}e_1$ and $x^{*,(2)}=-\frac{(1-2\delta)c}{\sqrt{2}L}e_1-\frac{\delta c}{\sqrt{2}L}e_2$, respectively, and there exists a uniform lower bound
						\begin{align}
								f^{(k)}(x)\ge f^*\triangleq - \frac{c^2}{2L}.
						\end{align}
						satisfying Assumption \ref{Assump: lower boundedness} for $k \in \{1, 2\}$.
						
						As the gradients are linear, Assumption \ref{Assump: L-smooth} is satisfied and the constant is exactly $L$.
						
						Further, Assumption \ref{Assump: bounded heterogeneity} is satisfied with constant $\xi$, as

							\begin{align}
								\max_{w \le R } \|\nabla f(x;\hat{b}^{(w,1)}) - \nabla f^{(1)}(x)\| &=0,\\
								\max_{w \le R } \|\nabla f(x;\hat{b}^{(w,2)}) - \nabla f^{(2)}(x)\| &=\max\{|1-2\delta|, \delta\}c \leq c \leq \xi.
								\label{proof: th9-as3}
							\end{align}
							The last inequality of \eqref{proof: th9-as3} comes from $c \triangleq \min\{A,\xi\} \leq \xi$.

						Note that $ f_{w, j}(x) = \hat{f}_{w, j}(x) $ when $w \leq R$ and $\tilde{f}_{w, j}(x) = \hat{f}_{w, j}(x)$ otherwise. Therefore, due to $\hat{f}_{w, j}(x) = \hat{f}_w(x)$, Assumption \ref{Assump: bounded variance} is satisfied with constant $\sigma = 0$.

						Assumption \ref{Assump: bounded effect of poisoned local gradients} is satisfied with constant $A$, since
						\begin{align}
								\label{proof: th9-as5-1}
								&\max_{w > R} \|\nabla f(x;\hat{b}^{(w,1)}) - \nabla f^{(1)}(x)\|=(1-\delta)c \leq c \leq A,\\
								&\max_{w > R} \|\nabla f(x;\hat{b}^{(w,2)}) - \nabla f^{(2)}(x)\|=\delta c \leq c \leq A.
								\label{proof: th9-as5-2}
							\end{align}
							The last inequalities of \eqref{proof: th9-as5-1} and \eqref{proof: th9-as5-2} come from $c \triangleq \min\{A,\xi\} \leq A$.	
						
						The two constructed instances result in the same set of local costs ($R$ of them are $f(x;1)$ and the others are $f(x;2)$) yet different orders of labels.  Since the algorithm whose output is invariant with respect to the identities of the workers cannot distinguish these two instances, it yields the same output. Nevertheless, according to \eqref{gradient of instance 1} and \eqref{gradient of instance 2}, the gradients of the two objectives at the algorithmic output are different. Therefore, the algorithmic output must fail on at least one among the two. Below, we investigate the learning errors in the two instances, and the larger one exactly gives the lower bound of the learning error for the algorithm.
						
						For any $x$, note that
						\begin{align}\label{eq: lower-bound}
								&\max \{\| \nabla f^{(1)}(x)\| , \|\nabla f^{(2)}(x)\| \} \\
								\geq &\frac12 \left(\|\nabla f^{(1)}(x)\| + \|\nabla f^{(2)}(x)\|\right) \nonumber\\
								\geq &\frac12\|\nabla f^{(1)}(x) - \nabla f^{(2)}(x)\| \nonumber\\
								= &\frac{\delta c}{2}. \nonumber
						\end{align}
						Specifically,  letting $x=x^t$ be the $t$-th iterate of the algorithm running on either of the two instances, \eqref{eq: lower-bound} gives that
						\begin{align}
								\max_{k \in \{1, 2\}} \E\|\nabla f^{(k)}(x^t)\|^2 =  \max_{k \in \{1, 2\}}\|\nabla f^{(k)}(x^t)\|^2 \geq \frac{\delta^2 c^2}{4}.
						\end{align}
						where the first equality is because there is no randomness in these two instances. Further, since
						\begin{align}
								\E\|\nabla f^{(1)}(x^t)\|^2 + \E\|\nabla f^{(2)}(x^t)\|^2 \geq \max_{k \in \{1, 2\}} \E\|\nabla f^{(k)}(x^t)\|^2 \geq \frac{\delta^2 c^2}{4},
						\end{align}
						we have
						\begin{align}
								&\frac{1}{T} \sum_{t=1}^{T} \E\|\nabla f^{(1)}(x^t)\|^2 + \frac{1}{T} \sum_{t=1}^{T} \E\|\nabla f^{(2)}(x^t)\|^2 \\
								=& \frac{1}{T} \sum_{t=1}^{T} \Big(\E\|\nabla f^{(1)}(x^t)\|^2 + \E\|\nabla f^{(2)}(x^t)\|^2\Big) \nonumber \\ \geq& \frac{\delta^2 c^2}{4}. \nonumber
						\end{align}
						With the fact that
						\begin{align}
								\hspace{-1em}			2\max_{k \in \{{1, 2}\}} \frac{1}{T} \sum_{t=1}^{T} \E\|\nabla f^{(k)}(x^t)\|^2
								\geq\frac{1}{T} \sum_{t=1}^{T} \E\|\nabla f^{(1)}(x^t)\|^2 + \frac{1}{T} \sum_{t=1}^{T} \E\|\nabla f^{(2)}(x^t)\|^2 \geq \frac{\delta^2c^2}{4},
						\end{align}
						we have
						\begin{align}\label{eq: lower bound over T}
								\max_{k \in \{{1, 2}\}} \frac{1}{T} \sum_{t=1}^{T} \E\|\nabla f^{(k)}(x^t)\|^2 \geq  \frac{\delta^2 c^2}{8}.
						\end{align}
						Choosing $R$ regular local costs $\{f_w(x) = f(x; \hat{b}^{(w, k)}): w \leq R\}$ and $W-R$ poisoned local costs $\{\tilde{f}_w(x) = f(x; \hat{b}^{(w, k)}): w > R\}$ where $k = \arg\max_{k \in \{1, 2\}} \frac{1}{T} \sum_{t=1}^{T}\E\|\nabla f^{(k)}(x^t)\|^2$, \eqref{eq: lower bound over T} gives that

							\begin{align}\label{proof: th9-case1-last-equation}
								\frac{1}{T} \sum_{t=1}^{T} \E\|\nabla f(x^t)\|^2 \geq \frac{\delta^2 c^2}{8} = \frac{\delta^2 \min\{A^2,\xi^2\}}{8},
							\end{align}	
						which concludes the proof.
					\end{proof}
					
					The proof is motivated by those of \citet[Theorem 3]{karimireddy2021byzantine} and \citet[Proposition 1]{allouah2023fixing}. The major difference is that we consider label poisoning attacks in which the poisoned workers can only poison their local labels, while \citet{karimireddy2021byzantine} and \citet{allouah2023fixing} consider Byzantine attacks in which the Byzantine workers can behave arbitrarily. Different types of attacks lead to different constructions of the instances in the proofs.

					\section{Impacts of Heterogeneity and Attack Strengths}
					\label{sec: Appendix F}
					To investigate the impacts of heterogeneity of data distributions and strengths of label poisoning attacks, we have conducted numerical experiments by varying the data distributions and the levels of label poisoning attacks, and presented the best classification accuracies in Figure \ref{fig:NN_mnist_alpha_prob}. Here, we provide all classification accuracies in Table \ref{table: impacts of heterogeneity and attack strengths} to complement Figure \ref{fig:NN_mnist_alpha_prob}.
					\begin{table}
						\centering
						\resizebox{0.7\textwidth}{!}{
							\begin{tabular}{ccccccc}
								\toprule
								\toprule
								$\beta$ & $p$  &  Mean & CC  &  FABA &  LFighter & TriMean  \\
								\midrule
								\multirow{6}{*}{5} & 0.0 & \textbf{0.9441} & 0.9385 &  0.9410  & 0.9420 & 0.9429  \\
								& 0.2 & 0.9426 & 0.9405 & \textbf{0.9439} & 0.9437 & 0.9425  \\
								& 0.4 & 0.9443  &0.9421  &  0.9437 & \textbf{0.9456} & 0.9430   \\
								& 0.6  & 0.9427  & 0.9402   & 0.9431 & \textbf{0.9439} & 0.9397 \\
								& 0.8  & 0.9429  & 0.9408   & 0.9424 &\textbf{ 0.9443} & 0.9386 \\
								& 1.0  & 0.9415  & 0.9382   & 0.9423 & \textbf{0.9437} & 0.9371 \\
								\midrule
								\multirow{6}{*}{1} & 0.0 & \textbf{0.9437 }& 0.9362 &  0.9390  & 0.9361 & 0.9402  \\
								& 0.2 & \textbf{0.9448} & 0.9371 & 0.9417 & 0.9341 & 0.9373  \\
								& 0.4 & 0.9417  &0.9404  &  \textbf{0.9447 }& 0.9404 & 0.9396   \\
								& 0.6  & 0.9386  & 0.9355   & 0.9409 & \textbf{0.9422} & 0.9365 \\
								& 0.8  & 0.9402  & 0.9323   & 0.9386 & \textbf{0.9431} & 0.9318 \\
								& 1.0  & 0.9424  & 0.9401   & 0.9414 & \textbf{0.9433} & 0.9273 \\
								\midrule
								\multirow{6}{*}{0.1} & 0.0 & \textbf{0.9407} & 0.9251 &  0.9404  & 0.9170 & 0.9134  \\
								& 0.2 & \textbf{0.9423} & 0.9292 & 0.9271 & 0.9313 & 0.9076  \\
								& 0.4 & \textbf{0.9420}  &0.9270  &  0.9229 & 0.9201 & 0.8942   \\
								& 0.6  & 0.9372  & 0.9226   & 0.8996 & \textbf{0.9377} & 0.8891 \\
								& 0.8  & 0.9278  & 0.9135   & 0.9431 & \textbf{0.9433} & 0.8498 \\
								& 1.0  & 0.8327  & 0.8305   & 0.9426 & \textbf{0.9449} & 0.8054 \\
								\midrule
								\multirow{6}{*}{0.05} & 0.0 &\textbf{ 0.9468} & 0.9317 &  0.8976  & 0.8617 & 0.8763  \\
								& 0.2 & \textbf{0.9467 }& 0.9311 & 0.8603 & 0.8845 & 0.8529  \\
								& 0.4 & \textbf{0.9474}  &0.9279  &  0.8601 & 0.8860 & 0.8526   \\
								& 0.6  & \textbf{0.9418}  & 0.9248   & 0.8571 & 0.9343 & 0.8492 \\
								& 0.8  & 0.9201  & 0.9155   & \textbf{0.9394} & 0.9385 & 0.8576 \\
								& 1.0  & 0.8573  & 0.8950   & \textbf{0.9384} & 0.9374 & 0.8106 \\
								\midrule
								\multirow{6}{*}{0.03} & 0.0 & \textbf{0.9426} & 0.9299 &  0.8634  & 0.8517 & 0.8523  \\
								& 0.2 & \textbf{0.9413} & 0.9298 & 0.8507 & 0.8493 & 0.8427  \\
								& 0.4 & \textbf{0.9393}  &0.9268  &  0.8506 & 0.8502 & 0.8370   \\
								& 0.6  & \textbf{0.9374}  & 0.9206   & 0.8481 & 0.8449 & 0.8246 \\
								& 0.8  & \textbf{0.9159}  & 0.8679   & 0.8019 & 0.8313 & 0.7869 \\
								& 1.0  & 0.8312  & 0.8152   & 0.7437 & \textbf{0.9396} & 0.7514 \\
								\midrule
								\multirow{6}{*}{0.01} & 0.0 & \textbf{0.9456} & 0.9164 &  0.8678  & 0.8681 & 0.8008  \\
								& 0.2 & \textbf{0.9441} & 0.9165 & 0.8657 & 0.8660 & 0.7788  \\
								& 0.4 & \textbf{0.9415}  &0.9135  &  0.8631 & 0.8632 & 0.7858   \\
								& 0.6  & \textbf{0.9356}  & 0.9039  & 0.8601 & 0.8399 & 0.7708 \\
								& 0.8  & \textbf{0.8827}  & 0.8750   & 0.8155 & 0.7864 & 0.7457 \\
								& 1.0  & \textbf{0.8505}  & 0.8278   & 0.7731 & 0.8162 & 0.7258 \\
								\bottomrule
								\bottomrule
								
						\end{tabular}}
						\caption{Accuracies of trained two-layer perceptrons by all aggregators on the MNIST dataset under static label flipping attacks. The hyper-parameter $\beta$ that characterizes the heterogeneity is in $[5, 1, 0.1, 0.05, 0.03,$ $0.01]$ and the flipping probability $p$ that characterizes the attack strength is in $[0.0, 0.2, 0.4, 0.6, 0.8, 1.0]$.}
						\label{table: impacts of heterogeneity and attack strengths}
					\end{table}
					
					\section{Bounded Variance of Stochastic Gradients}
					\label{sec: Appendix G}
					
					Now we verify the reasonableness of Assumption \ref{Assump: bounded variance}. We consider softmax regression on the MNIST dataset and setup $W = 10$ workers where $R = 9$ workers are regular and the remaining one is poisoned. We compute the variances of regular stochastic gradients and poisoned stochastic gradients under static label flipping attacks in the i.i.d., mild non-i.i.d. and non-i.i.d. cases, and present the maximum variance of regular stochastic gradients and poisoned stochastic gradient in Figure \ref{fig:SR_mnist_variance}. As depicted in Figure \ref{fig:SR_mnist_variance}, the variances of regular stochastic gradients and poisoned stochastic gradients are both bounded under static label flipping attacks, which validates Assumption \ref{Assump: bounded variance}.

					\begin{figure}
						\centering
						\includegraphics[scale=0.145]{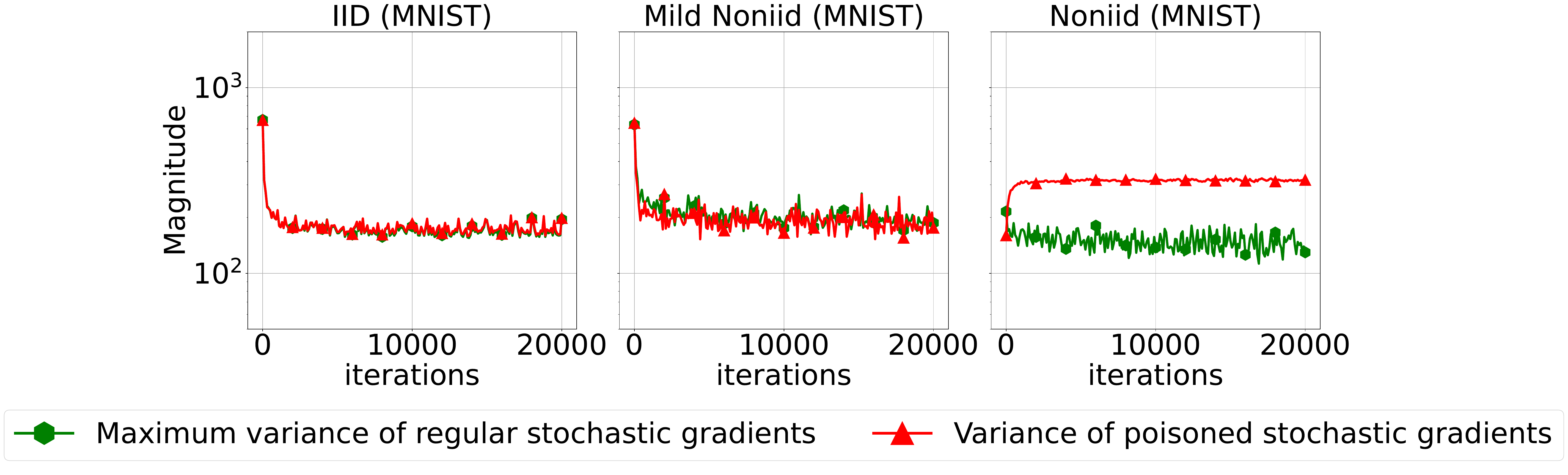}
						\caption{Maximum variance of regular stochastic gradients and variance of poisoned stochastic gradients in softmax regression on the MNIST dataset under static label flipping attacks.}
						\label{fig:SR_mnist_variance}
					\end{figure}

						\section{Impact of Fraction of Poisoned Workers}
						\label{Impact of fraction of the poisoned workers}
						To further investigate the impact of different fractions of poisoned workers, here we vary the number of poisoned workers. We setup $W=10$ workers, among which $R=8$ ($R=7$) are regular, while the remaining 2 (3) are poisoned. The experimental settings are the same as those in the nonconvex case. We use a step size of $\gamma = 0.01$ and a momentum coefficient of $\alpha=0.1$. As shown in Figures \ref{fig:NeuralNetwork_label_flipping_R=8}, \ref{fig:NeuralNetwork_furthest_label_flipping_R=8}, \ref{fig:NeuralNetwork_label_flipping_R=7}, and \ref{fig:NeuralNetwork_furthest_label_flipping_R=7}, the mean aggregator generally outperforms the robust aggregators in the non-i.i.d. case, which is consistent with the experimental results of $R=9$.
						Additionally, from Figures \ref{fig:NeuralNetwork_label_flipping} and \ref{fig:NeuralNetwork_furthest_label_flipping} to Figures \ref{fig:NeuralNetwork_label_flipping_R=8}, \ref{fig:NeuralNetwork_furthest_label_flipping_R=8}, \ref{fig:NeuralNetwork_label_flipping_R=7} and \ref{fig:NeuralNetwork_furthest_label_flipping_R=7}, we observe that as the fraction of poisoned workers increases, the performance of both the mean aggregator and robust aggregators decreases. This aligns with the theoretical results in Theorems \ref{thm:convergence with RAgg} and \ref{thm:convergence with Mean}.
						
						\begin{figure}
							\centering
							\includegraphics[scale=0.18]{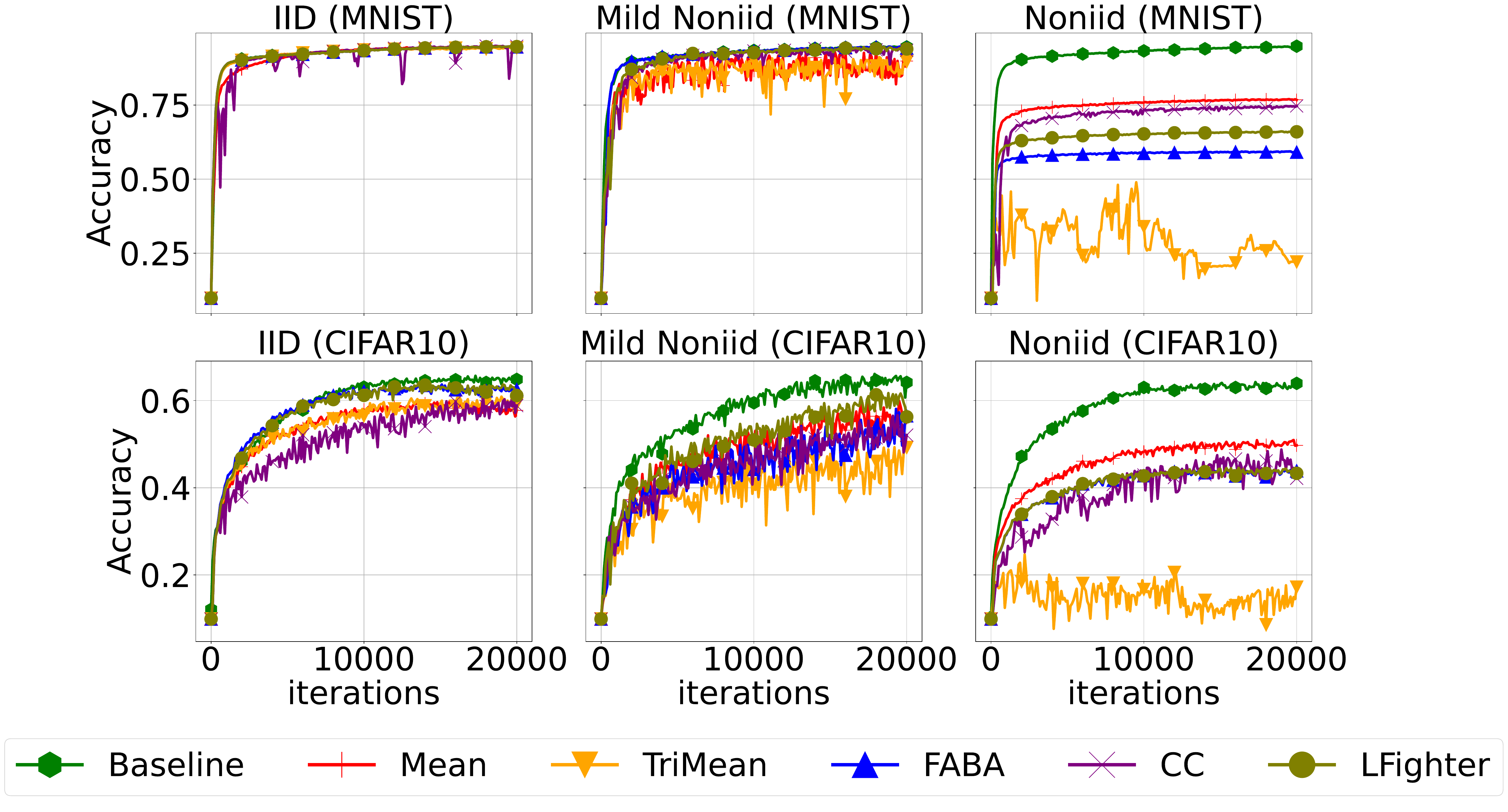}
							\caption{Accuracies of two-layer perceptrons on the MNIST dataset and convolutional neural networks on the CIFAR10 dataset under static label flipping attacks when $R=8$.}
							\label{fig:NeuralNetwork_label_flipping_R=8}
						\end{figure}
						\begin{figure}
							\centering
							\includegraphics[scale=0.18]{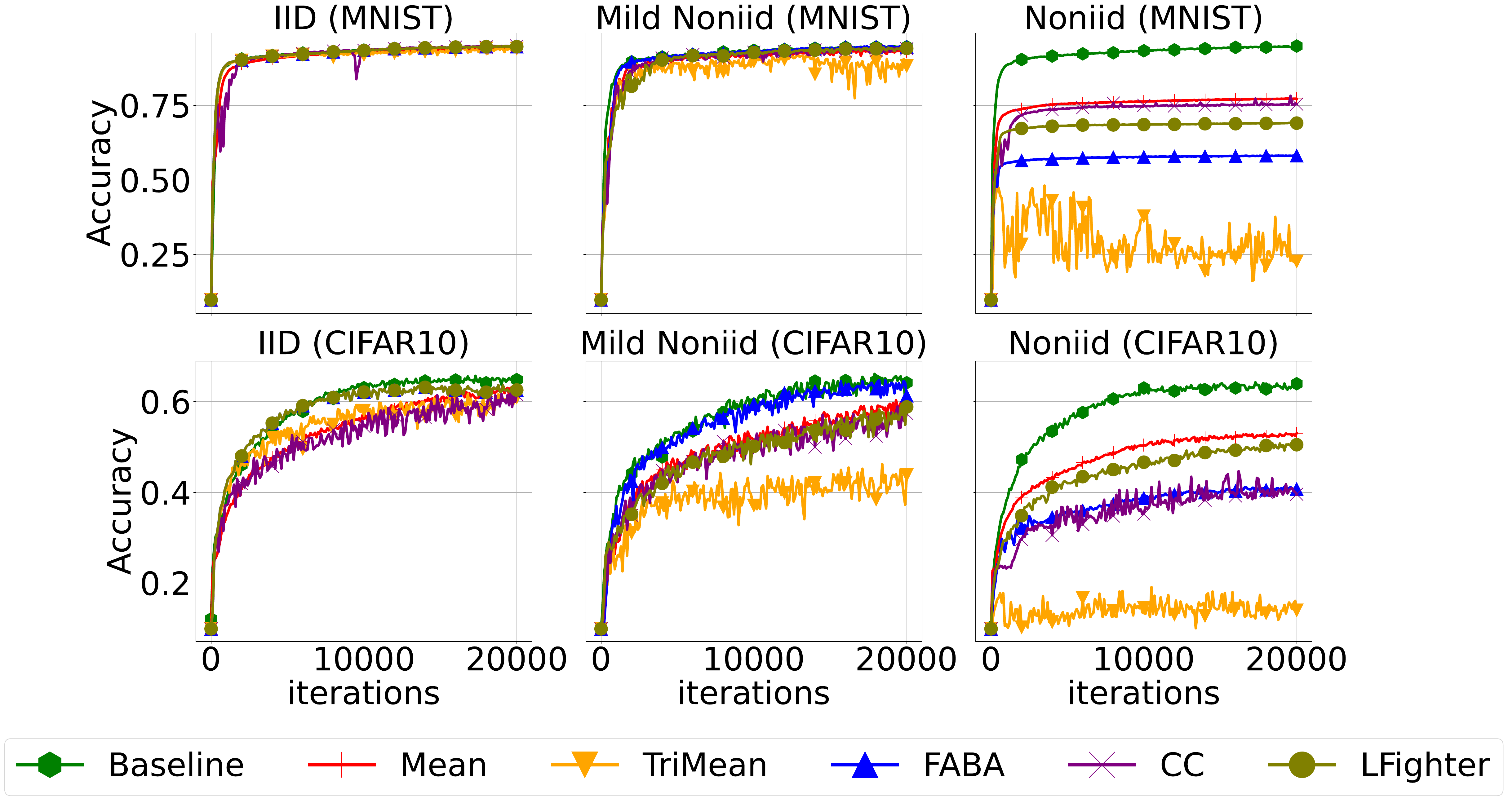}
							\caption{Accuracies of two-layer perceptrons on the MNIST dataset and convolutional neural networks on the CIFAR10 dataset under dynamic label flipping attacks when $R=8$.}
							\label{fig:NeuralNetwork_furthest_label_flipping_R=8}
						\end{figure}
						
						\begin{figure}
							\centering
							\includegraphics[scale=0.18]{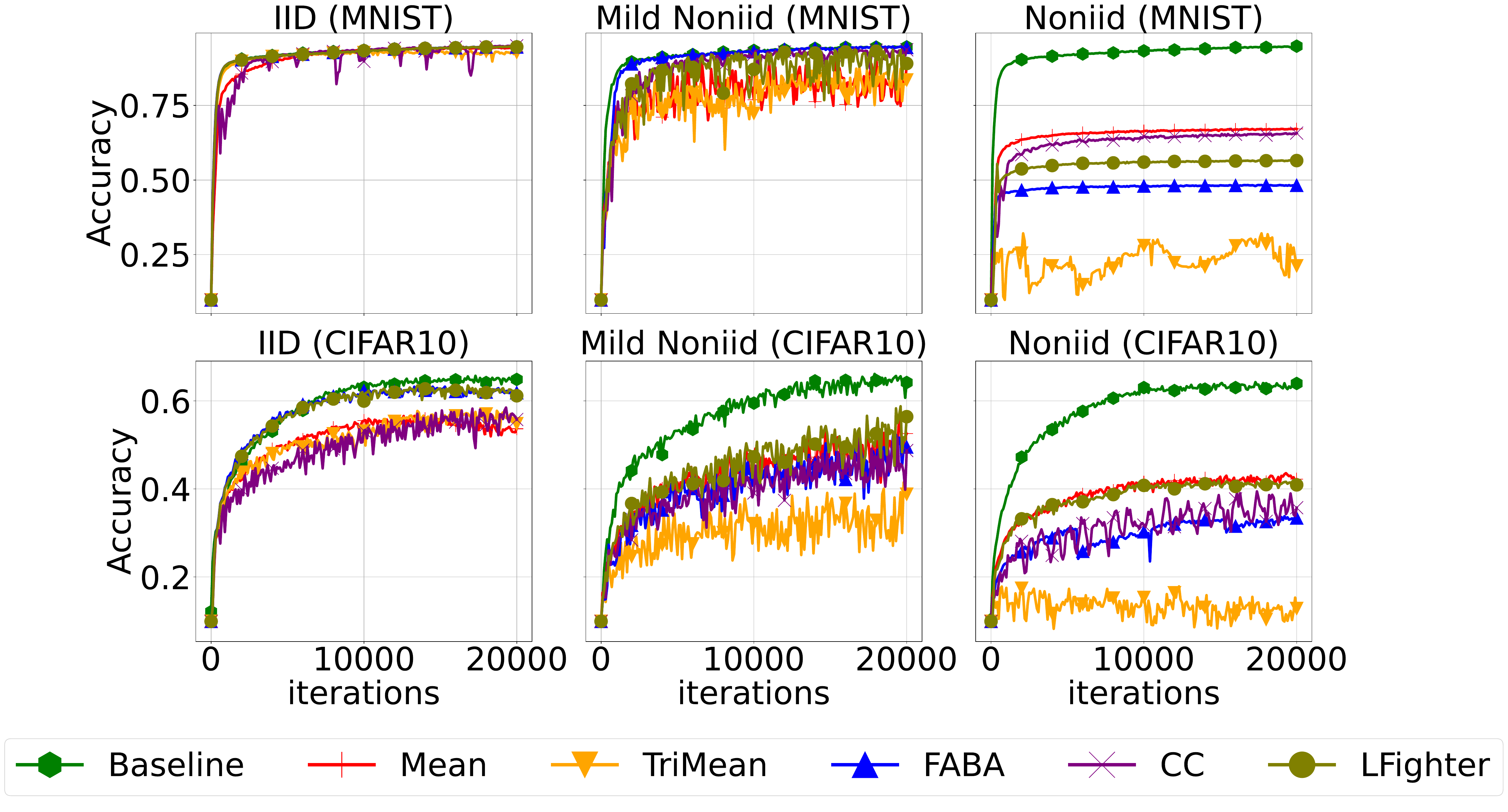}
							\caption{Accuracies of two-layer perceptrons on the MNIST dataset and convolutional neural networks on the CIFAR10 dataset under static label flipping attacks when $R=7$.}
							\label{fig:NeuralNetwork_label_flipping_R=7}
						\end{figure}
						\begin{figure}
							\centering
							\includegraphics[scale=0.18]{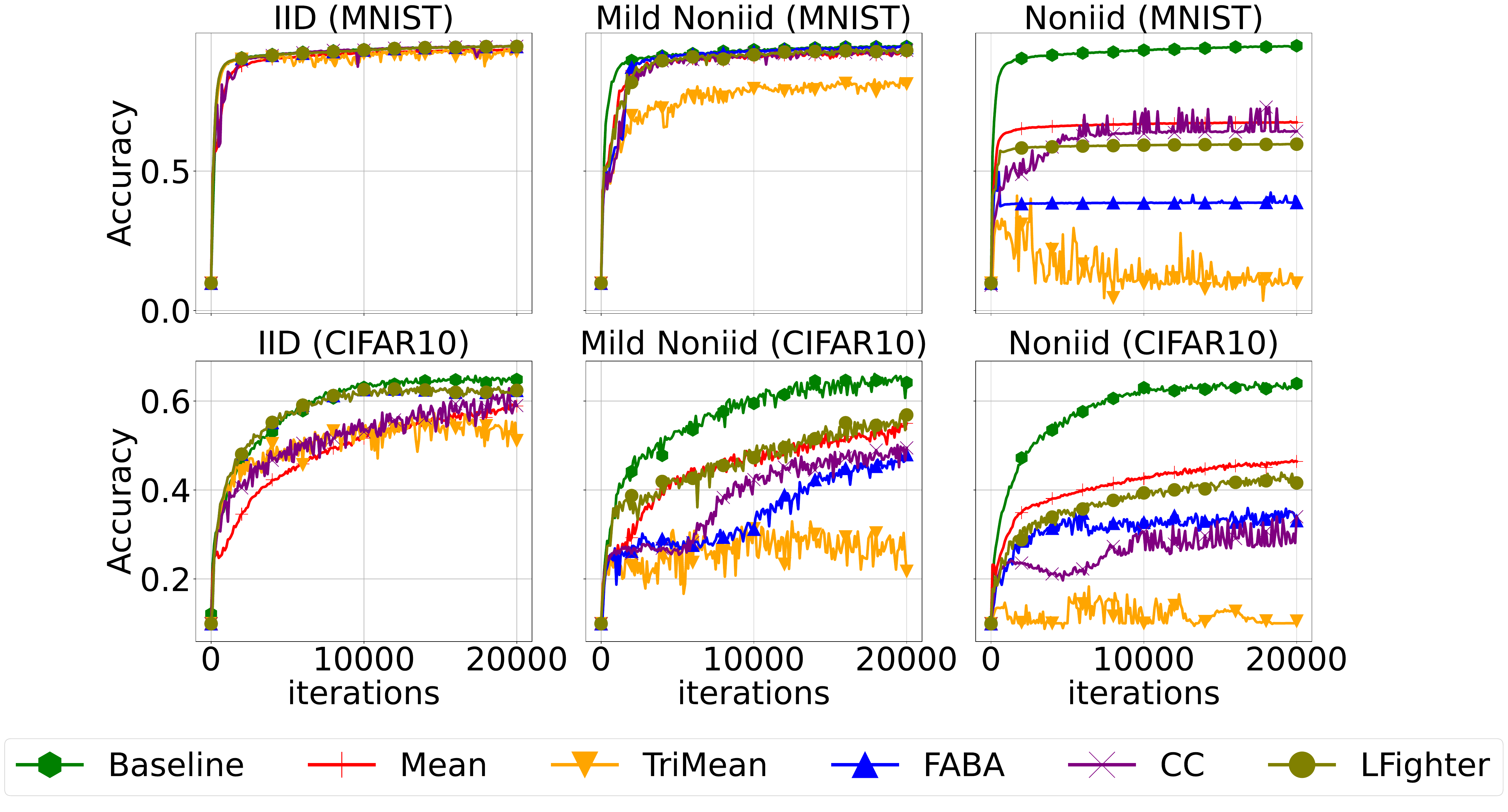}
							\caption{Accuracies of two-layer perceptrons on the MNIST dataset and convolutional neural networks on the CIFAR10 dataset under dynamic label flipping attacks when $R=7$.}
							\label{fig:NeuralNetwork_furthest_label_flipping_R=7}
						\end{figure}

					\vskip 0.2in
					\bibliography{refs}

				\end{document}